\documentclass{article}

\usepackage[main, final]{neurips_2026}

\usepackage[utf8]{inputenc}
\usepackage[T1]{fontenc}
\usepackage{hyperref}
\usepackage{url}
\usepackage{booktabs}
\usepackage{amsfonts}
\usepackage{nicefrac}
\usepackage{microtype}
\usepackage{xcolor}

\usepackage{graphicx}
\usepackage{subcaption}

\usepackage{amsmath}
\usepackage{amssymb}
\usepackage{mathtools}
\usepackage{amsthm}
\usepackage{tikz}

\usetikzlibrary{arrows.meta, calc, backgrounds, positioning}
\usepackage{thmtools}
\usepackage{thm-restate}
\usepackage{dsfont}
\usepackage{algorithm}
\usepackage{algorithmic,tcolorbox}

\usepackage[capitalize,noabbrev]{cleveref}

\theoremstyle{plain}
\newtheorem{theorem}{Theorem}[section]
\newtheorem{proposition}[theorem]{Proposition}
\newtheorem{lemma}[theorem]{Lemma}
\newtheorem{corollary}[theorem]{Corollary}
\theoremstyle{definition}

\theoremstyle{remark}

\newcommand{\bs}{\boldsymbol}

\newcommand{\RR}{\ensuremath{\mathbb{R}}}

\newcommand{\EE}{\ensuremath{\mathbb{E}}}
\newcommand{\PP}{\ensuremath{\mathbb{P}}}

\newcommand{\cZ}{\ensuremath{\mathcal{Z}}}
\newcommand{\cX}{\ensuremath{\mathcal{X}}}
\newcommand{\cN}{\ensuremath{\mathcal{N}}}

\newcommand{\cM}{\ensuremath{\mathcal{M}}}

\newcommand{\cP}{\ensuremath{\mathcal{P}}}

\newcommand{\cA}{\ensuremath{\mathcal{A}}}

\newcommand{\cD}{\ensuremath{\mathcal{D}}}

\renewcommand{\epsilon}{\varepsilon}
\renewcommand{\phi}{\varphi}

\renewcommand{\leq}{\leqslant}
\renewcommand{\geq}{\geqslant}

\newcommand{\emp}{\text{emp}}

\newcommand{\Memp}{\mathcal M^\emp}

\newcommand{\SeMIstar}{\mathrm{SeMI}^{\ast}}
\newcommand{\SeMISGD}{\mathrm{SeMI}^{\mathrm{SGD}}}
\newcommand{\SeMISGDmax}{\mathrm{SeMI}^{\mathrm{SGD}}_{\max}}
\newcommand{\SeMIunif}{\mathrm{SeMI}^{\mathrm{Unif}}}
\newcommand{\SeMImax}{\mathrm{SeMI}^{\max}}

\newcommand{\cU}{\mathcal{U}}

\newcommand{\bP}{\mathbb{P}}

\definecolor{Bleu}{HTML}{20bd89}
\definecolor{Red}{HTML}{f14e0c}
\hypersetup{colorlinks,citecolor=Bleu,linkcolor=Red}

\title{Sequential Membership Inference Attacks}

\author{%
  Thomas Michel\\
  Univ. Lille, Inria, CNRS, Centrale Lille\\
  UMR 9189-CRIStAL, France \\
  \And
  Debabrota Basu \\
  Univ. Lille, Inria, CNRS, Centrale Lille\\
  UMR 9189-CRIStAL, France \\
  \And
  Emilie Kaufmann\\
  Univ. Lille, CNRS, Inria, Centrale Lille\\
  UMR 9189-CRIStAL, France \\
}

\begin{document}

\maketitle

\begin{abstract}
  Modern AI models are not static. They go through multiple updates in their lifecycles.
  We propose to design Sequential Membership Inference (SeMI) attacks leading to tighter privacy audits by exploiting the sequence of models and injecting a target canary at a controlled insertion time.
  First, for empirical mean computation, we develop $\mathrm{SeMI}^{\ast}$, an {optimal SeMI attack to identify the presence of a target inserted at a specific insertion step}.
  We derive the power of $\mathrm{SeMI}^{\ast}$ to show that accessing the model sequence yields more powerful MI attacks than scrutinising only the final model.
  $\mathrm{SeMI}^{\ast}$ exhibits an isolation property. Its power depends on the statistics obtained right before and after insertion of the target.
  Leveraging this insight, we develop practical white-box (accessing model gradients) and black-box (accessing loss) SeMI attacks against models trained with (DP-)SGD.
  Across datasets and models trained with (DP-)SGD, our experiments show that SeMI attacks achieve higher power than snapshot-independent baselines, and yield tighter privacy audits thanks to (a) control over the insertion time and (b) observations across the model sequence.
  Code is available at \url{https://github.com/Thomick/pysemi}.
\end{abstract}

\section{Introduction}\label{sec:introduction}

Machine Learning (ML) models memorize training data, creating privacy risks for individuals whose data were used to develop them~\citep{shokri2017membership, carlini2022membership}. Research further establishes that even simple statistical models, like empirical mean estimation, leak membership information~\citep{homer2008resolving,azizeTargetsAreHarder2025}. %
While differential privacy~\citep{dwork2006calibrating, dpbook} aims to bound privacy leakage of algorithms, privacy auditing quantifies these risks by testing whether specific records, referred to as target or \textit{canary}, can be detected in the training set.
\textit{Membership Inference (MI) attacks} serve as the primary tool for auditing: an attacker, who distinguishes members from non-members with high confidence, establishes an empirical lower bound on the model's privacy leakage~\citep{jayaraman2019evaluating, jagielski2020auditing,nasrTightAuditingDifferentially2023}. MI attacks with higher distinguishing power yield tighter audits.
In addition, the target selection matters as points `far' from the data distribution leak more information than typical points~\citep{azizeTargetsAreHarder2025, carlini2022privacyonion}.

\cite{jayaraman2019evaluating} report 5--10$\times$ gaps between theoretical DP guarantees and empirical privacy estimates. Techniques for tight auditing~\citep{nasrTightAuditingDifferentially2023, steinkePrivacyAuditingOne2023} and debugging~\citep{tramer2022debugging} reduce this gap by carefully crafting adversarial targets and observing per-step gradients during training. While training access is confined to the developer, downstream auditors only see the released models, whether through full weight access or query-only APIs.
In addition, \textit{modern ML models rarely exist as isolated snapshots.} They are trained or fine-tuned in successive phases, with a new snapshot released after each: versioned foundation models, task-specific fine-tunes distributed on top of public backbones, continually-updated production systems, and federated models aggregated across rounds.
\textit{We show that scrutinizing the sequence of released models yields tighter audits than a single final snapshot},
while most of the auditors target a specific model snapshot. %

\cite{jagielskiHowCombineMembershipInference2023} demonstrate that combining the MI scores across model updates improves attack accuracy. They apply standard MI attacks separately to each model snapshot, and then, aggregate the scores via difference or ratio operations.
Similarly, the benefits of leveraging model snapshots are empirically shown for reconstruction attacks to recover data from online learning updates~\citep{salem2020updates}, machine unlearning attacks~\citep{chen2020machine}, and leakage analysis of snapshots in NLP~\citep{zanella2020analyzing}.
SeqMIA~\citep{li2024seqmia}, which extends TrajectoryMIA~\citep{liu2022trajectory} from a set of losses to ordered sequences of metrics, classifies sequences computed on distilled snapshots of a single released model (Appendix~\ref{app:literature}).
\textit{Our approach differs by deriving the optimal MI attacks via likelihood ratio tests and identifying the insertion time of a canary as a new actionable lever for tighter audits.}

\begin{figure}[t!]
	\centering\vspace*{-.8em}
	\includegraphics[width=\linewidth]{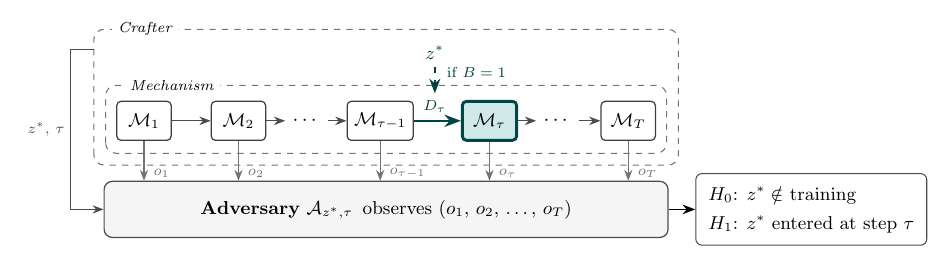}\vspace*{-1em}
	\caption{\textbf{Sequential MI game:} Crafter controls a mechanism that releases $T$ outputs $(o_1, \ldots, o_T)$. With probability $1/2$, the target $z^*$ is inserted into batch $D_\tau$ at step $\tau$. Otherwise, the mechanism runs unmodified. The adversary observes all outputs and decides whether $z^*$ is inserted ($H_1$) or not ($H_0$).}\label{fig:setup}\vspace*{-1.5em}
\end{figure}

Specifically, \textit{we study Membership Inference (MI) when the auditor observes a sequence of $T$ model snapshots released across training phases rather than a single static one.} We formalize this \textit{Sequential MI (SeMI)} setting as a hypothesis testing game between a crafter and an adversary over $T$ outputs with a choice of the insertion time $\tau$. The auditor, a model owner or a party allowed to insert data, runs both. As crafter it chooses the canary and its insertion time, and its adversary tests the released outputs. We show how Type-I/II errors of SeMI yield a privacy lower bound for auditing.
The sequential setting offers two levers for designing tight audits: \textit{selecting targets that leak information, and choosing when to insert them}. Specifically, we ask:
\vspace*{-0.5em}\begin{tcolorbox}[top=2pt,bottom=2pt,left=1pt,right=1pt]
	1. \textit{Does SeMI with insertion-time control improve over a static attack on the final model?}\\
	2. \textit{Does theory-driven design of SeMI attacks yield stronger MI attacks against (DP-)SGD?}\\
	3. \textit{Do SeMI attacks with canary and insertion-time control lead to tighter privacy audits?}
\end{tcolorbox}

\textbf{Our contributions} address these questions affirmatively.

\textbf{(1) Optimal SeMI attack, isolation property, and knowledge of $\tau$.} Motivated by the optimal static MI attack~\citep{azizeTargetsAreHarder2025} derived from the Neyman--Pearson lemma~\citep{neyman1933ix}, we derive the optimal SeMI attack $\SeMIstar$ as a likelihood-ratio (LR) test parameterized by the insertion time $\tau$. For empirical mean computation on Gaussian data, we obtain the closed-form Type-I and Type-II errors. The analysis exposes an \emph{isolation property}: consecutive outputs reveal the batch mean at $\tau$, so the LR depends on the data only through this batch statistic. The power of $\SeMIstar$ depends on the batch size $n$ and the target's distance from the mean, not on $\tau$ or $T$.
A single-observation attack on the final snapshot sees the target pooled among $nT$ samples, and its signal dilutes as training data accumulates. We further study $\SeMImax$ and $\SeMIunif$, which assume $\tau$ is unknown or uniform, and find that $\SeMIstar$ dominates both.

\textbf{(2) White- and black-box SeMI attacks on (DP-)SGD.} Since (DP-)SGD~\citep{kiefer1952stochastic, abadi2016deep} is the standard for training private ML models, we extend the SeMI attack to (DP-)SGD via a batch-gradient Gaussianity assumption. This yields $\SeMISGD$ and shows that the isolation property holds for its per-transition log-LR. In the white-box setting, $\SeMISGD$ takes a Mahalanobis-weighted score between the parameter update and the target gradient. In the black-box setting, where the auditor observes only target losses, we propose two per-transition surrogates, Loss Diff and Loss Ratio, recovering the Back-Front and Delta heuristics of~\cite{jagielskiHowCombineMembershipInference2023} as restrictions to the inserted phase.

\textbf{(3) From tighter SeMI attacks to sequential privacy audits.} We formalize privacy auditing of sequential mechanisms as a testing game that encodes the auditor's control over target selection and insertion time, and turn confidence intervals on the Type-I and Type-II errors of SeMI attacks on (DP-)SGD into $\varepsilon$ lower bounds. Auditing a softmax classifier trained with DP-SGD on Fashion-MNIST, we find that (i) $\SeMISGD$ is the only attack whose bound stays positive as the number of released snapshots $T$ increases, (ii) on the same pair of snapshots, the white-box $\SeMISGD$ gives tighter bounds than the black-box Loss Diff and Loss Ratio, (iii) knowing the insertion time tightens the audit relative to its worst-case counterpart, and (iv) audit tightness depends strongly on $\tau$, giving the auditor a lever for tighter bounds beyond target selection.

\vspace*{-1em}\section{Auditing and Membership Inference for Sequential Mechanisms}\label{sec:mi_tests}\vspace*{-.8em}
We formalise MI attacks against sequential mechanisms and their privacy audits. Data is collected over $T$ steps: at step $t \in [T]$, a batch $D_t$ of $n$ i.i.d.\ samples from $\mathcal{D}_t$ arrives, and a sequential mechanism $\mathcal{M}_t$ is applied to the accumulated dataset, releasing $o_t \triangleq \mathcal{M}_t(\bigcup_{i=1}^t D_i)$. We write $\cM$ for the global mechanism with input $D \triangleq \bigcup_{i=1}^{T} D_i$ and output $o \triangleq (o_1,\dots,o_T)$. The mechanisms and distributions may vary over time, modelling an ML system trained or fine-tuned in successive phases with each version released to users.
\textit{We aim to assess the privacy of the global mechanism $\cM$ while observing the released outputs of the intermediate mechanisms.}

To quantify the privacy level of a mechanism, we adopt the classical $(\varepsilon, \delta)$-Differential Privacy (DP) definition~\citep{dpbook}. A mechanism $\cM$ satisfies $(\varepsilon, \delta)$-DP if for all measurable sets of output $S$ and datasets $D, D'$ differing in one entry, and some $\varepsilon>0$ and $\delta \in [0,1)$,
\begin{equation}
    \Pr[\mathcal{M}(D) \in S] \leq e^\varepsilon \Pr[\mathcal{M}(D') \in S] + \delta\,. \label{def:privacy}
\end{equation}
Given some $\delta$, \textit{privacy auditing amounts to finding a lower bound on $\varepsilon$ for which \eqref{def:privacy} is satisfied.}

Two dominant approaches to auditing exist. The first estimates the smallest $\varepsilon$ for which the hockey-stick divergence $\sup_{S}\max\{\Pr[\cM(D) \in S] - e^\varepsilon \Pr[\cM(D') \in S], 0\}$ is below $\delta$~\citep{koskelaAuditingDifferentialPrivacy2024,basu2026sublinear}. The second uses the hypothesis-testing interpretation of DP~\citep{wasserman2010statistical, kairouz2015composition}, running MI tests (that are called \emph{attacks}) and turning their error probabilities into a lower bound on $\varepsilon$. The second approach also guides the design of optimal MI attacks and privacy-leaking canaries~\citep{nasrTightAuditingDifferentially2023,azizeTargetsAreHarder2025}, and we adopt it for sequential mechanisms.

\setlength{\textfloatsep}{8pt}
\begin{figure}[t!]
    \begin{minipage}[t]{0.54\linewidth}
        \begin{algorithm}[H]
            \caption{SeMI Game}
            \label{alg:crafter}
            \begin{algorithmic}[1]
                \STATE \textbf{Input:} Mechanisms $(\mathcal{M}_1, \ldots, \mathcal{M}_T)$, data distributions $(\mathcal{D}_1, \ldots, \mathcal{D}_T)$, batch size $n$, target $z^*$, insertion time $\tau$, adversary $\cA$
                \STATE Sample $B \sim \mathcal{B}(1/2)$
                \FOR{$t = 1, \ldots, T$}
                \STATE Sample batch $D_t \sim \mathcal{D}_t^{\otimes n}$
                \IF{$B = 1$ and $t = \tau$}
                \STATE Sample $j \sim \mathrm{Uniform}(\{1, \ldots, n\})$
                \STATE Replace $D_t[j] \leftarrow z^*$
                \ENDIF
                \STATE Compute $o_t = \mathcal{M}_t\left(\bigcup_{i=1}^t D_i\right)$
                \ENDFOR
                \STATE Adversary guesses $\widehat{B} = \cA(z^*,\tau,\{o_t\})$
                \STATE \textbf{Output:} $B, \widehat{B}$
            \end{algorithmic}
        \end{algorithm}
    \end{minipage}
    \hfill
    \begin{minipage}[t]{0.42\linewidth}\vspace*{.4em}
        \begin{algorithm}[H]
            \caption{SeMI Audit}
            \label{alg:mi-game}
            \begin{algorithmic}[1]
                \STATE \textbf{Input:} Mechanisms $(\mathcal{M}_1, \ldots, \mathcal{M}_T)$, data distributions $(\mathcal{D}_1, \ldots, \mathcal{D}_T)$, batch size $n$, target $z^*$, insertion time $\tau$, adversary $\cA$, rounds $R$, confidence $\xi$, privacy $\delta$
                \FOR{$r = 1, \ldots, R$}
                \STATE $B_r,\widehat{B}_r \gets \mathrm{SeMIGame}(\tau,\cA)$
                \ENDFOR
                \STATE Compute Clopper--Pearson upper bounds $\overline{\alpha}_R^{\xi}, \overline{\beta}_R^{\xi}$ at confidence level $\xi/2$ (Appendix~\ref{app:audit-cp-grid})
                \STATE \textbf{Output:} lower bound on $\varepsilon$ from~\eqref{eq:lb_epsilon} with $(\overline{\alpha}_R^{\xi},\overline{\beta}_R^{\xi})$
            \end{algorithmic}
        \end{algorithm}
    \end{minipage}
\end{figure}

\textbf{MI attacks} identify whether a target $z^*$ is present in the input dataset $D$ from the output of a mechanism~\citep{carlini2022membership,ye2022enhanced,leemann2023gaussian}. A crafter designs the canary $z^*$ and decides whether to insert it. With sequential access, we further ask whether $z^*$ belongs to a \textit{specific sub-dataset $D_{\tau}$, where $\tau \in \mathbb{N}$ is the insertion time}. A sequential MI attack tests two hypotheses parameterized by $z^*$:

$\mathbf{H_0}~ (\text{OUT}):$ $D \sim \cP_0$, if for all $t \in [T]$, $D_t \sim \cD_t^{\otimes n}$.

$\mathbf{H_1^{\tau}}~ (\text{IN}_{\tau}):$ $D \sim \cP_1^{\tau}$, if for all $t \in [T]$, $D_t \sim \cD_t^{\otimes n}$, and given $J \sim \mathcal{U}([n])$, the $J$-th entry of $D_{\tau}$ is replaced by $z^*$.

A \textbf{Sequential MI (SeMI) game}, stated in Algorithm~\ref{alg:crafter} and illustrated in Figure~\ref{fig:setup}, takes an insertion time $\tau$\footnote{For readability we suppress the dependence on $z^*$, since our focus is on where to insert it.}. The crafter draws $B \sim \mathcal{B}(1/2)$ and samples $D$ from $\cP_0$ if $B=0$ or from $\cP_1^{\tau}$ otherwise. The mechanism $\cM$ produces $o_{1:T}:=(o_1,\dots,o_T)$. A \textbf{SeMI attack} is an adversary $\cA$ mapping $(\tau, z^*, o_{1:T})$ to a guess $\widehat{B} = \cA(\tau,z^*,o_{1:T})$ for $B$. The associated errors are
\begin{align*}
    \alpha(\tau,\cA) & \triangleq \bP_{D \sim \cP_0}\!\left[\cA(\tau,z^*,\cM(D)) = 1\right],      \quad
    \beta(\tau,\cA)  \triangleq \bP_{D \sim \cP_1^{\tau}}\!\left[\cA(\tau,z^*,\cM(D)) = 0\right]\;.
\end{align*}

A general adversary that can exploit the knowledge of $\tau$ coincides with a statistical test of
\begin{eqnarray}
    \mathbf{H_0}~(\text{OUT}): (D \sim \cP_0) \ \ \text{ against }  \ \ \mathbf{H_1^{\tau}}~(\text{IN}_{\tau}): (D \sim \cP_1^{\tau})\label{def:test_tau}
\end{eqnarray}
and $\alpha(\tau,\cA)$ and $\beta(\tau,\cA)$ are respectively the type I and type II errors of this test.
The following extension of~\cite[]{kairouz2015composition} relates these two testing errors to the privacy parameters (see Appendix~\ref{app:audit-lemma-proof}).

\begin{restatable}[Connecting Sequential MI Attack Error and DP Audits]{lemma}{kairouzaudit}\label{lem:kairouz-audit}
    Fix a target $z^* \in \cZ$ and an insertion time $\tau$. For any sequential MI attack $\cA$, if the mechanism $\cM$ is $(\varepsilon, \delta)$-DP, then
    \begin{align*}
        \alpha(\tau,\cA) + e^{\varepsilon} \beta(\tau,\cA) & \geq 1-\delta, \qquad
        \beta(\tau,\cA) + e^{\varepsilon} \alpha(\tau,\cA) \geq 1-\delta.
    \end{align*}
\end{restatable}

\subsection{From SeMI Attacks to Privacy Audits}

Lemma~\ref{lem:kairouz-audit} yields a lower bound of privacy level for any insertion time and any adversary
\begin{equation}
    \varepsilon \geq \log\left(\max \left[\frac{1 - \delta - \alpha(\tau,\cA)}{\beta(\tau,\cA)},\frac{1 - \delta - \beta(\tau,\cA)}{\alpha(\tau,\cA)}\right]\right).\label{eq:lb_epsilon}
\end{equation}

A privacy audit (Algorithm~\ref{alg:mi-game}) runs $R$ independent rounds of the MI game and plugs high-confidence upper bounds on $\alpha(\tau,\cA)$ and $\beta(\tau,\cA)$ into~\eqref{eq:lb_epsilon}. In our experiments, we use a refinement of this scheme that can be applied to any attack comparing a statistic to a threshold. Instead of fixing $\cA$ (a threshold $\gamma$ on the attack statistic) in advance, we calibrate a grid of thresholds $\gamma_1,\dots,\gamma_K$ on separate runs, compute Clopper--Pearson confidence intervals~\citep{clopper1934use} at every grid point, and report the maximum lower bound across the grid with a union bound. %
Appendix~\ref{sec:audit_protocol} details this procedure and a calibration-free alternative based on the DKW inequality.
A tight lower bound now requires designing an insertion time and a SeMI attack that together achieve the smallest possible $\alpha$ and $\beta$.

\textbf{Threat model.} The adversary knows the target $z^*$ and the phase $\tau$ in which it may have been inserted. It observes the $T$ released models and cannot influence the training. White-box access gives it the released parameters, and black-box access allows querying the loss of a datapoint at each release. The white-box adversary also knows the training procedure and has reference data, from which it estimates the gradient mean and covariance. The black-box attacks use neither. The auditor is the model owner or a party allowed to insert data, and it runs both the crafter and the adversary. As crafter, it chooses $z^*$ and $\tau$ and performs the insertion (Algorithm~\ref{alg:crafter}). It then runs one of the attacks above as adversary on each of $R$ rounds and turns the measured error rates into a lower bound on $\varepsilon$ (Algorithm~\ref{alg:mi-game}). The adversary's knowledge of $\tau$ is therefore a consequence of the auditor's protocol. Previous auditing works rely on the same canary insertion~\citep{shokri2017membership,nasr2021adversary,carlini2022membership,maddockCANIFECraftingCanaries2023,nasrTightAuditingDifferentially2023,steinkePrivacyAuditingOne2023}, and we introduce the choice of the insertion time. Outside auditing, $\tau$ can be known from a subscription date or a federated round. When $\tau$ is unknown, $\SeMImax$ and $\SeMIunif$ test whether $z^*$ was inserted at all.

\subsection{From Hypothesis Tests to SeMI Attacks}

The connection between SeMI attacks and statistical hypothesis testing guides the design of effective adversaries. For the theoretical analysis, we adopt the white-box setting~\citep{sablayrolles2019white,carlini2022membership,nasrTightAuditingDifferentially2023, maddockCANIFECraftingCanaries2023}, where the adversary knows the data distributions $(\mathcal{D}_t)$, mechanisms $(\mathcal{M}_t)$, batch size $n$, and total steps $T$. %
Our auditing procedure carries over to other threat models: the form of $o_t$ changes (e.g., query outputs under black-box API access) and the adversary's test adapts accordingly, while the overall procedure remains the same.

As noted earlier, a SeMI attack that knows the insertion time $\tau$ corresponds to a statistical test for \eqref{def:test_tau} based on the observation $o_{1:T}=\cM(D)$. When the data-generating distributions $\cD_1,\dots,\cD_T$ are known, the Neyman--Pearson lemma~\citep{neyman1933ix} establishes that a Likelihood Ratio (LR) test is optimal for this simple hypothesis testing problem. Computing the LR associated with that test, under some assumptions on the distributions, is the main approach that we advocate in this paper, leading to $\SeMIstar$ for the empirical mean mechanism (Section \ref{sec:sequential_test}) and to $\SeMISGD$ for the DP-SGD mechanism (Section \ref{sec:practical_attack}).

However, we also explore SeMI attacks that are agnostic to the insertion time but can exploit the sequential nature of data collection. Such attacks $\cA(z^*,o_{1:T})$ only depend on $z^*$ and on the output of the mechanism. We first propose to compute a Generalized Likelihood Ratio (GLR) to test $\mathbf{H_0}$ against $ \cup_t\mathbf{H}_1^{t}: (\exists \tau : D \sim \cP_{1}^{\tau})$
leading to the $\SeMImax$ and $\SeMISGDmax$ attacks in our two settings. Finally, for the empirical mean mechanism, we also compute $\SeMIunif$, which is defined as a Likelihood Ratio test for $ \mathbf{H_0}$ against $\widetilde{\mathbf{H}_1}: \ ( D \sim \widetilde{\cP_{1}})$
where $\widetilde{\cP_1}$ is defined as: sample $\tau \sim \cU(\{1,\dots,T\})$ and then given $\tau$ sample $D$ from $\cP_1^{\tau}$.

In the next section, we see that both types of approaches lead to tighter bounds on $\varepsilon$ than the optimal MI attack with access to the final model only.

\vspace*{-0.3em}\section{SeMI Attacks for the Empirical Mean Mechanism}\label{sec:optimal_test}\vspace*{-0.3em}

For the empirical mean mechanism with Gaussian data, the likelihood ratio between $\mathbf{H_0}$ and $\mathbf{H_1^\tau}$ depends on the observed sequence only through the batch mean at the insertion step. This isolation property makes the power of the resulting attack, $\SeMIstar$, independent of $\tau$ and of $T$. The same argument carries over to distribution shifts, multiple insertions, and unknown insertion times.

\subsection{Optimal Test: $\SeMIstar$}\label{sec:sequential_test}

We analyze the empirical mean mechanism $\Memp$ with $\cD_t = \cD = \cN(\bs\mu, \bs\Sigma)$ for all $t$, mean vector $\bs\mu \in \RR^d$, and a positive definite covariance matrix $\bs\Sigma \in \RR^{d \times d}$. Each batch $D_t$ contains $n$ i.i.d.\ samples from $\cD$.\footnote{Though we fix the batch size for ease of exposition, SeMI tackles varying batch sizes.} At step $t$, $\Memp$ outputs the cumulative empirical mean $\hat{\bs\mu}_t \triangleq \frac{1}{nt} \sum_{j=1}^{t} \sum_{k=1}^{n} \mathbf{X}_{j,k}$, satisfying the recursion $\hat{\bs\mu}_t = (1-\frac{1}{t})\hat{\bs\mu}_{t-1} + \frac{1}{t}\bar{\mathbf{X}}_t$ with the $t$-th batch mean $\bar{\mathbf{X}}_t \triangleq \frac{1}{n}\sum_{k=1}^n \mathbf{X}_{t,k}$.

\begin{restatable}[$\SeMIstar$ Score: Multivariate Log Likelihood Ratio (LR)]{theorem}{multivariateloglr}    \label{thm:multivariate-log-lr}
    Let $\mathbf{N_\tau} \triangleq \bar{\mathbf{X}}_\tau - \bs\mu$, and $c(n) \triangleq \frac{n}{n-1}$. The log-LR for testing $\mathbf{H_0}$ against $\mathbf{H_1^\tau}$ is
    \begin{align}\label{eq:semistar}
        \log \mathrm{LR}_\tau & = \frac{d}{2}\log c(n) - \frac{c(n)}{2}\mathbf{N_\tau}^\top \bs\Sigma^{-1} \mathbf{N_\tau} + c(n) \mathbf{N_\tau}^\top \bs\Sigma^{-1}(\mathbf{z}^* - \bs\mu) - \frac{m^*}{2(n-1)},\vspace*{-1em}
    \end{align}
    where $m^* = \|\mathbf{z}^* - \bs\mu\|_{ \bs\Sigma^{-1}}^2$ is the squared Mahalanobis distance of the target from the population mean.
\end{restatable}
\vspace*{-.5em}

The test statistic depends on the target only through its Mahalanobis distance $m^*$. Targets farther from the mean in Mahalanobis distance are more detectable. $\SeMIstar$ recovers the asymptotic analysis of~\cite{azizeTargetsAreHarder2025} in the regime $d, n \to \infty$ with $d/n = \mathcal{O}(1)$. Our finite-sample expression  \eqref{eq:semistar} adds a quadratic term in $\mathbf{N_\tau}$ and a normalizer term ($\log c(n)$) that vanish as $n \to \infty$. We call $\SeMIstar$ the attack that thresholds \eqref{eq:semistar}. Appendix~\ref{app:multivariate} gives the full derivation and error analysis. %

\textbf{Isolation Property and Power of $\SeMIstar$}. The full LR over all the observations factorizes as a product of conditional densities over all transitions (Lemma~\ref{lem:lr-simplification}). For $t \neq \tau$, the conditional distribution of $\hat{\mu}_t$ given $\hat{\mu}_{t-1}$ is identical under both hypotheses: the target either has not yet been inserted ($t < \tau$) or its effect is absorbed into $\hat{\mu}_{t-1}$ ($t > \tau$). Those terms cancel, leaving the transition at $\tau$.
\begin{restatable}[Isolation Property]{lemma}{lrsimplification}
    \label{lem:lr-simplification}
    The LR for testing $\mathbf{H_0}$ against $\mathbf{H_1^\tau}$ based on $(\hat{\mu}_1, \ldots, \hat{\mu}_T)$ satisfies
    \begin{equation*}
        \frac{p(\hat{\mu}_1, \ldots, \hat{\mu}_T \mid \mathbf{H_1^\tau})}{p(\hat{\mu}_1, \ldots, \hat{\mu}_T \mid \mathbf{H_0})} = \frac{p(\hat{\mu}_\tau \mid \hat{\mu}_{\tau-1}, \mathbf{H_1^\tau})}{p(\hat{\mu}_\tau \mid \hat{\mu}_{\tau-1}, \mathbf{H_0})}\,.
    \end{equation*}
\end{restatable}
Specifically, for $d=1$ with $\cD = \cN(\mu, \sigma^2)$, the conditional distributions at time $\tau$ become
\begin{align*}
    \hat{\mu}_\tau \mid \hat{\mu}_{\tau-1}, \mathbf{H_0}       & \sim \cN\left(\tfrac{\tau-1}{\tau} \hat{\mu}_{\tau-1} + \tfrac{\mu}{\tau}, \tfrac{\sigma^2}{\tau^2 n}\right), \\
    \hat{\mu}_\tau \mid \hat{\mu}_{\tau-1}, \mathbf{H_1^\tau} & \sim \cN\left(\tfrac{\tau-1}{\tau} \hat{\mu}_{\tau-1} + \tfrac{(n-1)\mu + z^*}{\tau n},\tfrac{\sigma^2}{\tau^2 n c(n)}\right).
\end{align*}
\textbf{Observations.} 1.\textit{ Mean shift.} Under $\mathbf{H_1^\tau}$ the mean shifts by $(z^*-\mu)/(\tau n)$ and the variance drops by $(n-1)/n$, since one sample is fixed.
2. \textit{Isolation.} Given consecutive observations $(\hat{\mu}_{\tau-1}, \hat{\mu}_\tau)$, the batch mean is recovered exactly as $\bar{X}_\tau = \tau \hat{\mu}_\tau - (\tau-1)\hat{\mu}_{\tau-1}$. The likelihood ratio depends on the data only through $\bar{X}_\tau$, independent of the batches before and after and of $T$ itself.

We further analytically present the Type I and Type II errors of $\SeMIstar$ for $d=1$ (Appendix~\ref{app:optimal-test-proofs}).

\begin{restatable}[$\SeMIstar$'s Type I and Type II Errors]{lemma}{typeisequential}\label{lem:type1-sequential}
    Define $\gamma_{\max} = \frac{1}{2}\left[\frac{(z^* - \mu)^2}{\sigma^2} + \log\left(c(n)\right)\right]$.
    (a) For $\gamma > \gamma_{\max}$, $\alpha(\gamma) = 0$.
    (b) For $\gamma \leq \gamma_{\max}$, we get $        \alpha(\gamma) = \Phi(a + b(\gamma)) - \Phi(a - b(\gamma))$ and $\beta(\gamma)   = \Phi\Big(\frac{a}{\sqrt{c(n)}} - b(\gamma)\sqrt{c(n)}\Big) + \Phi\Big(-\frac{a}{\sqrt{c(n)}} - b(\gamma)\sqrt{c(n)}\Big)$, where $m^* = (z^*-\mu)^2/\sigma^2$, $a = \sqrt{m^* n}$, and $b(\gamma) = \sqrt{(n-1)\left(m^* + \log c(n) - 2\gamma\right)}$.
\end{restatable}
Both errors depend on $n$ and the signal strength $m^*$ but neither on $\tau$ nor on $T$. This $\tau$-invariance follows from the isolation property: the test operates on the recovered batch mean $\bar{X}_\tau$, whose distribution does not depend on when insertion occurred.

\textbf{Cost of Observing the Final Release vs. the Sequential Releases.} A classic adversary sees only $\hat{\mu}_T$ and tests the membership hypotheses from a single observation. Here, the LR takes the same quadratic form as $\SeMIstar$ with $n$ replaced by $nT$ because the single release pools $nT$ samples, among which one is the target.
Thus, the per-sample contribution shrinks by a factor of $T$. We call this baseline the \emph{Final Observation} test, and denote its errors $\alpha_{\mathrm{FO}}, \beta_{\mathrm{FO}}$.
\begin{restatable}[Final Observation MI's Type I and Type II Errors]{lemma}{typeisingle}
    \label{lem:type1-single}
    With $\gamma_{\max}^{(T)} = \frac{1}{2}\left[m^* + \log\left(c(nT)\right)\right]$ and $\gamma \leq \gamma_{\max}^{(T)}$,
    $\alpha_{\mathrm{FO}}(\gamma) = \Phi(a_T + b_T(\gamma)) - \Phi(a_T - b_T(\gamma))$, and $\beta_{\mathrm{FO}}(\gamma) = \Phi\left(a_T/\sqrt{c(nT)} - b_T(\gamma)\sqrt{c(nT)}\right) + \Phi\left(-a_T/\sqrt{c(nT)} - b_T(\gamma)\sqrt{c(nT)}\right)$,
    with $a_T = \sqrt{m^* nT}$ and $b_T(\gamma) = \sqrt{(nT-1)\left(m^* + \log(c(nT)) - 2\gamma\right)}$.
\end{restatable}
The target influence $(z^*-\mu)/(nT)$ vanishes as $T$ grows, so the power of Final Observation degrades with the number of training steps. $\SeMIstar$ operates on a recovered batch of size $n$ and keeps its power constant in $T$: sequential access is what prevents the signal from diluting.

\subsection{Relaxing the Assumptions of $\SeMIstar$}\label{sec:variations}\vspace*{-0.3em}

Each weakening of the known-$\tau$, stationary, and single-insertion setting leads to the same isolation argument applied to a modified transition. Table~\ref{tab:semi_variants} summarizes the resulting tests.

\textbf{Unknown $\tau$.} We propose two adversaries that do not know $\tau$ but are still grounded in hypothesis testing but consider different alternatives: $\bigcup_t \mathbf{H_1^t}$ for $\SeMImax$ and $\widetilde{\mathbf{H_1}}$ for $\SeMIunif$. Details on their computations are given in Appendices~\ref{app:uniform-prior-proofs} and~\ref{app:unknown-tau-proofs}.

\textbf{Distribution Shift.} When $\cD_t = \cN(\mu_t, \sigma_t^2)$ varies with $t$, the same cancellation argument applies: for $t \neq \tau$ the conditional of $\hat{\mu}_t$ given $\hat{\mu}_{t-1}$ is identical under both hypotheses regardless of $(\mu_t, \sigma_t^2)$. The log-LR depends only on $(\mu_\tau, \sigma_\tau^2)$ and the power is unchanged when the auditor knows these at the insertion step (Appendix~\ref{app:shift-robustness-proof}). Intuitively, isolation makes the test blind to every distribution but the one at $\tau$, so the effect of shifts at other steps cancels.

\textbf{Multiple Insertion Times.} The same cancellation extends to a known set $S \subseteq [T]$ of insertion steps, with the alternative $\mathbf{H_1^\tau}$ replaced by a multi-insertion alternative $\mathbf{H_1^S}$.
We defer the formal definition and the resulting factorization $\log\mathrm{LR}_S = \sum_{t \in S} \log\mathrm{LR}_t$ to Appendix~\ref{app:multi-insertion}.

\begin{table}[h!]
    \caption{SeMI variants. All attacks threshold a log-LR statistic on the observed snapshot sequence and rely on the isolation property of Lemma~\ref{lem:lr-simplification}. $\mathrm{LR}_t$ is the single-time likelihood ratio of Theorem~\ref{thm:multivariate-log-lr}. The SGD variant is introduced in Section~\ref{sec:practical_attack}. (NP = Neyman--Pearson)}\label{tab:semi_variants}
    \resizebox{\textwidth}{!}{
        \begin{tabular}{llll}
            \toprule
            Variant     & Knowledge of $\tau$    & Test statistic                               & Optimality                                    \\
            \midrule
            $\SeMIstar$ & known $\tau$           & $\log \mathrm{LR}_\tau$                      & NP-optimal against $\mathbf{H_1^\tau}$        \\
            $\SeMIunif$ & $\tau \sim \cU([T])$   & $\log\frac{1}{T}\sum_{t=1}^{T}\mathrm{LR}_t$ & NP-optimal against $\widetilde{\mathbf{H_1}}$ \\
            $\SeMImax$  & unknown                & $\max_{t}\log \mathrm{LR}_t$                 & GLR against $\bigcup_{t}\mathbf{H_1^t}$       \\
            $\SeMISGD$  & known $\tau$, (DP-)SGD & Thm.~\ref{thm:sgd-log-lr}                    & NP-optimal for Gaussian-gradient approx.      \\
            \bottomrule
        \end{tabular}}\vspace*{-0.5em}
\end{table}

\vspace*{-0.5em}\subsection{Numerical Illustration: SeMI versus Final Observation MI}%

We sample a base dataset of size $N = 20$ from $\cN(0,1)$ and form $T = 10$ batches of size $n = 10$ by subsampling with replacement from this base dataset, so consecutive batches can share entries and a single batch can contain duplicates. The target $z^*$ (Mahalanobis distance $5$) is inserted into batch $\tau = 5$ only. This subsampled-with-replacement protocol departs from the i.i.d.\ disjoint-batch assumption underlying $\SeMIstar$. The figures below show that the sequential advantage carries through. Errors are estimated from $R = 50\,000$ independent runs.

1. Figure~\ref{fig:lr_evolution} illustrates the \emph{isolation} property: the $\SeMIstar$ statistic exhibits a sharp step at $\tau$ that remains roughly constant in $T$, while the Final Observation statistic peaks around $\tau$ and decays as $T$ grows.

2. Figure~\ref{fig:exp2_advantage} quantifies the \emph{power gap}. $\SeMIstar$ keeps near-constant power across $T$, while $\SeMIunif$ and $\SeMImax$ degrade with $T$ but stay above Final Observation. The gap widens with the target's Mahalanobis distance.

The matching plots under the disjoint-batch protocol that exactly satisfies the test assumptions, ROC curves, and the multivariate generalization are reported in Appendices~\ref{app:exp-disjoint},~\ref{app:exp-unknown-tau}, and~\ref{app:exp-multivariate}.

\begin{figure*}[t!]
    \centering
    \begin{minipage}{0.495\textwidth}
        \begin{subfigure}{0.49\columnwidth}
            \includegraphics[width=\linewidth]{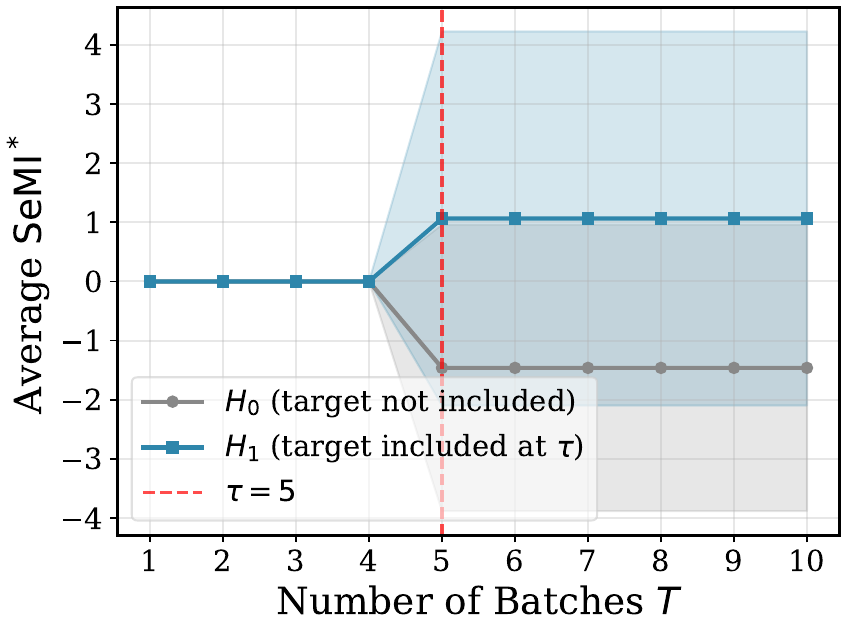}
            \caption{ $\SeMIstar$ statistic}
        \end{subfigure}
        \hfill
        \begin{subfigure}{0.49\columnwidth}
            \includegraphics[width=\linewidth]{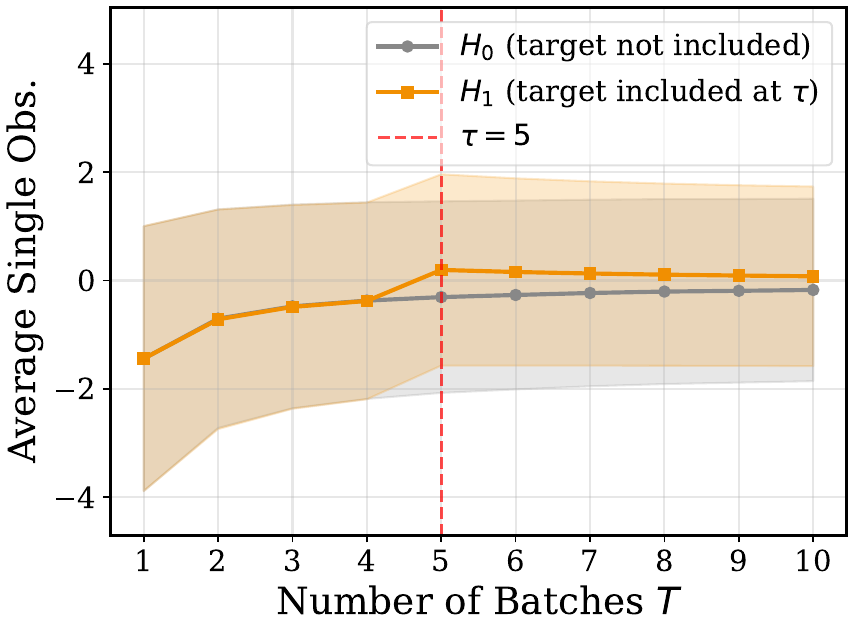}
            \caption{ Final Obs. statistic}
        \end{subfigure}
        \caption{Log-likelihood ratio as a function of observed snapshots ($n=10$, $\tau=5$). (a) $\SeMIstar$: separation appears at $\tau$ and remains constant. (b) Final Observation: separation appears at $\tau$ and decreases as $T$ grows.}
        \label{fig:lr_evolution}
    \end{minipage}
    \hfill
    \begin{minipage}{0.495\textwidth}
        \centering
        \begin{subfigure}[t]{0.49\columnwidth}
            \includegraphics[width=\linewidth]{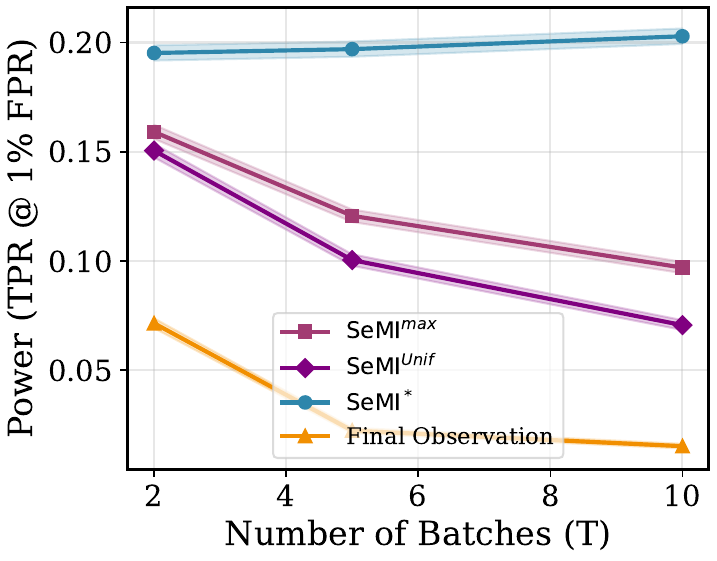}
            \caption{Power vs $T$}
        \end{subfigure}
        \hfill
        \begin{subfigure}[t]{0.49\columnwidth}
            \includegraphics[width=\linewidth]{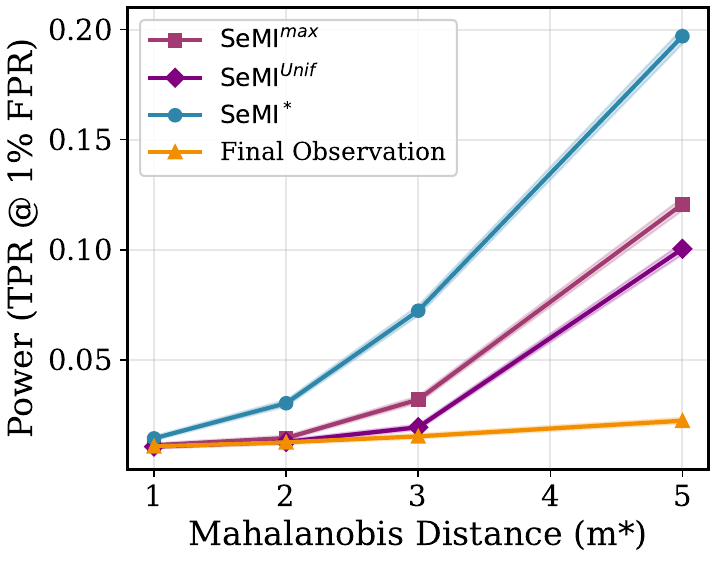}
            \caption{Power vs $z^\ast$ strength}
        \end{subfigure}\vspace*{-.5em}
        \caption{Comparison of sequential attacks ($n=10$). (a) Power (TPR at 1\% FPR) vs number of snapshots $T$ for $z^*$ with Mahalanobis distance 5. (b) Power vs target Mahalanobis distance as strength of chosen $z^\ast$ ($T=5$).}
        \label{fig:exp2_advantage}
    \end{minipage}%
\end{figure*}

\section{White- and Black-box SeMI Attacks for Gradient Descent}\label{sec:practical_attack}\vspace*{-0.5em}
The empirical mean analysis establishes the sequential advantage in an idealized setting. To audit real ML models, we need attacks that apply to gradient descent. Prior sequential attacks on SGD rely on heuristic score combinations~\citep{jagielskiHowCombineMembershipInference2023}. We design a principled LR-based white-box attack and a loss-based black-box attack for SGD that exploit the isolation property observed in the empirical mean case.

\textbf{I. White-box Attack.} In the white-box setting, we observe model parameters $\theta_0, \theta_1, \ldots, \theta_T$ with $\theta_t \in \RR^d$. At each step $t$, $\theta_t = \theta_{t-1} - \eta_t g_t$, where $\eta_t > 0$ is the learning rate and $g_t = \frac{1}{n_t}\sum_{x \in D_t} \nabla_\theta \ell(\theta_{t-1}; x)$ is the batch gradient on $D_t$ of size $n_t$. As in Section~\ref{sec:optimal_test}, we test $\mathbf{H_0}$ (target $z^*$ not in the training data) against $\mathbf{H_1^\tau}$ (one sample in $D_\tau$ replaced by $z^*$).

We assume that, conditioned on $\theta_{t-1}$, the batch gradient $g_t$ is Gaussian with mean $\mu_g(\theta_{t-1})$ and covariance $\Sigma_g(\theta_{t-1})/n_t$. This approximation can be justified by the central limit theorem when the batch aggregates gradients from many data points. A poor Gaussian fit costs power but leaves the audit valid, since Lemma~\ref{lem:kairouz-audit} bounds the errors of any test and the lower bound on $\varepsilon$ comes from Clopper--Pearson intervals on measured error rates. Gaussian modeling of attack scores is standard~\citep{carlini2022membership,ye2022enhanced,leemann2023gaussian}, and Appendix~\ref{app:exp-gaussianity} tests it on the batch-gradient means of our experiments.
Under this assumption, we observe that the conditional distributions mirror the empirical mean case. Specifically,
$ \mathbf{H_0}:    \theta_t \mid \theta_{t-1}  \sim \cN\left(\theta_{t-1} - \eta_t \mu_g(\theta_{t-1}), \frac{\eta_t^2 \Sigma_g(\theta_{t-1})}{n_t}\right)$ for all $t\in [T]$, and $\mathbf{H_1^\tau}: \theta_\tau \mid \theta_{\tau-1}\sim \cN\left(\theta_{\tau-1} - \eta_\tau \frac{(n_\tau-1)\mu_g + g^*}{n_\tau}, \frac{\eta_\tau^2(n_\tau-1)\Sigma_g}{n_\tau^2}\right)$.
Here, $g^* \triangleq \nabla_\theta \ell(\theta_{\tau-1}; z^*)$ is the target gradient, and we write $\mu_g$ and $\Sigma_g$ for $\mu_g(\theta_{\tau-1})$ and $\Sigma_g(\theta_{\tau-1})$.
The likelihood ratio further factorizes as in Lemma~\ref{lem:lr-simplification}, where only the transition at time $\tau$ contributes.

\begin{restatable}[$\SeMISGD$ score: Log Likelihood Ratio of Gradients]{theorem}{sgdloglr}
    \label{thm:sgd-log-lr}
    Define $\delta_g = \nabla_\theta \ell(\theta_{\tau-1}; z^*) - \mu_g(\theta_{\tau-1})$ and $N = \theta_\tau - \theta_{\tau-1} + \eta_\tau \mu_g(\theta_{\tau-1})$. The log likelihood ratio is:
    \begin{align*}
        \log \mathrm{LR} & = -\frac{d}{2}\log\left(\frac{n_\tau-1}{n_\tau}\right) - \frac{n_\tau}{2(n_\tau-1)\eta_\tau^2}N^\top \Sigma_g^{-1} N  - \frac{n_\tau}{(n_\tau-1)\eta_\tau}N^\top \Sigma_g^{-1}\delta_g - \frac{m^*}{2(n_\tau-1)}\,,
    \end{align*}
    where $m^* \triangleq \delta_g^\top \Sigma_g^{-1} \delta_g$ is the squared Mahalanobis distance of the target gradient from the population mean.
\end{restatable}
\textbf{Consequences.} \textit{1. Mahalanobis Scores:} The test statistic $\SeMISGD$ has the same structure as Theorem~\ref{thm:multivariate-log-lr}, with the batch mean replaced by the scaled parameter update $N/\eta_\tau$. $\delta_g$ measures how much the target's gradient deviates from the population mean of the gradient: targets whose gradients differ substantially from typical gradients are more detectable. The Mahalanobis distance $m^*$ accounts for the gradient covariance structure, giving less weight to directions with high variance.

2. \textit{Isolation Property.} From consecutive parameter observations $(\theta_{\tau-1}, \theta_\tau)$, the scaled update $N$ can be computed exactly. The likelihood ratio depends only on this quantity, not on earlier or later parameters. This mirrors the isolation property for the empirical mean. Isolation requires the target to touch exactly one transition and the update to depend only on the previous observed state. It fails if the canary persists into later phases (Appendix~\ref{app:multi-insertion}), if the auditor cannot see $\theta_{\tau-1}$, or if the optimizer aggregates several updates or the gradient history.

3. \textit{Extension to DP-SGD~\citep{abadi2016deep}.} For differentially private training via DP-SGD, each update includes clipped gradients and Gaussian noise: $\theta_t = \theta_{t-1} - \eta_t(g_t^{\mathrm{clip}} + \xi_t)$ where $\xi_t \sim \cN(0, \sigma_{\mathrm{DP}}^2 C^2 I/n_t^2)$ with clipping threshold $C$. The noise reduces attack power but preserves the sequential structure and the isolation property. We extend $\SeMISGD$ to DP-SGD by estimating and plugging in three quantities: the population mean gradient $\mu_g(\theta_{\tau-1})$, the gradient covariance $\Sigma_g(\theta_{\tau-1})$, and the target gradient $\nabla_\theta \ell(\theta_{\tau-1}; z^*)$. In our experiments, $\hat\mu_g$ and $\hat\Sigma_g$ are the sample mean and covariance of the unclipped gradients of 1000 reference samples at initialization, and the target gradient is computed at $\theta_{\tau-1}$. Clipping and DP noise therefore do not enter $\hat\Sigma_g$. For larger networks, we estimate the covariance on the last layer only (Appendix~\ref{app:exp-dpsgd-roc}), so its cost scales with that layer and not with the network.

\textbf{II. Black-box Attacks.} In this setting, the auditor queries each released checkpoint and observes only the target loss $\ell(\theta_t; z^*)$. We introduce two per-transition scores at the insertion time:
\begin{equation}
    \Delta_\tau(z^*) \triangleq \ell(\theta_{\tau-1}; z^*) - \ell(\theta_\tau; z^*),
    \qquad
    R_\tau(z^*) \triangleq \frac{\ell(\theta_{\tau-1}; z^*)}{\ell(\theta_\tau; z^*)}.
    \label{eq:loss-score}
\end{equation}
The isolation property of Section~\ref{sec:optimal_test} concentrates the membership signal at the transition $\theta_{\tau-1} \to \theta_\tau$, so a loss-based statistic should contrast that pair. These scores are inspired by~\cite{jagielskiHowCombineMembershipInference2023}: Back-Front contrasts the losses of the first and last models, and Delta thresholds the maximum loss change across consecutive steps. Restricting Back-Front to $(\theta_{\tau-1}, \theta_\tau)$, or evaluating Delta at the inserted step, recovers~\eqref{eq:loss-score}. We call the resulting attacks Loss Diff and Loss Ratio. Taking the maximum of~\eqref{eq:loss-score} over $\tau \in [T]$ recovers the Delta baselines of~\cite{jagielskiHowCombineMembershipInference2023} for the unknown-$\tau$ case.

\textbf{Numerical Results.}
Figure~\ref{fig:exp17_roc_T5} reports ROC curves at $T = 5$ for the attacks introduced above and the baselines of Section~\ref{sec:experiments}. We observe that \textit{without DP noise, $\SeMISGD$ and its max-over-$\tau$ variants display the best power}. Under DP-SGD at $\varepsilon = 8$, \textit{the noise compresses every curve toward the diagonal but $\SeMISGD$ retains a margin over Loss Diff and Loss Ratio} in the low false-positive-rate regime that drives audit tightness. An alternative pretrain-then-finetune setup is reported in Appendix~\ref{app:exp-dpsgd-roc}.
\begin{figure}[t!]
    \centering\vspace*{-0.5em}
    \begin{subfigure}{0.48\textwidth}
        \centering
        \includegraphics[width=0.7\linewidth]{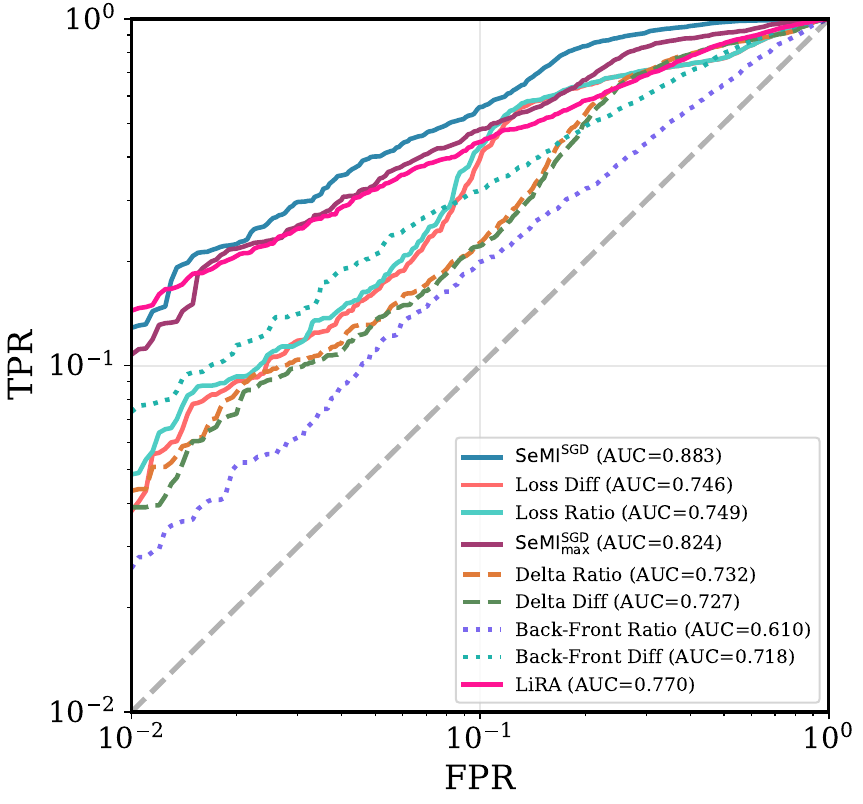}\vspace*{-.5em}
        \caption{Non-private ($\varepsilon = \infty$)}
    \end{subfigure}
    \begin{subfigure}{0.48\textwidth}
        \centering
        \includegraphics[width=0.7\linewidth]{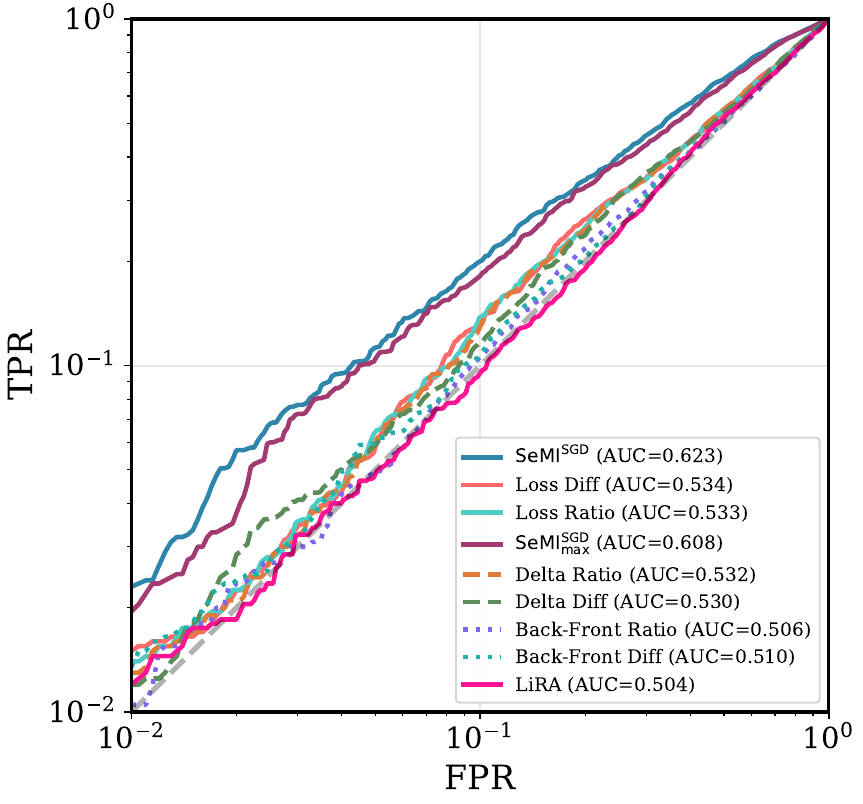}\vspace*{-.5em}
        \caption{Private ($\varepsilon = 8$)}
    \end{subfigure}\vspace*{-.5em}
    \caption{ROC curves at $T = 5$ snapshots on Fashion-MNIST of SeMI variants and attack baselines.}   \label{fig:exp17_roc_T5}
\end{figure}

\vspace*{-0.6em}\section{Sequential Privacy Audits with SeMI}
\label{sec:experiments}\vspace*{-0.5em}

The empirical mean-based experiments in Section~\ref{sec:optimal_test} establish the sequential advantage of $\SeMIstar$ in an idealized setting. Now, we test whether the same advantage transfers to auditing ML models trained with DP-SGD~\citep{abadi2016deep}. We organize the experiments around four questions.
\vspace{-0.2em}
\begin{tcolorbox}[top=2pt,bottom=2pt,left=1pt,right=1pt]
    \textbf{(Q1)} Does observing the full sequence of released snapshots yield tighter $\varepsilon$ lower bounds than observing only the final model, and how does the gain scale with the number of releases $T$? \textbf{(Q2)} At fixed $T$, does the LR-based $\SeMISGD$ outperform heuristic scores built from losses or logits? \textbf{(Q3)} How much tightness is lost when the insertion time $\tau$ is unknown? \textbf{(Q4)} What is the impact of the insertion time $\tau$ on the audit tightness?
\end{tcolorbox}\vspace*{-0.2em}

\textbf{Experimental Setup.} A model is trained for $L$ total DP-SGD steps split into $T$ phases of $L/T$ steps, releasing $\theta_t$ at every phase boundary. The auditor observes $\theta_0,\dots,\theta_T$. We keep the single-step statistic of Theorem~\ref{thm:sgd-log-lr} and absorb the $L/T$ steps of a phase into an effective learning rate $\eta_{\mathrm{eff}} = (L/T)\,q\,\eta$, with $q = B/|D|$ and $D$ the full training pool. The dataset is split into $T$ disjoint subsets and each phase is an independent DP-SGD run: weights carry over while the optimizer state and privacy accountant reset, so by parallel composition $\varepsilon_{\mathrm{phase}} = \varepsilon$. Each round inserts $z^*$ in a single fixed phase $\tau$ and not in any other. We set $\tau = \lfloor T/2 \rfloor + 1$ for the snapshot-count sweep of Observation~1, and $\tau = 5$ for all other experiments at $T = 5$ in Sections~\ref{sec:practical_attack} and~\ref{sec:experiments}; Observation~4 sweeps $\tau$ to isolate its effect.

We consider $T \in \{2, 5, 10\}$, total budget $\varepsilon = 8$, $\delta = 10^{-4}$, clipping norm $C = 1$, and learning rate $\eta = 0.01$. In the main paper, we report softmax regression on Fashion-MNIST ($L = 2500$, $B = 4096$). A fine-tuning variant (pretrain on a public split, then $T=10$ DP-SGD updates on a private split) on Fashion-MNIST, CIFAR-10, and Purchase-100 is reported in Appendix~\ref{app:exp-dpsgd-roc}.
Following~\cite{azizeTargetsAreHarder2025}, we select an in-distribution canary from the data. Among 1000 candidates drawn from the full dataset, we keep the one whose gradient lies farthest from the data distribution in Mahalanobis distance $m^*$, estimated on a reference set disjoint from the candidate pool. Theorems~\ref{thm:multivariate-log-lr} and~\ref{thm:sgd-log-lr} make the statistic depend on the target only through $m^*$, so under the Gaussian models of these theorems this choice maximizes power.
Each audit runs $R = 1900$ rounds of the SeMI game (Algorithm~\ref{alg:mi-game}).
We compute the empirical privacy lower bound $\hat\varepsilon$ via the calibrated-grid protocol of Appendix~\ref{app:audit-cp-grid}: $R_{\mathrm{cal}} = 100$ separate calibration rounds define a grid of $K = 99$ thresholds at the centiles of the test statistic, and we apply Clopper--Pearson upper bounds at level $\xi/(2K)$ at every grid point, taking the maximum lower bound across the grid at overall confidence $1 - \xi = 0.98$.

\textbf{Attack Baselines.}\label{sec:exp_baselines} We compare three types of methods. (a) \emph{Informed sequential} audits use the knowledge of the insertion time: $\SeMISGD$, and the black-box heuristics \emph{Loss Diff} and \emph{Loss Ratio}.
(b) \emph{Uninformed sequential} audits use the full snapshot sequence but are oblivious to $\tau$: \emph{Delta} attacks score the maximum loss change across consecutive updates, $\max_t (\ell(\theta_{t-1}; z^*) - \ell(\theta_t; z^*))$ (Diff) or $\max_t \ell(\theta_{t-1}; z^*) / \ell(\theta_t; z^*)$ (Ratio)~\citep{jagielskiHowCombineMembershipInference2023}.
\emph{Back-Front} attacks score the total change from initial to final model, $\ell(\theta_0; z^*) - \ell(\theta_T; z^*)$ (Diff) or its ratio~\citep{jagielskiHowCombineMembershipInference2023}.
(c) \emph{Snapshot-independent} audits use only the final release: LiRA~\citep{carlini2022membership}.

\noindent\textbf{Observation 1: Sequential Advantage from Released Snapshots.}\label{sec:exp_snapshots}
Figure~\ref{fig:Tsweep} reports $\hat\varepsilon$ versus the number of released snapshots $T$ at fixed total budget $\varepsilon = 8$ and fixed training steps $L$. $\SeMISGD$ dominates every baseline across $T$.

\begin{figure}
    \centering\vspace*{-0.7em}
    \includegraphics[width=.9\linewidth]{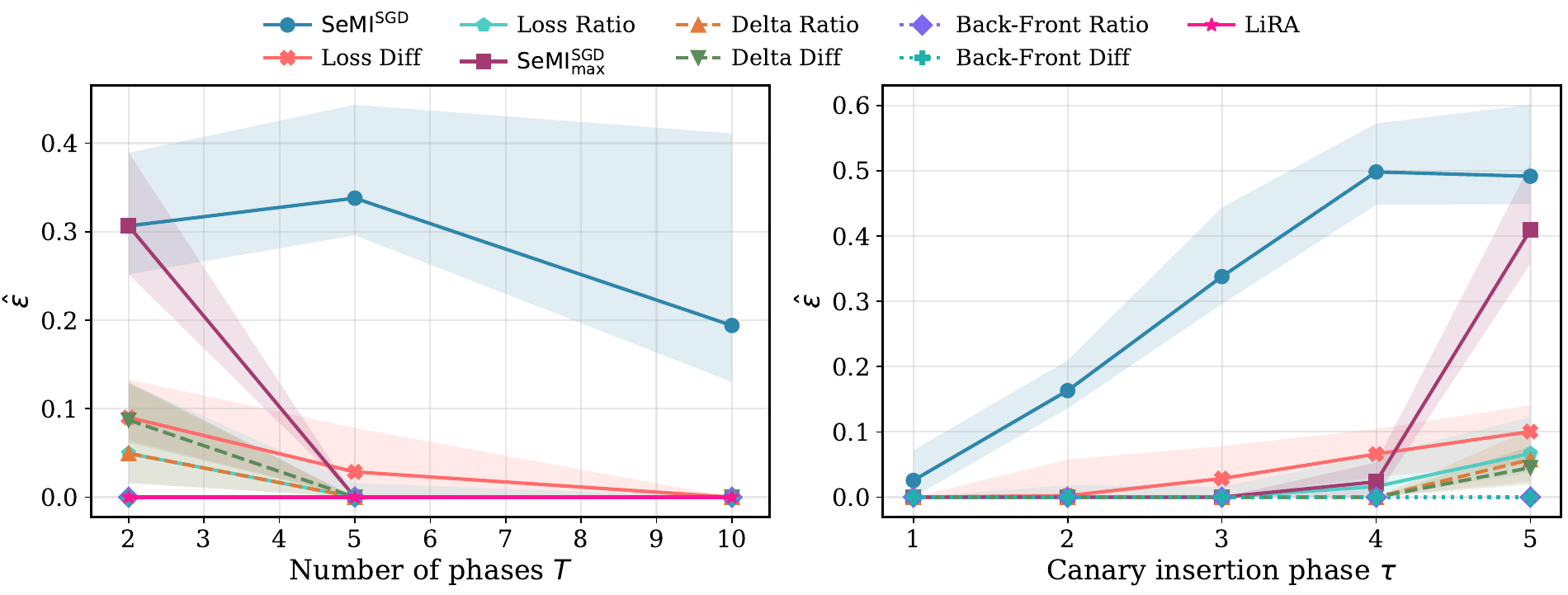}\vspace*{-1em}
    \begin{subfigure}[t]{0.49\linewidth}
        \centering
        \caption{$\hat\varepsilon$ vs.\ \#released snapshots $T$ at fixed $\tau$.}
        \label{fig:Tsweep}
    \end{subfigure}
    \hfill
    \begin{subfigure}[t]{0.49\linewidth}
        \centering
        \caption{$\hat\varepsilon$ vs.\ insertion time $\tau$ at $T=5$, when $\tau$ is fixed.}
        \label{fig:random_vs_fixed}
    \end{subfigure}\vspace*{-.5em}
    \caption{Empirical privacy bound $\hat\varepsilon$ at $\varepsilon = 8$, $\delta = 10^{-4}$ with fixed training steps $L$. Higher is tighter.}  \label{fig:Tsweep_and_fixed_tau}
\end{figure}

\noindent\textbf{Observation 2: Likelihood Ratio vs.\ Heuristic Score Functions.}
\label{sec:exp_finetune}
Figure~\ref{fig:epsilon_distrib} shows violin plots of the empirical lower bound $\hat\varepsilon$ for every attack at $T=5$ and total budget $\varepsilon = 8$. Each violin is obtained by subsampling $90\%$ of the $R = 1900$ audit rounds without replacement and recomputing $\hat\varepsilon$ on each subsample. Among the informed attacks (those that use the insertion time $\tau$), $\SeMISGD$ yields the tightest median bound, and its distribution sits above those of Loss Diff and Loss Ratio. All three informed attacks contrast $\theta_{\tau-1}$ and $\theta_\tau$ to use the isolation property of Section~\ref{sec:optimal_test}. The gap reflects that the white-box LR statistic uses the parameter update, which carries more information than the losses at this pair of snapshots. The same ranking holds at $\varepsilon \in \{1, 2, 4\}$ (Appendix~\ref{app:exp-low-eps}).
\begin{figure}
    \centering
    \includegraphics[width=.7\linewidth]{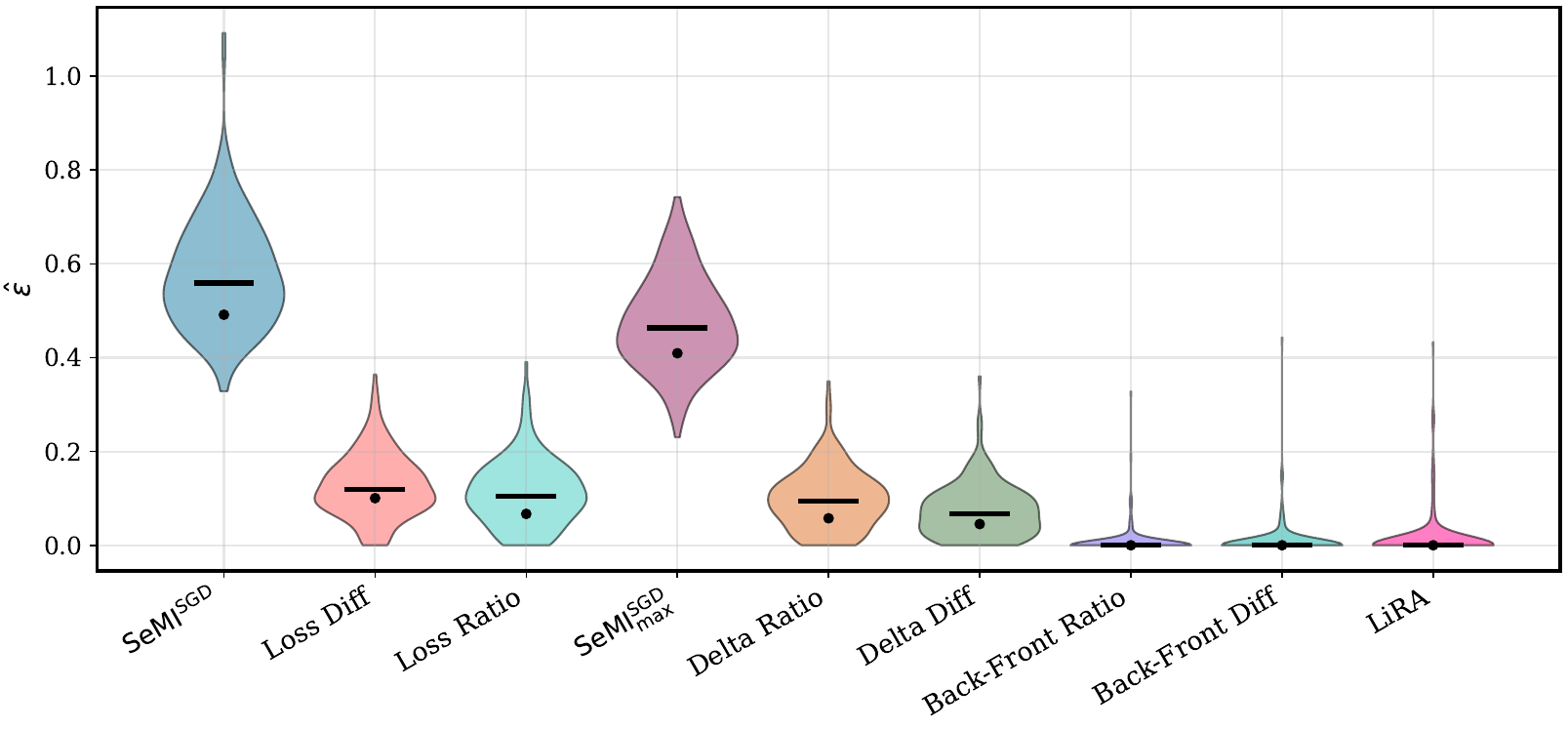}\vspace*{-0.5em}
    \caption{Distribution of empirical epsilon lower bounds at $\delta=10^{-4}$.}\label{fig:epsilon_distrib}
\end{figure}

\noindent\textbf{Observation 3: Knowing the Insertion Time Yields Tighter Audits.}
\label{sec:exp_unknown_tau}
Figure~\ref{fig:epsilon_distrib} pairs each informed attack with its uninformed counterpart: $\SeMISGD$ with $\SeMISGDmax$, Loss Diff with Delta Diff, and Loss Ratio with Delta Ratio. In all three pairs the known-$\tau$ variant gives a tighter $\hat\varepsilon$, mirroring Section~\ref{sec:optimal_test}. The informed attack tests $\mathbf{H_0}$ against $\mathbf{H_1^\tau}$, and when its model holds, no test that ignores $\tau$ has higher power against $\mathbf{H_1^\tau}$ at the same level.

\noindent\textbf{Observation 4: Impact of the Insertion Time on the Audit.}
\label{sec:exp_fixed_tau}
Sweeping $\tau$ across the $T=5$ phases (Figure~\ref{fig:random_vs_fixed}) shows that audit tightness depends strongly on $\tau$. For $\SeMISGD$, later insertion times yield tighter bounds: the per-transition signal is largest once the model has moved away from initialization and the population gradient is small. Choosing a good $\tau$ a priori for a given attack and training protocol remains open. We expect it to depend on the learning-rate schedule, which makes some updates more sensitive than others, on how the attack estimates its score, which can tie it to earlier or later updates, and on the canary, which may stop being an outlier later in training.

\vspace*{-.6em}\section{Discussion and Future Work}\label{sec:conclusion}\vspace*{-.5em}

We derived the optimal MI attack for sequential model observations and characterized its power in closed form under the empirical mean mechanism. The isolation property shows that sequential access recovers the batch statistic at insertion time, removing the dilution that weakens static attacks as training data accumulates. The extension to SGD relies on batch-gradient Gaussianity, yet experiments on DP-SGD training show that $\SeMISGD$ yields tighter privacy lower bounds than heuristic baselines, with the strongest gains in low-privacy regimes.

A promising direction is to adapt single-run audits~\citep{steinkePrivacyAuditingOne2023} to the sequential setting by assigning each canary its own insertion phase. Distribution shift is another aspect of successive model releases warranting further study: $\SeMIstar$ retains its power when $\tau$ is known, but quantifying the effect of an insertion-time distribution chosen by a crafter remains open.

\begin{ack}
  This work has been partially supported by the French National Research Agency (ANR) in the framework of the PEPR IA FOUNDRY project (ANR-23-PEIA-0003). We also acknowledge the Inria-ISI, Kolkata Associate Team ``SeRAI'', and the ANR JCJC for the REPUBLIC project (ANR-22-CE23-0003-01) for partially supporting the project.
\end{ack}

\bibliographystyle{apalike}
\bibliography{references}

\newpage
\appendix

\section{Connection to Literature}\label{app:literature}

\paragraph{Connection to \cite{jagielskiHowCombineMembershipInference2023}.} Appendix B.2 of \cite{jagielskiHowCombineMembershipInference2023} contains a Neyman--Pearson-optimal distinguisher for membership inference on mean estimation, overlapping in spirit with our derivation of $\SeMIstar$. Their analysis considers two released estimates $(\hat\mu_0, \hat\mu_1)$ and yields asymptotic concentration bounds; we treat the full sequence of $T$ snapshots, derive the joint likelihood ratio from which the isolation property follows, and obtain exact non-asymptotic expressions for the Type-I and Type-II errors. Their tests ask whether a target was included in any update, or when it was included; the Appendix B.2 distinguisher addresses only inclusion. Because our auditing procedure inserts the target at a chosen time, our primary test $\SeMIstar$ asks whether insertion occurred at that specific $\tau$. We also address the inclusion question through $\SeMIunif$ (prior on $\tau$) and $\SeMImax$ (no prior). Their likelihood-ratio analysis is confined to mean estimation; we extend the principled derivation to SGD via a batch-gradient Gaussianity approximation, at the cost of an additional modelling assumption that their distribution-free bounds avoid.

\paragraph{Connection to \cite{chang2024efficient}.} \cite{chang2024efficient} propose the free-training attack (FTA) for membership inference in federated learning, thresholding a closed-form (least-squares slope) aggregation of per-sample performance metrics across FL rounds. The statistic is heuristic and carries no likelihood-ratio optimality guarantee. Intuitively, our isolation property explains why slope aggregation across many rounds dilutes the signal that is in fact concentrated at the insertion round. FTA asks only whether a target was ever included, and its analysis is specific to FL with held-out evaluation; $\SeMIstar$'s isolation argument applies to any sequential mechanism satisfying the Gaussian-batch model, at the price of that modelling assumption.

\paragraph{Connection to \cite{nasrTightAuditingDifferentially2023}.} \cite{nasrTightAuditingDifferentially2023} audit DP-SGD by inserting a canary at a chosen batch and showing empirically that observing per-step gradients or model snapshots tightens the resulting $\varepsilon$ estimates. Our analysis explains the empirical gain through the isolation property: consecutive snapshots isolate the per-step update at $\tau$ (exactly in the empirical-mean case, up to the per-step DP noise under DP-SGD). Intuitively, audit power then scales with the per-step noise rather than the aggregate noise over training (Theorem~\ref{thm:sgd-log-lr}). SeMI also covers settings their procedure does not: it does not require per-step gradient access and applies to any sequence of released models, such as public checkpoints or pretrained-then-finetuned releases.

\paragraph{Connection to SeqMIA~\citep{li2024seqmia}.} SeqMIA and SeMI both exploit a sequence of observations, but the sequences come from different places. SeqMIA is black-box on a single released model. It distills the target, computes several metrics on snapshots of the student, and classifies the resulting sequence with an attention-based RNN trained on shadow models, which decides whether the target belongs to the training set. SeMI reads the sequence of models that the mechanism released and tests whether the target was inserted in a given phase. Among our tests, the counterparts of SeqMIA are the black-box Loss Diff and Loss Ratio. Neither attack contains the other, as SeqMIA never sees a genuine intermediate model and our tests do not apply to a single release.

\paragraph{Dot-product scores.} The leading term $N^\top \Sigma_g^{-1} \delta_g$ in Theorem~\ref{thm:sgd-log-lr} is a $\Sigma_g^{-1}$-weighted inner product between the scaled parameter update and the target gradient deviation. This parallels gradient-canary auditing methods such as CANIFE~\citep{maddockCANIFECraftingCanaries2023}, which craft a canary gradient and score its inner product with the observed parameter update; the use of the Mahalanobis weighting connects to the asymptotically optimal MI analysis of~\citet{azizeTargetsAreHarder2025}. The sequential adaptation restricts the inner product to the per-transition update $N$ between $\theta_{\tau-1}$ and $\theta_\tau$, concentrating the score on the snapshot pair where the target's signal sits. $\mathrm{SeMI}^{\mathrm{SGD}}$ derives the precise weighting and target gradient deviation from the Gaussian-gradient model, recovering the heuristic dot-product score as its leading term and giving non-asymptotic guarantees in finite samples rather than only in the limit.\newpage

\section{Proof of the Audit Lemma}
\label{app:audit-lemma-proof}

This appendix proves Lemma~\ref{lem:kairouz-audit} in the sequential setting. %

Note that fixing $\tau$ does not weaken the DP guarantee. By assumption, $\cM$ is $(\varepsilon,\delta)$-DP with respect to the joint dataset $D=\bigcup_t D_t$: inputs differing in one element yield close outputs, regardless of the batch that element sits in. The crafter's choice of $\tau$ isolates a specific neighbouring pair, since under $B=0$ and $B=1$ the datasets differ in one element located in batch $D_\tau$. DP holds for this pair, and Lemma~\ref{lem:kairouz-audit} bounds the errors of \emph{any} test distinguishing the two cases, including tests whose decision rule depends on $\tau$.%

\kairouzaudit*

\begin{proof}
    Fix $z^* \in \cZ$ and $\tau \in [T]$. %

    Sample $X_{t,k} \stackrel{\text{iid}}{\sim} \cD_t$ for $t \in [T]$, $k \in [n]$, and $J \sim \cU([n])$, jointly independent. Set $D_t = (X_{t,1}, \dots, X_{t,n})$ and $D = \bigcup_{t=1}^{T} D_t$. Define $D'$ to coincide with $D$ except that the $J$-th entry of $D_\tau$ is replaced by $z^*$. By construction, $D \sim \cP_0$ and $D' \sim \cP_1^{\tau}$, and the two datasets differ in exactly one element.

    Let $\Phi \triangleq \cA(z^*,\tau,\cdot) \circ \cM$ denote the composed (and possibly randomized) mapping from a dataset to $\widehat{B} \in \{0,1\}$. DP is preserved under post-processing~\citep{dpbook}, so $\Phi$ is $(\varepsilon, \delta)$-DP whenever $\cM$ is.

    By construction, for every realization $\omega$, the pair $(D(\omega), D'(\omega))$ differs in exactly one element (the $J(\omega)$-th entry of batch $\tau$), so it is a valid pair of neighboring datasets. Applying the $(\varepsilon,\delta)$-DP guarantee of $\Phi$ to the event $\{\widehat{B} = 0\}$ at the fixed neighboring pair $(D(\omega), D'(\omega))$ gives, for every $\omega$,
    \begin{equation*}
        \bP_{\Phi}\!\left[\Phi(D(\omega)) = 0\right] \leq e^{\varepsilon}\, \bP_{\Phi}\!\left[\Phi(D'(\omega)) = 0\right] + \delta,
    \end{equation*}
    where the probability is over $\Phi$'s internal randomness only. Taking expectation over $\omega$ (i.e., over the coupling $(D, D', J)$) and using $D \sim \cP_0$, $D' \sim \cP_1^{\tau}$ marginally,
    \begin{equation*}
        \bP_{D \sim \cP_0}\!\left[\Phi(D) = 0\right] \leq e^{\varepsilon}\, \bP_{D' \sim \cP_1^{\tau}}\!\left[\Phi(D') = 0\right] + \delta.
    \end{equation*}

    By definition, $\bP_{\cP_0}[\Phi = 0] = 1 - \alpha(\tau,\cA)$ and $\bP_{\cP_1^{\tau}}[\Phi = 0] = \beta(\tau,\cA)$. Substituting yields
    \begin{equation*}
        1 - \alpha(\tau,\cA) \leq e^{\varepsilon}\, \beta(\tau,\cA) + \delta,
    \end{equation*}
    which rearranges to $\alpha(\tau,\cA) + e^{\varepsilon} \beta(\tau,\cA) \geq 1 - \delta$.

    For the second inequality, apply DP in the reverse direction: the neighboring pair is $(D'(\omega), D(\omega))$ and the event is $\{\widehat{B} = 1\}$. For every $\omega$,
    \begin{equation*}
        \bP_{\Phi}\!\left[\Phi(D'(\omega)) = 1\right] \leq e^{\varepsilon}\, \bP_{\Phi}\!\left[\Phi(D(\omega)) = 1\right] + \delta.
    \end{equation*}
    Taking expectation over $\omega$ gives $\bP_{\cP_1^{\tau}}[\Phi = 1] \leq e^{\varepsilon}\, \bP_{\cP_0}[\Phi = 1] + \delta$, i.e.,
    \begin{equation*}
        1 - \beta(\tau,\cA) \leq e^{\varepsilon}\, \alpha(\tau,\cA) + \delta,
    \end{equation*}
    which rearranges to $\beta(\tau,\cA) + e^{\varepsilon} \alpha(\tau,\cA) \geq 1 - \delta$.
\end{proof}

\section{Audit Protocol and Lower Bound on $\varepsilon$}
\label{sec:audit_protocol}

This appendix details the audit protocol used in our experiments (\Cref{app:audit-cp-grid}) and a calibration-free DKW alternative (\Cref{app:audit-dkw}). Both produce a high-probability lower bound on $\varepsilon$ from rounds of the SeMI game at a fixed insertion time $\tau$, by plugging concentration-based upper confidence bounds on the test errors into~\eqref{eq:lb_epsilon}.

\paragraph{Test family and setup.} All tests that we consider have the form $\cA^{\gamma}(z^*,\tau,o_{1:T}) = \mathds{1}(S > \gamma)$, where $S = s(z^*,\tau,o_{1:T})$ is a real-valued test statistic and $\gamma$ a threshold. We run $R$ rounds of the SeMI game (Algorithm~\ref{alg:crafter}) and record membership indicators $B_1,\dots,B_R$ and statistic values $S_1,\dots,S_R$ rather than the decisions $\widehat{B}^{\gamma}_1,\dots,\widehat{B}^{\gamma}_R$ at a fixed $\gamma$. Let $N_i(R) \triangleq \sum_{r=1}^{R}\mathds{1}(B_r=i)$ count rounds with membership label $i \in \{0,1\}$. For any $\gamma \in \RR$, the empirical Type-I and Type-II errors are
\begin{align*}
    \widehat{\alpha}_{R}(\gamma) =  \frac{1}{N_0(R)}\sum_{r=1}^{R}\mathds{1}(S_r > \gamma)\,\mathds{1}(B_r = 0),
    \qquad
    \widehat{\beta}_{R}(\gamma) = \frac{1}{N_1(R)}\sum_{r=1}^{R}\mathds{1}(S_r \leq \gamma)\,\mathds{1}(B_r = 1).
\end{align*}

\subsection{Clopper--Pearson with a calibrated threshold grid}\label{app:audit-cp-grid}

This is the protocol used in our experiments. It splits the audit budget into a calibration phase and an evaluation phase, so the thresholds at which Clopper--Pearson intervals are computed are independent of the data used to estimate the errors.

\paragraph{Calibration.} We run $R_{\text{cal}}$ independent SeMI rounds and collect the resulting statistic values. The threshold grid $\gamma_1 < \dots < \gamma_K$ is the set of $K$ empirical quantiles of these values at levels $1/(K+1),\,\dots,\,K/(K+1)$. The calibration runs are disjoint from the audit runs.

\paragraph{Pointwise upper bounds.} For $x \in \{0,\dots,m\}$ successes out of $m$ Bernoulli trials and confidence level $\eta \in (0,1)$, the Clopper--Pearson upper limit~\citep{clopper1934use} is
\begin{equation}
    \mathrm{CP}_{\mathrm{up}}(x, m, \eta) \triangleq \inf\bigl\{ p \in [0,1] :\; \bP_{Y \sim \mathrm{Binomial}(m, p)}(Y \leq x) \leq \eta \bigr\}, \label{eq:cp_up}
\end{equation}
with the convention $\mathrm{CP}_{\mathrm{up}}(m, m, \eta) = 1$. Equivalently, $\mathrm{CP}_{\mathrm{up}}(x, m, \eta)$ is the $1-\eta$ quantile of $\mathrm{Beta}(x+1,\,m-x)$. By construction, $\bP_{Y \sim \mathrm{Binomial}(m,p)}(\mathrm{CP}_{\mathrm{up}}(Y, m, \eta) < p) \leq \eta$ for every $p \in [0,1]$ and every $m \geq 1$.

Fix $\xi \in (0,1)$. Conditional on the calibration set and on $(N_0(R), N_1(R))$, the count $N_0(R)\,\widehat{\alpha}_R(\gamma_k)$ is $\mathrm{Binomial}(N_0(R),\, \alpha(\gamma_k))$ and the count $N_1(R)\,\widehat{\beta}_R(\gamma_k)$ is $\mathrm{Binomial}(N_1(R),\, \beta(\gamma_k))$. We apply Clopper--Pearson upper limits at level $\xi/(2K)$,
\begin{align*}
    \overline{\alpha}_{R}^{\xi}(\gamma_k) = \mathrm{CP}_{\mathrm{up}}\bigl(N_0(R)\,\widehat{\alpha}_R(\gamma_k),\,N_0(R),\,\xi/(2K)\bigr), \\[2pt]
    \overline{\beta}_{R}^{\xi}(\gamma_k)  = \mathrm{CP}_{\mathrm{up}}\bigl(N_1(R)\,\widehat{\beta}_R(\gamma_k),\,N_1(R),\,\xi/(2K)\bigr),
\end{align*}
with the convention $\overline{\alpha}_{R}^{\xi}(\gamma_k) = 1$ when $N_0(R) = 0$ and $\overline{\beta}_{R}^{\xi}(\gamma_k) = 1$ when $N_1(R) = 0$, so that the bound is trivially valid on these zero-probability events. By~\eqref{eq:cp_up}, $\bP\bigl(\alpha(\gamma_k) > \overline{\alpha}_R^{\xi}(\gamma_k) \bigm| \text{calibration}\bigr) \leq \xi/(2K)$, and likewise for $\beta(\gamma_k)$.

\paragraph{Audit lower bound.} The factor $1/(2K)$ leaves room for a union bound over the $K$ thresholds and the two error types. Plugging the upper bounds into~\eqref{eq:lb_epsilon} and maximizing over the grid yields
\begin{equation}
    \underline{\varepsilon}_{R}^{\xi,\mathrm{CP}}(\delta) = \log\left(\max_{k \in [K]}\max\left[\frac{1 - \delta - \overline{\alpha}_R^{\xi}(\gamma_k)}{\overline{\beta}_R^\xi(\gamma_k)},\;\frac{1 - \delta - \overline{\beta}_R^\xi(\gamma_k)}{\overline{\alpha}_R^{\xi}(\gamma_k)}\right]\right). \label{eq:lb_epsilon_R_cp}
\end{equation}

\begin{restatable}{lemma}{epsilonlbcp}\label{lem:epsilonlb-cp}
    If the mechanism used in SeMI satisfies $(\varepsilon,\delta)$-DP, then $\bP\bigl(\varepsilon \geq \underline{\varepsilon}_{R}^{\xi,\mathrm{CP}}(\delta)\bigr) \geq 1 - \xi$, where the probability is jointly over the calibration rounds and the audit rounds.
\end{restatable}

\begin{proof}
    Condition on the calibration set, so that $\gamma_1,\dots,\gamma_K$ are fixed. By the Clopper--Pearson coverage property and a union bound over the $2K$ events,
    \[
        \bP\bigl(E \,\bigm|\, \text{calibration}\bigr) \geq 1 - \xi,
        \qquad
        E \triangleq \bigl\{\forall\,k \in [K]:\; \alpha(\gamma_k) \leq \overline{\alpha}_R^{\xi}(\gamma_k),\; \beta(\gamma_k) \leq \overline{\beta}_R^{\xi}(\gamma_k)\bigr\}.
    \]
    The conditional bound holds for every realization of the calibration set, so it holds marginally: $\bP(E) \geq 1 - \xi$. On $E$, fix any $k \in [K]$. Lemma~\ref{lem:kairouz-audit} gives $\alpha(\gamma_k) + e^{\varepsilon}\beta(\gamma_k) \geq 1-\delta$. If $\beta(\gamma_k) > 0$, dividing through and using $\alpha(\gamma_k) \leq \overline{\alpha}_R^{\xi}(\gamma_k)$, $\beta(\gamma_k) \leq \overline{\beta}_R^{\xi}(\gamma_k)$ yields $e^\varepsilon \geq (1-\delta-\overline{\alpha}_R^{\xi}(\gamma_k))/\overline{\beta}_R^{\xi}(\gamma_k)$. If $\beta(\gamma_k) = 0$, the right-hand side is non-positive on $E$ and the inequality holds vacuously. The symmetric inequality of Lemma~\ref{lem:kairouz-audit} treats the case with the roles of $\alpha$ and $\beta$ exchanged. Taking the maximum over $k$ and over the two forms, then applying $\log$, gives $\varepsilon \geq \underline{\varepsilon}_R^{\xi,\mathrm{CP}}(\delta)$ on $E$.
\end{proof}

\subsection{DKW envelope: a calibration-free alternative}\label{app:audit-dkw}

The protocol above pays a $\log K$ factor in each Clopper--Pearson width and requires a separate calibration budget. An alternative we have not seen used in privacy auditing to our knowledge replaces the pointwise bounds and the union bound over $\gamma$ with a uniform DKW envelope, which gives a single confidence band valid for every $\gamma \in \RR$ and removes the need to calibrate.

Given $\xi > 0$, define DKW upper bounds on the errors
\begin{align*}
    \overline{\alpha}_{R}^{\xi,\mathrm{DKW}}(\gamma)  = \widehat{\alpha}_{R}(\gamma) + \sqrt{\frac{\log(2/\xi)}{2N_0(R)}}, \qquad
    \overline{\beta}_{R}^{\xi,\mathrm{DKW}}(\gamma)   = \widehat{\beta}_{R}(\gamma) + \sqrt{\frac{\log(2/\xi)}{2N_1(R)}},
\end{align*}
and the audit lower bound
\begin{equation}
    \underline{\varepsilon}_{R}^{\xi,\mathrm{DKW}}(\delta) =  \log\left(\max_{\gamma \in \RR}\max\left[\frac{1 - \delta - \overline{\alpha}_R^{\xi,\mathrm{DKW}}(\gamma)}{\overline{\beta}_R^{\xi,\mathrm{DKW}}(\gamma)},\;\frac{1 - \delta - \overline{\beta}_R^{\xi,\mathrm{DKW}}(\gamma)}{\overline{\alpha}_R^{\xi,\mathrm{DKW}}(\gamma)}\right]\right). \label{eq:lb_epsilon_R_dkw}
\end{equation}

\begin{restatable}{lemma}{epsilonlbdkw}\label{lem:epsilonlb-dkw}
    If the mechanism used in SeMI satisfies $(\varepsilon,\delta)$-DP, then $\bP\bigl(\varepsilon \geq \underline{\varepsilon}_{R}^{\xi,\mathrm{DKW}}(\delta)\bigr) \geq 1 - \xi$.
\end{restatable}

\begin{proof}
    Let $F_i(\gamma) \triangleq \bP(S_r \leq \gamma \mid B_r = i)$ for $i \in \{0,1\}$ be the population CDFs of the test statistic under each hypothesis, and let
    \begin{equation*}
        \hat F_i(\gamma) \triangleq \frac{1}{N_i(R)}\sum_{r=1}^{R} \mathds{1}(S_r \leq \gamma)\,\mathds{1}(B_r = i)
    \end{equation*}
    be their empirical counterparts on the rounds with $B_r = i$. Then $\alpha(\gamma) = 1 - F_0(\gamma)$ and $\beta(\gamma) = F_1(\gamma)$, with $\widehat\alpha_R(\gamma) = 1 - \hat F_0(\gamma)$ and $\widehat\beta_R(\gamma) = \hat F_1(\gamma)$.

    \textbf{Step 1: uniform concentration.}
    Fix $n_1 \geq 1$. Given $N_1(R) = n_1$, the $n_1$ statistics $\{S_r : B_r = 1\}$ are i.i.d.\ with CDF $F_1$, and $\hat F_1$ reduces on this conditioning event to the empirical CDF of these $n_1$ samples, which we denote $\hat F_1^{n_1}$. By the Dvoretzky--Kiefer--Wolfowitz (DKW) inequality~\citep{massart1990},
    \[
        \bP\!\left(\sup_{\gamma}\bigl(F_1(\gamma) - \hat{F}_1^{n_1}(\gamma)\bigr) > t \;\middle|\; N_1(R) = n_1\right) \leq e^{-2n_1 t^2}.
    \]
    Setting $t = t(n_1) \triangleq \sqrt{\log(2/\xi)/(2n_1)}$ makes the bound equal $\xi/2$. Marginalizing over $N_1(R)$,
    \begin{align*}
        \bP\!\left(\sup_\gamma\bigl(\beta(\gamma) - \widehat{\beta}_R(\gamma)\bigr) > t(N_1(R))\right)
         & = \sum_{n_1=1}^{\infty} \bP(N_1(R) = n_1)\cdot e^{-2n_1 t(n_1)^2}                     \\
         & = \frac{\xi}{2}\,\bP(N_1(R) \geq 1) \leq \frac{\xi}{2}.
    \end{align*}
    (When $N_1(R) = 0$ we set $\overline{\beta}_R^{\xi,\mathrm{DKW}}(\gamma) = 1$. Since $\beta(\gamma) \in [0,1]$, the event $\{\beta(\gamma) > \overline{\beta}_R^{\xi,\mathrm{DKW}}(\gamma)\}$ is empty, and the resulting lower bound on $\varepsilon$ is trivial. The same convention applies to $\overline{\alpha}_R^{\xi,\mathrm{DKW}}(\gamma)$ when $N_0(R) = 0$.)
    Since $\alpha(\gamma) - \widehat{\alpha}_R(\gamma) = \hat{F}_0(\gamma) - F_0(\gamma)$, the same argument applied to the group $\{S_r : B_r = 0\}$ gives $\bP(\sup_\gamma(\alpha(\gamma) - \widehat{\alpha}_R(\gamma)) > t(N_0(R))) \leq \xi/2$. A union bound over the two groups yields
    \[
        \bP(E) \geq 1 - \xi, \qquad
        E \triangleq \bigl\{\forall\,\gamma \in \RR:\; \alpha(\gamma) \leq \overline{\alpha}_R^{\xi,\mathrm{DKW}}(\gamma),\; \beta(\gamma) \leq \overline{\beta}_R^{\xi,\mathrm{DKW}}(\gamma)\bigr\}.
    \]

    \textbf{Step 2: from event $E$ to the lower bound.}
    On $E$, fix any $\gamma \in \RR$. By Lemma~\ref{lem:kairouz-audit},
    \begin{equation}\label{eq:kairouz-pointwise}
        \alpha(\gamma) + e^\varepsilon \beta(\gamma) \geq 1 - \delta.
    \end{equation}
    We split on the sign of $\beta(\gamma)$. If $\beta(\gamma) > 0$, dividing through gives $e^\varepsilon \geq (1-\delta-\alpha(\gamma))/\beta(\gamma)$. On $E$, $\alpha(\gamma) \leq \overline{\alpha}_R^{\xi,\mathrm{DKW}}(\gamma)$ and $\beta(\gamma) \leq \overline{\beta}_R^{\xi,\mathrm{DKW}}(\gamma)$, with $\overline{\beta}_R^{\xi,\mathrm{DKW}}(\gamma) > 0$ whenever $N_1(R) \geq 1$, so
    \begin{equation}\label{eq:lb-from-alphabar}
        e^\varepsilon \;\geq\; \frac{1-\delta-\alpha(\gamma)}{\beta(\gamma)} \;\geq\; \frac{1-\delta-\overline{\alpha}_R^{\xi,\mathrm{DKW}}(\gamma)}{\overline{\beta}_R^{\xi,\mathrm{DKW}}(\gamma)}.
    \end{equation}
    If instead $\beta(\gamma) = 0$, \eqref{eq:kairouz-pointwise} reduces to $\alpha(\gamma) \geq 1-\delta$, and on $E$ the upper bound $\overline{\alpha}_R^{\xi,\mathrm{DKW}}(\gamma) \geq \alpha(\gamma) \geq 1-\delta$ makes the right-hand side of~\eqref{eq:lb-from-alphabar} non-positive. The inequality then holds vacuously, so~\eqref{eq:lb-from-alphabar} is valid for every $\gamma$.

    The symmetric inequality of Lemma~\ref{lem:kairouz-audit} yields $e^\varepsilon \geq (1-\delta-\overline{\beta}_R^{\xi,\mathrm{DKW}}(\gamma))/\overline{\alpha}_R^{\xi,\mathrm{DKW}}(\gamma)$ by the same argument with the roles of $\alpha$ and $\beta$ exchanged. Taking the maximum over $\gamma$ and over the two forms, then applying $\log$, gives $\varepsilon \geq \underline{\varepsilon}_R^{\xi,\mathrm{DKW}}(\delta)$ on $E$. Hence $\bP(\varepsilon \geq \underline{\varepsilon}_R^{\xi,\mathrm{DKW}}(\delta)) \geq \bP(E) \geq 1 - \xi$.
\end{proof}

\section{Proofs for $\SeMIstar$}
\label{app:optimal-test-proofs}

We consider $T$ batches of size $n$ drawn from $\cD = \cN(\mu, \sigma^2)$ with known parameters $(\mu, \sigma^2)$. The target $z^* \in \RR$ is fixed. We test $\mathbf{H_0}$ against $\mathbf{H_1^\tau}$, where $\mathbf{H_0}$ is the null hypothesis that $z^*$ is not in $D_\tau$ and $\mathbf{H_1^\tau}$ is the alternative hypothesis that $z^*$ is in $D_\tau$. The insertion time $\tau \in \{1, \ldots, T\}$ is assumed known.

\textbf{Notation.} We write $\bar{X}_t = \frac{1}{n}\sum_{k=1}^n X_{t,k}$ for the batch mean and $\hat{\mu}_t = \frac{1}{tn}\sum_{j=1}^t\sum_{k=1}^n X_{j,k}$ for the cumulative empirical mean, which satisfies the recursion $\hat{\mu}_t = \frac{t-1}{t}\hat{\mu}_{t-1} + \frac{1}{t}\bar{X}_t$.

\subsection{Proof of Lemma~\ref{lem:lr-simplification}}
\label{app:proof-lr-simplification}

\lrsimplification*

\begin{proof}
    The likelihood ratio factorizes using the chain rule for conditional densities:
    \begin{equation*}
        \mathrm{LR}_{\tau} = \frac{p(\hat{\mu}_1, \ldots, \hat{\mu}_T \mid \mathbf{H_1^\tau})}{p(\hat{\mu}_1, \ldots, \hat{\mu}_T \mid \mathbf{H_0})} = \prod_{t=1}^{T} \frac{p(\hat{\mu}_t \mid \hat{\mu}_1, \ldots, \hat{\mu}_{t-1}, \mathbf{H_1^\tau})}{p(\hat{\mu}_t \mid \hat{\mu}_1, \ldots, \hat{\mu}_{t-1}, \mathbf{H_0})}.
    \end{equation*}

    The recursive relation $\hat{\mu}_t = \frac{t-1}{t}\hat{\mu}_{t-1} + \frac{1}{t}\bar{X}_t$ implies that, given $\hat{\mu}_{t-1}$, the distribution of $\hat{\mu}_t$ depends only on the distribution of the new batch mean $\bar{X}_t$. Earlier observations $\hat{\mu}_1, \ldots, \hat{\mu}_{t-2}$ provide no additional information once $\hat{\mu}_{t-1}$ is known.

    \textbf{Terms for $t \neq \tau$ cancel.}
    For $t < \tau$, the target $z^*$ has not yet been inserted into any batch. Under both $\mathbf{H_0}$ and $\mathbf{H_1^\tau}$, all samples in batches $D_1, \ldots, D_{t}$ are drawn i.i.d.\ from $\mathcal{D}$. The conditional distribution of $\hat{\mu}_t$ given $\hat{\mu}_{t-1}$ is therefore identical under both hypotheses, and the corresponding factor in the product equals 1.

    For $t > \tau$, the batch $D_t$ contains no target sample under either hypothesis. Although the effect of $z^*$ is present in $\hat{\mu}_{t-1}$ under $\mathbf{H_1^\tau}$, conditioning on $\hat{\mu}_{t-1}$ absorbs this effect. The new batch $\bar{X}_t$ is drawn from $\mathcal{N}(\mu, \sigma^2/n)$ under both hypotheses, so the conditional distributions of $\hat{\mu}_t$ given $\hat{\mu}_{t-1}$ are again identical. These factors also equal 1.

    The only factor that differs between the two hypotheses is the one at $t = \tau$, where the batch $D_\tau$ contains $z^*$ under $\mathbf{H_1^\tau}$ but not under $\mathbf{H_0}$.
\end{proof}

\subsection{Univariate Log Likelihood Ratio}
\label{app:proof-log-lr}

\begin{theorem}[Log Likelihood Ratio, univariate]
    \label{thm:log-lr}
    The log likelihood ratio for univariate samples has the form
    \begin{align*}
        \log \mathrm{LR}_\tau & = -\frac{1}{2}\log\left(\frac{n-1}{n}\right) - \frac{n}{2(n-1)\sigma^2}(\bar{X}_\tau - \mu)^2 + \frac{(z^* - \mu) n}{(n-1)\sigma^2}(\bar{X}_\tau - \mu) - \frac{(z^* - \mu)^2}{2(n-1)\sigma^2}.
    \end{align*}
\end{theorem}

\begin{proof}
    We derive the log likelihood ratio from the conditional distributions at time $\tau$.

    \textbf{Conditional distributions.}
    From the main text, the conditional distributions are:
    \begin{align*}
        \hat{\mu}_\tau \mid \hat{\mu}_{\tau-1}, \mathbf{H_0}      & \sim \cN\left(m_0, v_0\right) \\
        \hat{\mu}_\tau \mid \hat{\mu}_{\tau-1}, \mathbf{H_1^\tau} & \sim \cN\left(m_1, v_1\right)
    \end{align*}
    where
    \begin{align*}
        m_0 & = \frac{\tau-1}{\tau} \hat{\mu}_{\tau-1} + \frac{\mu}{\tau},              & v_0 & = \frac{\sigma^2}{\tau^2 n},        \\
        m_1 & = \frac{\tau-1}{\tau} \hat{\mu}_{\tau-1} + \frac{(n-1)\mu + z^*}{\tau n}, & v_1 & = \frac{(n-1)\sigma^2}{\tau^2 n^2}.
    \end{align*}

    \textbf{Log likelihood ratio.}
    Let $N = \hat{\mu}_\tau - m_0$. The log likelihood ratio between two Gaussians is:
    \begin{equation*}
        \log \mathrm{LR}_\tau = -\frac{1}{2}\log\left(\frac{v_1}{v_0}\right) + \frac{N^2}{2v_0} - \frac{(N - (m_1-m_0))^2}{2v_1}.
    \end{equation*}
    Expanding $(N - (m_1-m_0))^2$:
    \begin{align*}
        \log \mathrm{LR}_\tau & = -\frac{1}{2}\log\left(\frac{v_1}{v_0}\right) + N^2\left(\frac{1}{2v_0} - \frac{1}{2v_1}\right) + \frac{N(m_1-m_0)}{v_1} - \frac{(m_1-m_0)^2}{2v_1}.
    \end{align*}

    \textbf{Variance ratio and mean difference.}
    The variance ratio is:
    \begin{equation*}
        \frac{v_1}{v_0} = \frac{(n-1)\sigma^2 / (\tau^2 n^2)}{\sigma^2 / (\tau^2 n)} = \frac{n-1}{n}.
    \end{equation*}
    The mean difference is:
    \begin{equation*}
        m_1 - m_0 = \frac{(n-1)\mu + z^*}{\tau n} - \frac{\mu}{\tau} = \frac{(n-1)\mu + z^* - n\mu}{\tau n} = \frac{z^* - \mu}{\tau n}.
    \end{equation*}

    \textbf{Simplification.}
    Using $N = \frac{1}{\tau}(\bar{X}_\tau - \mu)$ and computing the coefficients:
    \begin{align*}
        \frac{1}{2v_0} - \frac{1}{2v_1} & = \frac{\tau^2 n}{2\sigma^2} - \frac{\tau^2 n^2}{2(n-1)\sigma^2} = -\frac{\tau^2 n}{2(n-1)\sigma^2}, \\
        \frac{m_1-m_0}{v_1}             & = \frac{(z^* - \mu)\tau n}{(n-1)\sigma^2},                                                           \\
        \frac{(m_1-m_0)^2}{2v_1}        & = \frac{(z^* - \mu)^2}{2(n-1)\sigma^2}.
    \end{align*}
    Substituting $N\tau = \bar{X}_\tau - \mu$, the $\tau$ factors cancel in the quadratic and linear terms, yielding:
    \begin{equation*}
        \log \mathrm{LR}_\tau = -\frac{1}{2}\log\left(\frac{n-1}{n}\right) - \frac{n}{2(n-1)\sigma^2}(\bar{X}_\tau - \mu)^2 + \frac{(z^* - \mu)n}{(n-1)\sigma^2}(\bar{X}_\tau - \mu) - \frac{(z^* - \mu)^2}{2(n-1)\sigma^2}.
    \end{equation*}
\end{proof}

\subsection{Proof of Lemma~\ref{lem:type1-sequential}}
\label{app:proof-type1-sequential}

\typeisequential*

\begin{proof}
    The test statistic has the quadratic form $\log \mathrm{LR}_\tau = c_0 + c_1 N + c_2 N^2$ where $c_2 < 0$.

    The condition $\log \mathrm{LR}_\tau \geq \gamma$ is equivalent to $c_2 N^2 + c_1 N + (c_0 - \gamma) \geq 0$. Since $c_2 < 0$, this downward-opening parabola is non-negative only when $N$ lies between its two roots (if they exist).

    The discriminant of the quadratic is:
    \begin{equation*}
        \Delta = c_1^2 - 4c_2(c_0 - \gamma) = \frac{\tau^2 n}{\sigma^2(n-1)}\left[\frac{(z^* - \mu)^2}{\sigma^2} - \log\left(\frac{n-1}{n}\right) - 2\gamma\right].
    \end{equation*}

    \textbf{Case 1:} If $\Delta < 0$, the parabola never reaches zero, so $\alpha = 0$. This occurs when $\gamma > \gamma_{\max}$.

    \textbf{Case 2:} For $\Delta \geq 0$, the roots are:
    \begin{equation*}
        N_{1,2} = \frac{z^* - \mu}{\tau} \mp \frac{\sigma\sqrt{n-1}}{\tau\sqrt{n}}\sqrt{\frac{(z^* - \mu)^2}{\sigma^2} - \log\left(\frac{n-1}{n}\right) - 2\gamma}.
    \end{equation*}
    Under $\mathbf{H_0}$, $N \sim \mathcal{N}(0, \sigma^2/(\tau^2 n))$. The Type I error is:
    \begin{equation*}
        \alpha(\gamma) = P_0(N_1 \leq N \leq N_2) = \Phi\left(\frac{N_2}{\sigma/(\tau\sqrt{n})}\right) - \Phi\left(\frac{N_1}{\sigma/(\tau\sqrt{n})}\right).
    \end{equation*}
    Substituting the expressions for $N_1$ and $N_2$ and simplifying (noting that $\tau$ cancels) yields the stated formula with $a = \frac{(z^* - \mu)\sqrt{n}}{\sigma}$ and $b(\gamma) = \sqrt{n-1}\sqrt{\frac{(z^* - \mu)^2}{\sigma^2} - \log\left(\frac{n-1}{n}\right) - 2\gamma}$.

    \textbf{Type II error.} The Type II error is $\beta(\gamma) = P_1(\log \mathrm{LR}_\tau < \gamma) = P_1(N < N_1) + P_1(N > N_2)$. Under $\mathbf{H_1^\tau}$, the batch mean is $\bar{X}_\tau \sim \mathcal{N}\left(\frac{(n-1)\mu + z^*}{n}, \frac{(n-1)\sigma^2}{n^2}\right)$, so
    \begin{equation*}
        N = \frac{1}{\tau}(\bar{X}_\tau - \mu) \sim \mathcal{N}\left(\frac{z^* - \mu}{\tau n}, \frac{(n-1)\sigma^2}{\tau^2 n^2}\right).
    \end{equation*}
    Let $\mu_1 = \frac{z^* - \mu}{\tau n}$ and $\sigma_1 = \frac{\sigma\sqrt{n-1}}{\tau n}$. Then
    \begin{equation*}
        \beta(\gamma) = \Phi\left(\frac{N_1 - \mu_1}{\sigma_1}\right) + \Phi\left(\frac{-(N_2 - \mu_1)}{\sigma_1}\right),
    \end{equation*}
    and substituting the expressions for $N_{1,2}$ gives
    \begin{align*}
        \frac{N_{1,2} - \mu_1}{\sigma_1} = \frac{(z^* - \mu)\sqrt{n-1}}{\sigma} \mp \sqrt{n}\sqrt{\frac{(z^* - \mu)^2}{\sigma^2} - \log\left(\frac{n-1}{n}\right) - 2\gamma},
    \end{align*}
    yielding the stated formula.
\end{proof}

\subsection{Proof of Lemma~\ref{lem:type1-single}}
\label{app:proof-type1-single}

\typeisingle*

\begin{proof}
    The proof follows the same structure as Lemma~\ref{lem:type1-sequential}, with $n$ replaced by $nT$ throughout.

    For the single-observation case, the test statistic is $\log \mathrm{LR} = c_0^{(T)} + c_1^{(T)} W + c_2^{(T)} W^2$ where $W = \hat{\mu}_T - \mu$ and:
    \begin{align*}
        c_0^{(T)} & = -\frac{1}{2}\log\left(\frac{nT-1}{nT}\right) - \frac{(z^* - \mu)^2}{2(nT-1)\sigma^2}, \\
        c_1^{(T)} & = \frac{nT(z^* - \mu)}{(nT-1)\sigma^2},                                                 \\
        c_2^{(T)} & = -\frac{nT}{2(nT-1)\sigma^2}.
    \end{align*}

    The discriminant is:
    \begin{equation*}
        \Delta = \frac{nT}{\sigma^2(nT-1)}\left[\frac{(z^* - \mu)^2}{\sigma^2} - \log\left(\frac{nT-1}{nT}\right) - 2\gamma\right].
    \end{equation*}
    For $\Delta \geq 0$, the roots are:
    \begin{equation*}
        W_{1,2} = \frac{z^* - \mu}{nT} \mp \frac{\sigma\sqrt{nT-1}}{nT}\sqrt{\frac{(z^* - \mu)^2}{\sigma^2} - \log\left(\frac{nT-1}{nT}\right) - 2\gamma}.
    \end{equation*}
    Under $\mathbf{H_0}$, $W \sim \mathcal{N}(0, \sigma^2/(nT))$, giving the stated Type I error.

    \textbf{Type II error.} Under $\mathbf{H_1^\tau}$, one sample among $nT$ is replaced by $z^*$, so
    \begin{equation*}
        W = \hat{\mu}_T - \mu \sim \mathcal{N}\left(\frac{z^* - \mu}{nT}, \frac{(nT-1)\sigma^2}{(nT)^2}\right).
    \end{equation*}
    The Type II error $\beta_{\mathrm{FO}}(\gamma) = P_1(W < W_1) + P_1(W > W_2)$ follows by the same calculation as the sequential case in Lemma~\ref{lem:type1-sequential}, with $n$ replaced by $nT$.
\end{proof}

\subsection{Extension to Multiple Insertion Times}
\label{app:multi-insertion}

The MI game of Section~\ref{sec:mi_tests} inserts the target at a single step $\tau$. The cancellation argument of Lemma~\ref{lem:lr-simplification} extends to a known set of insertion steps, at the cost of redefining the alternative hypothesis. Let $S \subseteq \{1, \ldots, T\}$ be a known set of insertion steps and let $\mathbf{H_1^S}$ denote the hypothesis under which a copy of the target $z^*$ replaces one sample of $D_t$ for each $t \in S$, while each $D_t$ with $t \notin S$ is drawn i.i.d.\ from $\cD$.

\begin{lemma}[Factorization under multiple insertions]
    \label{lem:multi-insertion}
    The likelihood ratio for $\mathbf{H_0}$ against $\mathbf{H_1^S}$ based on $(\hat\mu_1, \ldots, \hat\mu_T)$ satisfies
    \begin{equation*}
        \log \mathrm{LR}_S = \sum_{t \in S} \log \mathrm{LR}_t,
    \end{equation*}
    where $\log \mathrm{LR}_t$ is the single-step log-LR from Theorem~\ref{thm:log-lr}.
\end{lemma}

\begin{proof}
    The likelihood factorizes as $\prod_{t=1}^{T} p(\hat\mu_t \mid \hat\mu_{t-1}, H)$. For $t \notin S$, the conditional at step $t$ matches the $t \neq \tau$ argument of Lemma~\ref{lem:lr-simplification}: either no insertion has occurred by step $t$, or all past insertions are absorbed into $\hat\mu_{t-1}$. Those factors equal $1$. Each $t \in S$ contributes the single-step ratio $\mathrm{LR}_t$, and the logarithm is additive.
\end{proof}

\section{Multivariate Gaussian Extension}
\label{app:multivariate}

We extend the sequential membership inference approach to the multivariate Gaussian setting where datapoints lie in $\RR^d$. We consider $T$ batches of size $n$ drawn from $\cD = \cN(\bs\mu, \bs\Sigma)$ where $\bs\mu \in \RR^d$ and $\bs\Sigma \in \RR^{d \times d}$ is positive definite and known. The target $\mathbf{z}^* \in \RR^d$ and $\tau$ are known. We test $\mathbf{H_0}$ versus $\mathbf{H_1^\tau}$ as defined in \eqref{def:test_tau}.

\textbf{Notation.} Bold symbols denote vectors ($\mathbf{x} \in \RR^d$) and matrices ($\bs\Sigma \in \RR^{d \times d}$). The Mahalanobis distance is $\|\mathbf{x}\|_{\bs\Sigma} = \sqrt{\mathbf{x}^\top \bs\Sigma^{-1} \mathbf{x}}$.

\subsection{Conditional Distributions}

The recursive relation $\hat{\bs\mu}_t = \frac{t-1}{t}\hat{\bs\mu}_{t-1} + \frac{1}{t}\bar{\mathbf{X}}_t$ extends naturally.

\begin{lemma}[Multivariate Conditional Distributions]
    \label{lem:multivariate-conditional}
    The conditional distributions at time $\tau$ are:
    \begin{align*}
        \hat{\bs\mu}_\tau \mid \hat{\bs\mu}_{\tau-1}, \mathbf{H_0} & \sim \cN\left(\frac{\tau-1}{\tau}\hat{\bs\mu}_{\tau-1} + \frac{\bs\mu}{\tau}, \frac{\bs\Sigma}{\tau^2 n}\right),                              \\
        \hat{\bs\mu}_\tau \mid \hat{\bs\mu}_{\tau-1}, \mathbf{H_1^\tau} & \sim \cN\left(\frac{\tau-1}{\tau}\hat{\bs\mu}_{\tau-1} + \frac{(n-1)\bs\mu + \mathbf{z}^*}{\tau n}, \frac{(n-1)\bs\Sigma}{\tau^2 n^2}\right).
    \end{align*}
\end{lemma}

\begin{proof}
    \textbf{Batch mean distributions.}
    Under $\mathbf{H_0}$, the batch mean is $\bar{\mathbf{X}}_\tau \sim \cN(\bs\mu, \bs\Sigma/n)$.

    Under $\mathbf{H_1^\tau}$, the batch contains $(n-1)$ samples from $\cD$ plus $\mathbf{z}^*$:
    \begin{equation*}
        \bar{\mathbf{X}}_\tau = \frac{1}{n}\left[(n-1)\bs\mu_{\text{sample}} + \mathbf{z}^*\right]
    \end{equation*}
    where $\bs\mu_{\text{sample}} \sim \cN(\bs\mu, \bs\Sigma/(n-1))$. The result follows.

    The recursion $\hat{\bs\mu}_\tau = \frac{\tau-1}{\tau}\hat{\bs\mu}_{\tau-1} + \frac{1}{\tau}\bar{\mathbf{X}}_\tau$ is an affine transformation, yielding the stated conditional distributions.
\end{proof}

\subsection{Log Likelihood Ratio}

\begin{theorem}[Multivariate Log Likelihood Ratio]
    \label{thm:multivariate-log-lr-app}
    Define $\mathbf{N_\tau} = \tau(\hat{\bs\mu}_\tau - \frac{\tau-1}{\tau}\hat{\bs\mu}_{\tau-1} - \frac{\bs\mu}{\tau}) = \bar{\mathbf{X}}_\tau - \bs\mu$. The log likelihood ratio is:
    \begin{align*}
        \log \mathrm{LR}_\tau & = -\frac{d}{2}\log\left(\frac{n-1}{n}\right) - \frac{n}{2(n-1)}\mathbf{N_\tau}^\top \bs\Sigma^{-1} \mathbf{N_\tau} \notag                                               \\
                              & \quad + \frac{n}{n-1}\mathbf{N_\tau}^\top \bs\Sigma^{-1}(\mathbf{z}^* - \bs\mu) - \frac{(\mathbf{z}^* - \bs\mu)^\top \bs\Sigma^{-1}(\mathbf{z}^* - \bs\mu)}{2(n-1)}.
    \end{align*}
\end{theorem}

\begin{proof}
    \textbf{Conditional parameters.}
    Let $\mathbf{m}_0, \mathbf{m}_1$ be the conditional means and $\mathbf{V}_0 = \frac{\bs\Sigma}{\tau^2 n}$, $\mathbf{V}_1 = \frac{(n-1)\bs\Sigma}{\tau^2 n^2}$ the conditional covariances.

    \textbf{Covariance ratio and mean difference.}
    We have $\mathbf{V}_1 = \frac{n-1}{n}\mathbf{V}_0$, so $\det(\mathbf{V}_1) = \left(\frac{n-1}{n}\right)^d \det(\mathbf{V}_0)$ and $\mathbf{V}_1^{-1} = \frac{n}{n-1}\mathbf{V}_0^{-1}$. The mean difference is:
    \begin{equation*}
        \mathbf{m}_1 - \mathbf{m}_0 = \frac{(n-1)\bs\mu + \mathbf{z}^*}{\tau n} - \frac{\bs\mu}{\tau} = \frac{\mathbf{z}^* - \bs\mu}{\tau n}.
    \end{equation*}

    \textbf{Log likelihood ratio.}
    The multivariate Gaussian log likelihood ratio is:
    \begin{equation*}
        \log \mathrm{LR}_\tau = -\frac{1}{2}\log\frac{\det(\mathbf{V}_1)}{\det(\mathbf{V}_0)} + \frac{1}{2}\mathbf{L_\tau}^\top \mathbf{V}_0^{-1}\mathbf{L_\tau} - \frac{1}{2}(\mathbf{L_\tau} - (\mathbf{m}_1 - \mathbf{m}_0))^\top \mathbf{V}_1^{-1}(\mathbf{L_\tau} - (\mathbf{m}_1 - \mathbf{m}_0))
    \end{equation*}
    where $\mathbf{L_\tau} = \hat{\bs\mu}_\tau - \mathbf{m}_0 = \frac{1}{\tau}\mathbf{N_\tau}$.
    The determinant term gives $-\frac{d}{2}\log\left(\frac{n-1}{n}\right)$. Expanding the quadratic terms with $\mathbf{V}_0^{-1} = \tau^2 n \bs\Sigma^{-1}$ and simplifying yields the stated result.
\end{proof}

\begin{corollary}[Quadratic Form]
    \label{cor:multivariate-quadratic-form}
    The log likelihood ratio has the form $\log \mathrm{LR}_\tau = c_0 + \mathbf{c}_1^\top \mathbf{N_\tau} + \mathbf{N_\tau}^\top \mathbf{C}_2 \mathbf{N_\tau}$ where:
    \begin{align*}
        c_0          & = -\frac{d}{2}\log\left(\frac{n-1}{n}\right) - \frac{m^*}{2(n-1)}, \\
        \mathbf{c}_1 & = \frac{n}{n-1}\bs\Sigma^{-1}(\mathbf{z}^* - \bs\mu),         \\
        \mathbf{C}_2 & = -\frac{n}{2(n-1)}\bs\Sigma^{-1},
    \end{align*}
    and $m^* = (\mathbf{z}^* - \bs\mu)^\top \bs\Sigma^{-1}(\mathbf{z}^* - \bs\mu)$ is the squared Mahalanobis distance.
\end{corollary}

\subsection{Type I and II Error Analysis}

\begin{lemma}[Distributions of the Test Statistic]
    Under $\mathbf{H_0}$: $\mathbf{N_\tau} \sim \cN\left(\mathbf{0}, \frac{\bs\Sigma}{n}\right)$.

    Under $\mathbf{H_1^\tau}$: $\mathbf{N_\tau} \sim \cN\left(\frac{\mathbf{z}^* - \bs\mu}{n}, \frac{(n-1)\bs\Sigma}{n^2}\right)$.
\end{lemma}

\begin{lemma}[Type I Error Analysis - Diagonal Case]\label{lem:multivariate-type1-diag}
    For the log likelihood ratio test with threshold $\gamma$ in the multivariate Gaussian case with diagonal covariance $\bs\Sigma = \mathrm{diag}(\sigma_1^2, \ldots, \sigma_d^2)$, let:
    \begin{equation*}
        \gamma_{\max} = \frac{1}{2}\left[m^* - d\log\left(\frac{n-1}{n}\right)\right]
    \end{equation*}
    where $m^* = \sum_{j=1}^d \frac{(z^*_j - \mu_j)^2}{\sigma_j^2}$ is the squared Mahalanobis distance. For $\gamma \leq \gamma_{\max}$, the Type I error is:
    \begin{equation*}
        \alpha(\gamma) = \PP_0\left(\|\mathbf{U}\|_{\bs\Sigma^{-1}}^2 \leq R^2(\gamma)\right)
    \end{equation*}
    where $\mathbf{U} \sim \mathcal{N}\left(-(\mathbf{z}^* - \bs\mu), \frac{\bs\Sigma}{n}\right)$ and $R^2(\gamma) = \frac{2(n-1)}{n} \left[\frac{n^2 m^*}{2} - \gamma - \frac{d}{2}\log\left(\frac{n-1}{n}\right)\right]$.
\end{lemma}

\begin{proof}
    With diagonal covariance, the log likelihood ratio sums over dimensions: $\log \mathrm{LR}_\tau = \sum_{j=1}^d \log \mathrm{LR}_{\tau,j}$.
    The sum can be written as $\log \mathrm{LR}_\tau = \mathbf{N_\tau}^\top \mathbf{C}_2 \mathbf{N_\tau} + \mathbf{c}_1^\top \mathbf{N_\tau} + c_0$ with diagonal $\mathbf{C}_2$.

    Substituting the coefficients from Corollary~\ref{cor:multivariate-quadratic-form} (adapted for diagonal $\bs\Sigma$), we complete the square. Let $\mathbf{M_\tau} = \mathbf{N_\tau} - (\mathbf{z}^* - \bs\mu)$. Under $\mathbf{H_0}$, $\mathbf{N_\tau} \sim \mathcal{N}(\mathbf{0}, \bs\Sigma/n)$, so $\mathbf{M_\tau} \sim \mathcal{N}(-(\mathbf{z}^* - \bs\mu), \bs\Sigma/n)$.
    
    The condition $\log \mathrm{LR}_\tau \geq \gamma$ transforms to:
    \begin{equation*}
        \mathbf{M_\tau}^\top \mathbf{C}_2 \mathbf{M_\tau} \geq \gamma - c_0 + \frac{1}{4}\mathbf{c}_1^\top \mathbf{C}_2^{-1} \mathbf{c}_1.
    \end{equation*}
    Substituting the specific values for $c_0, \mathbf{c}_1, \mathbf{C}_2$, the RHS simplifies to $\gamma + \frac{d}{2}\log\left(\frac{n-1}{n}\right) - \frac{n^2 m^*}{2}$.
    
    Since $\mathbf{C}_2 = -\frac{n}{2(n-1)}\bs\Sigma^{-1}$ is negative definite, the inequality reverses:
    \begin{equation*}
        \mathbf{M_\tau}^\top \bs\Sigma^{-1} \mathbf{M_\tau} \leq -\frac{2(n-1)}{n}\left[\gamma + \frac{d}{2}\log\left(\frac{n-1}{n}\right) - \frac{n^2 m^*}{2}\right].
    \end{equation*}
    Letting $\mathbf{U} = \mathbf{M_\tau}$, the result follows. The probability is the mass of a shifted Gaussian within an ellipsoid defined by the Mahalanobis distance.
\end{proof}

\begin{corollary}[Explicit Type I Error]
    \label{cor:multivariate-type1-explicit}
    For diagonal covariance $\bs\Sigma = \mathrm{diag}(\sigma_1^2, \ldots, \sigma_d^2)$, the Type I error is:
    \begin{equation*}
        \alpha(\gamma) = F_{\chi^2_d(n m^*)}\left(n \, R^2(\gamma)\right)
    \end{equation*}
    where $F_{\chi^2_d(\lambda)}$ is the CDF of the non-central chi-squared distribution with $d$ degrees of freedom and non-centrality parameter $\lambda = n m^*$, and $R^2(\gamma)$ is defined in Lemma~\ref{lem:multivariate-type1-diag}.
\end{corollary}

\begin{proof}
    For diagonal $\bs\Sigma$, the squared Mahalanobis norm decomposes as $\|\mathbf{U}\|_{\bs\Sigma^{-1}}^2 = \sum_{j=1}^d U_j^2/\sigma_j^2$. Each component $U_j \sim \cN(-(z_j^* - \mu_j), \sigma_j^2/n)$, so the standardized variable $W_j = \sqrt{n} U_j/\sigma_j \sim \cN(-\sqrt{n}(z_j^* - \mu_j)/\sigma_j, 1)$.
    
    The sum $\sum_{j=1}^d W_j^2 = n \|\mathbf{U}\|_{\bs\Sigma^{-1}}^2$ follows a non-central chi-squared distribution with $d$ degrees of freedom and non-centrality parameter $\lambda = \sum_{j=1}^d n(z_j^* - \mu_j)^2/\sigma_j^2 = n m^*$.
\end{proof}

\section{Robustness to Distribution Shift}
\label{app:shift-robustness-proof}

We consider the setting where each batch may be drawn from a different distribution: $D_t \sim \cN(\mu_t, \sigma_t^2)$ for $t = 1, \ldots, T$. The target $z^* \in \RR$ and the insertion time $\tau$ are fixed. We test $\mathbf{H_0}$ versus $\mathbf{H_1^\tau}$.

\textbf{Notation.} Parameters $(\mu_t, \sigma_t^2)$ may vary with $t$, but only $(\mu_\tau, \sigma_\tau^2)$ at the insertion time affect the test.

\subsection{Main Result}
\label{app:proof-shift-robustness}

\begin{theorem}[Robustness to Distribution Shift]
    \label{thm:shift-robustness}
    The likelihood ratio test for $\mathbf{H_0}: z^* \notin D_\tau$ versus $\mathbf{H_1^\tau}: z^* \in D_\tau$ depends only on $(\mu_\tau, \sigma_\tau^2)$. The distributions of batches $D_t$ for $t \neq \tau$ do not affect the test.
\end{theorem}

\begin{proof}
    \textbf{Likelihood ratio decomposition.}
    The cumulative empirical mean satisfies $\hat{\mu}_t = \frac{t-1}{t}\hat{\mu}_{t-1} + \frac{1}{t}\bar{X}_t$. The likelihood ratio decomposes as:
    \begin{equation*}
        \mathrm{LR} = \prod_{t=1}^{T} \frac{p(\hat{\mu}_t \mid \hat{\mu}_{t-1}, \mathbf{H_1^\tau})}{p(\hat{\mu}_t \mid \hat{\mu}_{t-1}, \mathbf{H_0})}.
    \end{equation*}

    \textbf{Terms for $t \neq \tau$ cancel.}
    For $t < \tau$, the target has not been inserted. Under both hypotheses, batch $D_t$ is drawn from $\mathcal{N}(\mu_t, \sigma_t^2)$. The conditional distributions of $\hat{\mu}_t$ given $\hat{\mu}_{t-1}$ are identical under both hypotheses, so the corresponding factor equals 1.

    For $t > \tau$, batch $D_t$ contains no target under either hypothesis. Although the effect of $z^*$ is present in $\hat{\mu}_{t-1}$ under $\mathbf{H_1^\tau}$, conditioning on $\hat{\mu}_{t-1}$ absorbs this effect. The batch $\bar{X}_t$ is drawn from $\mathcal{N}(\mu_t, \sigma_t^2/n)$ under both hypotheses, so the conditional distributions of $\hat{\mu}_t$ given $\hat{\mu}_{t-1}$ are identical. These factors also equal 1.

    \textbf{The transition at $t = \tau$.}
    The only factor that differs is at $t = \tau$, where the conditional distributions are:
    \begin{align*}
        \hat{\mu}_\tau \mid \hat{\mu}_{\tau-1}, \mathbf{H_0}      & \sim \cN\left(\frac{\tau-1}{\tau}\hat{\mu}_{\tau-1} + \frac{\mu_\tau}{\tau}, \frac{\sigma_\tau^2}{\tau^2 n}\right),                     \\
        \hat{\mu}_\tau \mid \hat{\mu}_{\tau-1}, \mathbf{H_1^\tau} & \sim \cN\left(\frac{\tau-1}{\tau}\hat{\mu}_{\tau-1} + \frac{(n-1)\mu_\tau + z^*}{\tau n}, \frac{(n-1)\sigma_\tau^2}{\tau^2 n^2}\right).
    \end{align*}
    These expressions involve only $(\mu_\tau, \sigma_\tau^2)$. The parameters $(\mu_t, \sigma_t^2)$ for $t \neq \tau$ do not appear in the likelihood ratio.
\end{proof}

\section{$\SeMIunif$}
\label{app:uniform-prior-proofs}

We derive the optimal test for $\mathbf{H_0}$ against $\widetilde{\mathbf{H_1}}$ where $\tau \sim \mathrm{Uniform}(\{1, \ldots, T\})$. As before, $T$ batches of size $n$ are drawn from $\cD = \cN(\mu, \sigma^2)$ with known parameters.

\textbf{Notation.} We write $\mathrm{LR}_t$ for the likelihood ratio at time $t$ (derived in Theorem~\ref{thm:log-lr}) and $\log \mathrm{LR}_t$ for its logarithm. The batch means $\bar{X}_1, \ldots, \bar{X}_T$ are independent under both hypotheses.

\subsection{Derivation of the Likelihood Ratio}
\label{app:uniform-lr-derivation}

\begin{theorem}[$\SeMIunif$ Likelihood Ratio]
    \label{thm:uniform-lr}
    Under $\widetilde{\mathbf{H_1}}$ with $\tau \sim \mathrm{Uniform}(\{1, \ldots, T\})$, the likelihood ratio between $\widetilde{\mathbf{H_1}}$ and $\mathbf{H_0}$ is
    \begin{equation*}
        \mathrm{LR}_{\mathrm{Unif}} = \frac{1}{T} \sum_{t=1}^{T} \mathrm{LR}_t.
    \end{equation*}
\end{theorem}

\begin{proof}
    Under $\widetilde{\mathbf{H_1}}$, the insertion time satisfies $\bP(\tau = t \mid \widetilde{\mathbf{H_1}}) = 1/T$. The marginal likelihood of the observations is:
    \begin{equation*}
        p(\hat{\mu}_1, \ldots, \hat{\mu}_T \mid \widetilde{\mathbf{H_1}}) = \sum_{t=1}^{T} \bP(\tau = t) \cdot p(\hat{\mu}_1, \ldots, \hat{\mu}_T \mid \tau = t, \widetilde{\mathbf{H_1}}).
    \end{equation*}

    Conditioned on $\tau = t$, the data distribution matches $\mathbf{H_1^t}$:
    \begin{equation*}
        p(\hat{\mu}_1, \ldots, \hat{\mu}_T \mid \tau = t, \widetilde{\mathbf{H_1}}) = p(\hat{\mu}_1, \ldots, \hat{\mu}_T \mid \mathbf{H_1^t}).
    \end{equation*}

    Dividing by the null likelihood:
    \begin{align*}
        \mathrm{LR}_{\mathrm{Unif}} & = \frac{p(\hat{\mu}_1, \ldots, \hat{\mu}_T \mid \widetilde{\mathbf{H_1}})}{p(\hat{\mu}_1, \ldots, \hat{\mu}_T \mid \mathbf{H_0})} = \frac{1}{T} \sum_{t=1}^{T} \frac{p(\hat{\mu}_1, \ldots, \hat{\mu}_T \mid \mathbf{H_1^t})}{p(\hat{\mu}_1, \ldots, \hat{\mu}_T \mid \mathbf{H_0})} = \frac{1}{T} \sum_{t=1}^{T} \mathrm{LR}_t.
    \end{align*}
\end{proof}

Each $\mathrm{LR}_t$ depends on the data only through $\bar{X}_t$ (by the isolation property), and batch means are independent. The test statistic aggregates evidence across all time steps.

\subsection{Log-Sum-Exp Form}
\label{app:logsumexp-form}

\begin{corollary}[Log Likelihood Ratio]
    \label{cor:uniform-log-lr}
    The log likelihood ratio is
    \begin{equation*}
        \log \mathrm{LR}_{\mathrm{Unif}} = \mathrm{logsumexp}(\log \mathrm{LR}_1, \ldots, \log \mathrm{LR}_T) - \log T
    \end{equation*}
    where $\mathrm{logsumexp}(x_1, \ldots, x_T) = \log\left(\sum_{t=1}^T e^{x_t}\right)$.
\end{corollary}

\begin{proof}
    Taking the logarithm of Theorem~\ref{thm:uniform-lr}:
    \begin{equation*}
        \log \mathrm{LR}_{\mathrm{Unif}} = \log\left(\frac{1}{T} \sum_{t=1}^{T} \mathrm{LR}_t\right) = \log\left(\sum_{t=1}^{T} e^{\log \mathrm{LR}_t}\right) - \log T.
    \end{equation*}
\end{proof}

The log-sum-exp function is numerically stable: $\mathrm{logsumexp}(x_1, \ldots, x_T) = m + \log\left(\sum_{t=1}^T e^{x_t - m}\right)$ where $m = \max_t x_t$.

\subsection{Relationship to the GLR Test}
\label{app:uniform-vs-glr}

\begin{proposition}[Bounds]
    \label{prop:uniform-bounds}
    The log likelihood ratio satisfies
    \begin{equation*}
        \max_{t} \log \mathrm{LR}_t - \log T \leq \log \mathrm{LR}_{\mathrm{Unif}} \leq \max_{t} \log \mathrm{LR}_t.
    \end{equation*}
\end{proposition}

\begin{proof}
    \textbf{Upper bound.} Since $\mathrm{LR}_t \geq 0$:
    \begin{equation*}
        \frac{1}{T} \sum_{t=1}^{T} \mathrm{LR}_t \leq \max_t \mathrm{LR}_t.
    \end{equation*}

    \textbf{Lower bound.} The average is at least the maximum divided by $T$:
    \begin{equation*}
        \frac{1}{T} \sum_{t=1}^{T} \mathrm{LR}_t \geq \frac{1}{T} \max_t \mathrm{LR}_t.
    \end{equation*}
\end{proof}

When one $\mathrm{LR}_{t^*}$ dominates, $\log \mathrm{LR}_{\mathrm{Unif}} \approx \log \mathrm{LR}_{t^*} - \log T$. When multiple batches contribute comparable evidence, the test accumulates their contributions.

\section{$\SeMImax$}
\label{app:unknown-tau-proofs}

Consider the case where $T$ batches of size $n$ are drawn from $\cD = \cN(\mu, \sigma^2)$ with known parameters $(\mu, \sigma^2)$, but the insertion time is unknown. The target $z^* \in \RR$ is fixed. We test:
\begin{align*}
    \mathbf{H_0} & : \text{All datapoints sampled i.i.d.\ from } \cD \\
    \overline{\mathbf{H_1}} & : \text{One sample in } D_\tau \text{ replaced by } z^* \text{ for some unknown } \tau \in \{1, \ldots, T\}
\end{align*}

\textbf{Notation.} We write $\bar{X}_t$ for the batch mean, $\hat{\mu}_t$ for the cumulative empirical mean, and $\log \mathrm{LR}_\tau$ for the log likelihood ratio at time $\tau$. The test statistic for $\SeMImax$ is $\max_{\tau} \log \mathrm{LR}_\tau$. Single-batch error rates are $\alpha_0(\gamma) = P_0(\log \mathrm{LR}_1 \geq \gamma)$ and $\beta_0(\gamma) = P_1(\log \mathrm{LR}_1 < \gamma)$.

\subsection{Independence of Log Likelihood Ratios under $\mathbf{H_0}$}
\label{app:proof-independence-h0}

\begin{theorem}[Independence under $\mathbf{H_0}$]
    \label{thm:independence-h0}
    Under $\mathbf{H_0}$, the statistics $\log \mathrm{LR}_1, \ldots, \log \mathrm{LR}_T$ are independent and identically distributed.
\end{theorem}

\begin{proof}
    From Theorem~\ref{thm:log-lr}, each $\log \mathrm{LR}_\tau$ depends on the data only through $\bar{X}_\tau - \mu$. The coefficients $c_0$, $c_1$, and $c_2$ in the quadratic form depend on $n$, $\sigma^2$, $\mu$, and $z^*$, but not on $\tau$.

    Under $\mathbf{H_0}$, each batch $D_\tau$ is sampled independently from $\mathcal{N}(\mu, \sigma^2)^n$. The batch means $\bar{X}_1, \ldots, \bar{X}_T$ are therefore independent random variables. Moreover, each $\bar{X}_\tau \sim \mathcal{N}(\mu, \sigma^2/n)$ for all $\tau \in \{1, \ldots, T\}$. Since the log likelihood ratio is a function of $\bar{X}_\tau - \mu$ with coefficients that do not depend on $\tau$, the statistics $\log \mathrm{LR}_1, \ldots, \log \mathrm{LR}_T$ have the same distribution.

    The statistics $\log \mathrm{LR}_\tau$ are functions of independent, identically distributed random variables (the batch means), and these functions are identical across $\tau$. Therefore, $\log \mathrm{LR}_1, \ldots, \log \mathrm{LR}_T$ are i.i.d.\ under $\mathbf{H_0}$.
\end{proof}

\subsection{Type I error}
\label{app:proof-glr-type1}

\begin{restatable}[$\SeMImax$ Type I Error]{corollary}{glrtypei}
	\label{cor:glr-type1}
	The Type I error of $\SeMImax$ with threshold $\gamma$ is
	\begin{equation*}
		\alpha_{\max}(\gamma) = 1 - [1 - \alpha_0(\gamma)]^T
	\end{equation*}
	where $\alpha_0(\gamma)$ is the Type I error of the single-batch test from Lemma~\ref{lem:type1-sequential}.
\end{restatable}

\begin{proof}
    The $\SeMImax$ test rejects when $\max_{\tau=1,\ldots,T} \log \mathrm{LR}_\tau \geq \gamma$. The Type I error is:
    \begin{align*}
        \alpha_{\max}(\gamma) & = P_0(\SeMImax \geq \gamma)                                            \\
                              & = P_0\left(\max_{\tau=1,\ldots,T} \log \mathrm{LR}_\tau \geq \gamma\right) \\
                              & = 1 - P_0(\log \mathrm{LR}_\tau < \gamma \text{ for all } \tau).
    \end{align*}

    By Theorem~\ref{thm:independence-h0}, the statistics are independent under $\mathbf{H_0}$:
    \begin{equation*}
        P_0(\log \mathrm{LR}_\tau < \gamma \text{ for all } \tau) = \prod_{\tau=1}^T P_0(\log \mathrm{LR}_\tau < \gamma).
    \end{equation*}
    Since all $\log \mathrm{LR}_\tau$ have the same distribution under $\mathbf{H_0}$:
    \begin{equation*}
        \prod_{\tau=1}^T P_0(\log \mathrm{LR}_\tau < \gamma) = [P_0(\log \mathrm{LR}_1 < \gamma)]^T = [1 - \alpha_0(\gamma)]^T.
    \end{equation*}
    Therefore:
    \begin{equation*}
        \alpha_{\max}(\gamma) = 1 - [1 - \alpha_0(\gamma)]^T.
    \end{equation*}
\end{proof}
To achieve overall Type I error $\alpha$, set the threshold $\gamma^*$ such that $\alpha_0(\gamma^*) = 1 - (1-\alpha)^{1/T}$. For small $\alpha$ and moderate $T$, this approximates the Bonferroni correction $\alpha_0 \approx \alpha/T$.

\subsection{Type II error}
\label{app:proof-glr-type2}

\begin{restatable}[$\SeMImax$ Type II Error]{corollary}{glrtypeii}
	\label{cor:glr-type2}
	The Type II error of $\SeMImax$, given true insertion time $\tau^*$, is
	\begin{equation*}
		\beta_{\max}(\gamma \mid \tau^*) = \beta_0(\gamma) \cdot [1 - \alpha_0(\gamma)]^{T-1}
	\end{equation*}
	where $\beta_0(\gamma)$ is the Type II error from Lemma~\ref{lem:type1-sequential}. This does not depend on $\tau^*$.
\end{restatable}

\begin{proof}
    Under $\mathbf{H_1^{\tau^*}}$, the target $z^*$ is in batch $\tau^*$. The Type II error is the probability of failing to reject $\mathbf{H_0}$.

    The test fails to reject when all log likelihood ratios are below the threshold:
    \begin{equation*}
        \beta_{\max}(\gamma \mid \tau^*) = P_1(\SeMImax < \gamma \mid \mathbf{H_1^{\tau^*}}) = P_1\left(\max_{\tau} \log \mathrm{LR}_\tau < \gamma\right).
    \end{equation*}
    Separating the true insertion time from others:
    \begin{equation*}
        P_1\left(\max_{\tau} \log \mathrm{LR}_\tau < \gamma\right) = P_1\left(\log \mathrm{LR}_{\tau^*} < \gamma \text{ and } \max_{\tau \neq \tau^*} \log \mathrm{LR}_\tau < \gamma\right).
    \end{equation*}

    Under $\mathbf{H_1^{\tau^*}}$, the statistic $\log \mathrm{LR}_{\tau^*}$ depends only on batch $D_{\tau^*}$, while $\log \mathrm{LR}_\tau$ for $\tau \neq \tau^*$ depends only on batch $D_\tau$. Since the batches are independent, these statistics are independent:
    \begin{equation*}
        P_1\left(\log \mathrm{LR}_{\tau^*} < \gamma \text{ and } \max_{\tau \neq \tau^*} \log \mathrm{LR}_\tau < \gamma\right) = P_1(\log \mathrm{LR}_{\tau^*} < \gamma) \cdot P_1\left(\max_{\tau \neq \tau^*} \log \mathrm{LR}_\tau < \gamma\right).
    \end{equation*}

    The first factor is $\beta_0(\gamma)$, the Type II error of the single-batch test. For batches $\tau \neq \tau^*$, no target is present, so $\log \mathrm{LR}_\tau$ has the null distribution. By the same argument as in Corollary~\ref{cor:glr-type1}:
    \begin{equation*}
        P_1\left(\max_{\tau \neq \tau^*} \log \mathrm{LR}_\tau < \gamma\right) = [1 - \alpha_0(\gamma)]^{T-1}.
    \end{equation*}
    Therefore:
    \begin{equation*}
        \beta_{\max}(\gamma \mid \tau^*) = \beta_0(\gamma) \cdot [1 - \alpha_0(\gamma)]^{T-1}.
    \end{equation*}
    The result does not depend on $\tau^*$ because both $\alpha_0(\gamma)$ and $\beta_0(\gamma)$ are independent of the insertion time.
\end{proof}
The Type II error decomposes as a product: $\beta_0(\gamma)$ is the probability that the test fails to detect the target in the correct batch, and $[1-\alpha_0(\gamma)]^{T-1}$ is the probability that none of the other batches trigger a false detection. Independence from $\tau^*$ follows from the identical distribution of batch means.

\section{SGD Analysis}
\label{app:sgd}

We consider the SGD setting where we observe model parameters $\theta_0, \theta_1, \ldots, \theta_T$ with $\theta_t \in \RR^d$. At each step $t$, the update is $\theta_t = \theta_{t-1} - \eta_t g_t$ where $\eta_t > 0$ is the learning rate and $g_t$ is the batch gradient on dataset $D_t$ of size $n_t$. The target $z^*$ lies in the data domain $\cX$ and the insertion time $\tau$ is known. We test:
\begin{align*}
    \mathbf{H_0}      & : \text{All samples in } D_1, \ldots, D_T \text{ are drawn i.i.d.\ from } \cD \\
    \mathbf{H_1^\tau} & : \text{One sample in } D_\tau \text{ is replaced by } z^*
\end{align*}

We assume that, conditioned on $\theta_{t-1}$, the batch gradient $g_t$ is Gaussian with mean $\mu_g(\theta_{t-1})$ and covariance $\Sigma_g(\theta_{t-1})/n_t$. This approximation is justified by the central limit theorem when the batch aggregates gradients from many data points.

\textbf{Notation.} The loss function on data point $x$ with parameter $\theta$ is denoted by $\ell(\theta; x)$. The batch gradient is $g_t = \frac{1}{n_t}\sum_{x \in D_t} \nabla_\theta \ell(\theta_{t-1}; x)$. We denote by $\mu_g(\theta)$ and $\Sigma_g(\theta)$ the mean and covariance of the batch gradient when scaled by $\sqrt{n_t}$, i.e., $g_t \mid \theta_{t-1} \sim \cN(\mu_g(\theta_{t-1}), \Sigma_g(\theta_{t-1})/n_t)$.

\subsection{Conditional Distributions}

\begin{lemma}[Parameter Distribution Under $\mathbf{H_0}$]
    \label{lem:sgd-h0}
    Under $\mathbf{H_0}$, the conditional distribution of $\theta_t$ given $\theta_{t-1}$ is:
    \begin{equation*}
        \theta_t \mid \theta_{t-1}, \mathbf{H_0} \sim \cN\left(\theta_{t-1} - \eta_t \mu_g(\theta_{t-1}), \frac{\eta_t^2 \Sigma_g(\theta_{t-1})}{n_t}\right).
    \end{equation*}
\end{lemma}

\begin{lemma}[Parameter Distribution Under $\mathbf{H_1^\tau}$]
    \label{lem:sgd-h1}
    Under $\mathbf{H_1^\tau}$, the conditional distribution of $\theta_\tau$ given $\theta_{\tau-1}$ is:
    \begin{equation*}
        \theta_\tau \mid \theta_{\tau-1}, \mathbf{H_1^\tau} \sim \cN\left(\theta_{\tau-1} - \eta_\tau \cdot \frac{(n_\tau-1)\mu_g(\theta_{\tau-1}) + \nabla_\theta \ell(\theta_{\tau-1}; z^*)}{n_\tau}, \frac{\eta_\tau^2(n_\tau-1)\Sigma_g(\theta_{\tau-1})}{n_\tau^2}\right).
    \end{equation*}
    For $t \neq \tau$, the conditional distribution is identical to that under $\mathbf{H_0}$.
\end{lemma}

\begin{proof}
    \textbf{Batch gradient under $\mathbf{H_1^\tau}$.}
    Batch $D_\tau$ contains $(n_\tau - 1)$ samples from $\cD$ and the fixed target $z^*$. The batch gradient is:
    \begin{equation*}
        g_\tau = \frac{1}{n_\tau}\left[\sum_{i=1}^{n_\tau-1} \nabla_\theta \ell(\theta_{\tau-1}; x_i) + \nabla_\theta \ell(\theta_{\tau-1}; z^*)\right].
    \end{equation*}

    Under the Gaussian gradient assumption, the sum of $(n_\tau - 1)$ i.i.d.\ Gaussian gradients has mean $(n_\tau-1)\mu_g(\theta_{\tau-1})$ and covariance $(n_\tau-1)\Sigma_g(\theta_{\tau-1})$. Adding the fixed target gradient shifts the mean.

    The result follows from $\theta_\tau = \theta_{\tau-1} - \eta_\tau g_\tau$.
\end{proof}

\subsection{Likelihood Ratio}

\begin{theorem}[SGD Likelihood Ratio Simplification]
    \label{thm:sgd-lr-simplification}
    The likelihood ratio for testing $\mathbf{H_0}$ against $\mathbf{H_1^\tau}$ based on $(\theta_0, \ldots, \theta_T)$ satisfies:
    \begin{equation*}
        \frac{p(\theta_0, \ldots, \theta_T \mid \mathbf{H_1^\tau})}{p(\theta_0, \ldots, \theta_T \mid \mathbf{H_0})} = \frac{p(\theta_\tau \mid \theta_{\tau-1}, \mathbf{H_1^\tau})}{p(\theta_\tau \mid \theta_{\tau-1}, \mathbf{H_0})}.
    \end{equation*}
\end{theorem}

\begin{proof}
    By the chain rule,
    \begin{equation*}
        \frac{p(\theta_0, \ldots, \theta_T \mid H)}{p(\theta_0)} = \prod_{t=1}^T p(\theta_t \mid \theta_0, \ldots, \theta_{t-1}, H).
    \end{equation*}
    The update $\theta_t = \theta_{t-1} - \eta_t g_t$ implies that $\theta_t$ depends on $(\theta_0, \ldots, \theta_{t-1})$ only through $\theta_{t-1}$, since $g_t$ depends on $\theta_{t-1}$ and the samples in $D_t$. Thus
    \begin{equation*}
        p(\theta_t \mid \theta_0, \ldots, \theta_{t-1}, H) = p(\theta_t \mid \theta_{t-1}, H).
    \end{equation*}
    The likelihood ratio becomes
    \begin{equation*}
        \mathrm{LR} = \prod_{t=1}^T \frac{p(\theta_t \mid \theta_{t-1}, \mathbf{H_1^\tau})}{p(\theta_t \mid \theta_{t-1}, \mathbf{H_0})}.
    \end{equation*}
    For $t < \tau$, the target has not been inserted, so both conditional distributions are identical and the ratio equals 1. For $t > \tau$, batch $D_t$ contains no target under either hypothesis, so conditioned on $\theta_{t-1}$ the distributions coincide and the ratio equals 1. Only the term at $t = \tau$ contributes.
\end{proof}

\begin{theorem}[SGD Log Likelihood Ratio]
    \label{thm:sgd-log-lr-app}
    Define $\delta_g = \nabla_\theta \ell(\theta_{\tau-1}; z^*) - \mu_g(\theta_{\tau-1})$ and $N = \theta_\tau - \theta_{\tau-1} + \eta_\tau \mu_g(\theta_{\tau-1})$. The log likelihood ratio is:
    \begin{align*}
        \log \mathrm{LR} & = -\frac{d}{2}\log\left(\frac{n_\tau-1}{n_\tau}\right) - \frac{n_\tau}{2(n_\tau-1)\eta_\tau^2}N^\top \Sigma_g^{-1} N \notag \\
                         & \quad - \frac{n_\tau}{(n_\tau-1)\eta_\tau}N^\top \Sigma_g^{-1}\delta_g - \frac{m^*}{2(n_\tau-1)}
    \end{align*}
    where $m^* = \delta_g^\top \Sigma_g^{-1} \delta_g$ is the Mahalanobis distance of the target gradient from the population mean. We denote this sequential test for SGD as $\SeMISGD$.
\end{theorem}

\begin{proof}
    \textbf{Conditional parameters.}
    Let $m_0 = \theta_{\tau-1} - \eta_\tau \mu_g(\theta_{\tau-1})$ and $m_1 = \theta_{\tau-1} - \eta_\tau \frac{(n_\tau-1)\mu_g(\theta_{\tau-1}) + \nabla_\theta \ell(\theta_{\tau-1}; z^*)}{n_\tau}$.
    The covariances are $V_0 = \frac{\eta_\tau^2 \Sigma_g}{n_\tau}$ and $V_1 = \frac{\eta_\tau^2 (n_\tau-1)\Sigma_g}{n_\tau^2}$.

    \textbf{Mean difference.}
    \begin{equation*}
        m_1 - m_0 = -\frac{\eta_\tau}{n_\tau}(\nabla_\theta \ell(\theta_{\tau-1}; z^*) - \mu_g(\theta_{\tau-1})) = -\frac{\eta_\tau \delta_g}{n_\tau}.
    \end{equation*}

    \textbf{Log likelihood ratio.}
    With $N = \theta_\tau - m_0$, expand the quadratic terms and substitute $V_0^{-1} = \frac{n_\tau}{\eta_\tau^2}\Sigma_g^{-1}$. Collecting terms yields the stated result.
\end{proof}

\begin{theorem}[Distribution of the Log Likelihood Ratio]
    \label{thm:sgd-log-lr-distribution}
    Under $\mathbf{H_0}$:
    \begin{equation*}
        N \mid \theta_{\tau-1}, \mathbf{H_0} \sim \mathcal{N}\left(0, \frac{\eta_\tau^2 \Sigma_g}{n_\tau}\right).
    \end{equation*}
    Under $\mathbf{H_1^\tau}$:
    \begin{equation*}
        N \mid \theta_{\tau-1}, \mathbf{H_1^\tau} \sim \mathcal{N}\left(-\frac{\eta_\tau \delta_g}{n_\tau}, \frac{\eta_\tau^2(n_\tau-1)\Sigma_g}{n_\tau^2}\right).
    \end{equation*}
\end{theorem}

\begin{proof}
    Under $\mathbf{H_0}$, $\theta_\tau \mid \theta_{\tau-1} \sim \mathcal{N}(m_0, V_0)$ with $m_0 = \theta_{\tau-1} - \eta_\tau \mu_g(\theta_{\tau-1})$. Since $N = \theta_\tau - m_0$, we have $N \sim \mathcal{N}(0, V_0)$.

    Under $\mathbf{H_1^\tau}$, $\theta_\tau \mid \theta_{\tau-1} \sim \mathcal{N}(m_1, V_1)$. Thus $N = \theta_\tau - m_0 \sim \mathcal{N}(m_1 - m_0, V_1)$.
    Recalling that $m_1 - m_0 = -\frac{\eta_\tau \delta_g}{n_\tau}$, we have the stated mean and covariance.
\end{proof}

\subsection{Application to Linear Regression}

For linear regression with squared loss $\ell(\theta; (x,y)) = \frac{1}{2}(y - \theta^\top x)^2$ and Gaussian covariates $x \sim \cN(0, \Sigma_x)$:

\begin{lemma}[Mean Gradient]
    \label{lem:linreg-mean-grad}
    Let $\Delta = \theta - \theta^*$. The mean gradient is $\mu_g(\theta) = \Sigma_x \Delta$.
\end{lemma}

\begin{proof}
    Let $\nabla_\theta \ell(\theta; (x,y)) = (\theta^\top x - y)x = (\theta^\top x - \theta^{*\top}x - \epsilon)x = (\Delta^\top x - \epsilon)x$.
    Taking expectations:
    \begin{equation*}
        \mu_g(\theta) = \EE[(\Delta^\top x)x] - \EE[\epsilon]\EE[x] = \EE[xx^\top]\Delta = \Sigma_x \Delta.
    \end{equation*}
\end{proof}

\begin{lemma}[Gradient Covariance]
    \label{lem:linreg-cov-grad}
    The covariance of the per-sample gradient is:
    \begin{equation*}
        \Sigma_g(\theta) = (\sigma_\epsilon^2 + \Delta^\top \Sigma_x \Delta)\Sigma_x + \Sigma_x \Delta \Delta^\top \Sigma_x.
    \end{equation*}
\end{lemma}

\begin{proof}
    The second moment is $\EE[gg^\top] = \EE[(\Delta^\top x - \epsilon)^2 xx^\top]$. Since $\epsilon \perp x$ and $\EE[\epsilon]=0$:
    \begin{equation*}
        \EE[gg^\top] = \EE[(\Delta^\top x)^2 xx^\top] + \sigma_\epsilon^2 \Sigma_x.
    \end{equation*}
    Using Isserlis' theorem for the fourth moment of Gaussian $x$:
    \begin{equation*}
        \EE[(\Delta^\top x)^2 xx^\top] = (\Delta^\top \Sigma_x \Delta)\Sigma_x + 2\Sigma_x \Delta \Delta^\top \Sigma_x.
    \end{equation*}
    The covariance is $\Sigma_g = \EE[gg^\top] - \mu_g\mu_g^\top$. Subtracting $\mu_g\mu_g^\top = \Sigma_x \Delta \Delta^\top \Sigma_x$ yields the result.
\end{proof}

\begin{corollary}[Gradient Statistics at Optimum]
    At $\theta = \theta^*$, we have $\Delta = 0$, so $\mu_g(\theta^*) = 0$ and $\Sigma_g(\theta^*) = \sigma_\epsilon^2 \Sigma_x$.
\end{corollary}

\begin{lemma}[Target Gradient Deviation]
    At $\theta = \theta^*$, the target gradient is $\nabla_\theta \ell(\theta^*; (x^*, y^*)) = -\epsilon^* x^*$ where $\epsilon^*$ is the target residual. The deviation from the population mean is $\delta_g = -\epsilon^* x^*$.
\end{lemma}

\begin{proof}
    At optimum, $\nabla_\theta \ell = (\theta^{*\top}x^* - y^*)x^* = -\epsilon^* x^*$. Since $\mu_g(\theta^*) = 0$, $\delta_g = -\epsilon^* x^*$.
\end{proof}

\begin{theorem}[Detectability at Optimum]
    \label{thm:sgd-detectability}
    At $\theta = \theta^*$, the Mahalanobis distance of the target gradient is:
    \begin{equation*}
        m^* = \frac{(\epsilon^*)^2}{\sigma_\epsilon^2} \cdot (x^*)^\top \Sigma_x^{-1} x^*.
    \end{equation*}
    This factors as product of a \emph{label outlier score} $(\epsilon^*/\sigma_\epsilon)^2$ and a \emph{feature leverage score} $(x^*)^\top \Sigma_x^{-1} x^*.$
\end{theorem}

\begin{proof}
    Substituting into $m^* = \delta_g^\top \Sigma_g(\theta^*)^{-1} \delta_g$:
    \begin{equation*}
        m^* = (-\epsilon^* x^*)^\top (\sigma_\epsilon^2 \Sigma_x)^{-1} (-\epsilon^* x^*) = \frac{(\epsilon^*)^2}{\sigma_\epsilon^2} (x^*)^\top \Sigma_x^{-1} x^*.
    \end{equation*}
\end{proof}
Samples with large residuals (relative to population noise) and unusual feature vectors (relative to population covariance) are most detectable.

\section{Additional Experiments}
\label{app:experiments}

This appendix provides additional experimental results supporting the main paper. Section~\ref{app:exp-disjoint} reports the empirical-mean illustrations under the disjoint-batch protocol that exactly matches the $\SeMIstar$ assumptions. Section~\ref{app:exp-unknown-tau} extends the unknown-$\tau$ comparison. Section~\ref{app:exp-sgd} validates $\SeMISGD$ on linear regression. Section~\ref{app:exp-multivariate} reports the multivariate Gaussian setting. Section~\ref{app:exp-dpsgd-roc} reports a complementary DP-SGD pretrain-then-finetune protocol on three datasets. Section~\ref{app:exp-tau-selection} quantifies the gap between averaging over $\tau$ and selecting the best $\tau$ post hoc.

\paragraph{Compute.} The main DP-SGD audit on Fashion-MNIST and the three pretrain-then-finetune audits ran on a single NVIDIA GeForce RTX 2080 Ti: approximately 48 hours for the main audit, and approximately 5 hours each for the Fashion-MNIST, CIFAR-10, and Purchase-100 fine-tuning runs. The empirical-mean illustrations (Appendices~\ref{app:exp-disjoint}--\ref{app:exp-unknown-tau}), the SGD validation on linear regression (Appendix~\ref{app:exp-sgd}), and the multivariate experiment (App.~\ref{app:exp-multivariate}) run on CPU in minutes.

\subsection{Disjoint-Batch Protocol}
\label{app:exp-disjoint}

The numerical illustrations in Section~\ref{sec:sequential_test} use a subsampled-with-replacement protocol that departs from the i.i.d.\ disjoint-batch assumption underlying $\SeMIstar$. Here we report the matching figures under the disjoint protocol that exactly satisfies the test assumptions: $T = 10$ independent batches of size $n = 10$ drawn i.i.d.\ from $\cN(0,1)$, target $z^*$ inserted at $\tau = 5$ with Mahalanobis distance $5$, $R = 50\,000$ runs.

\begin{figure}[h!]
    \centering
    \begin{subfigure}{0.49\textwidth}
        \centering
        \includegraphics[width=\linewidth]{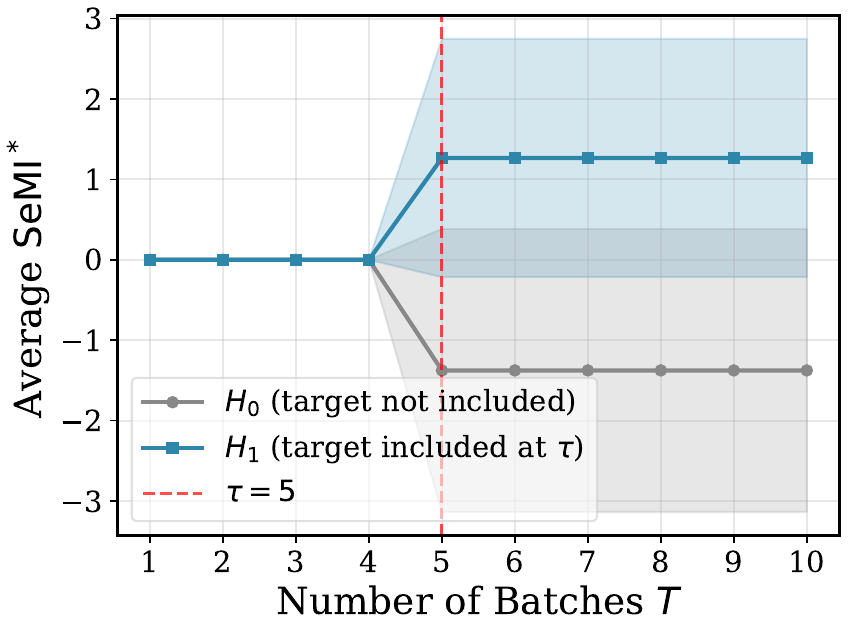}
        \caption{$\SeMIstar$ statistic}
    \end{subfigure}
    \hfill
    \begin{subfigure}{0.49\textwidth}
        \centering
        \includegraphics[width=\linewidth]{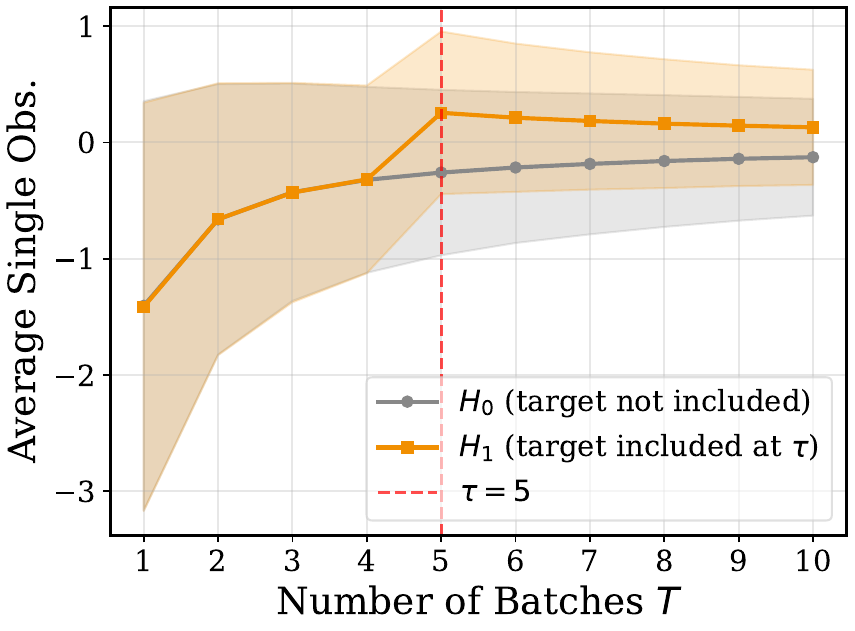}
        \caption{Final Obs.\ statistic}
    \end{subfigure}
    \caption{Log-likelihood ratio under the disjoint-batch protocol ($n = 10$, $\tau = 5$).}
    \label{fig:exp1_disjoint_lr}
\end{figure}

\begin{figure}[h!]
    \centering
    \begin{subfigure}{0.48\textwidth}
        \centering
        \includegraphics[width=\linewidth]{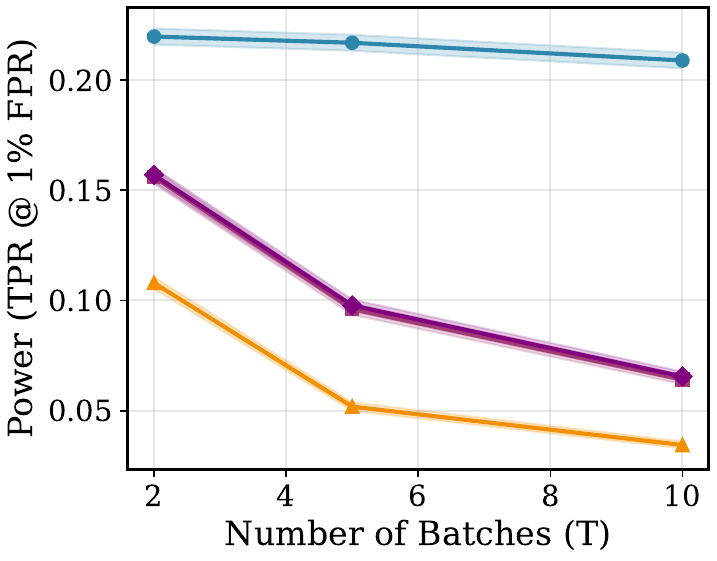}
        \caption{Power vs $T$}
    \end{subfigure}
    \hfill
    \begin{subfigure}{0.48\textwidth}
        \centering
        \includegraphics[width=\linewidth]{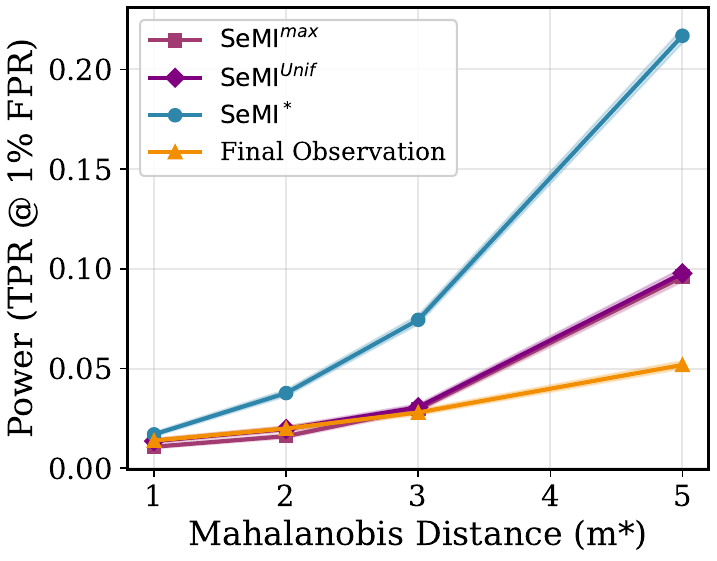}
        \caption{Power vs $z^\ast$ strength}
    \end{subfigure}
    \caption{Power comparison under the disjoint-batch protocol ($n = 10$, TPR at 1\% FPR). The qualitative behavior matches the subsampled-with-replacement protocol used in the main text.}
    \label{fig:exp2_disjoint_power}
\end{figure}

\subsection{Unknown Insertion Time}
\label{app:exp-unknown-tau}

This section provides additional results for the sequential tests when the insertion time $\tau$ is unknown. We use $n=10$, $\mu=0$, $\sigma=1$, and target $z^*$ with Mahalanobis distance 5. Power comparisons as a function of $T$ and signal strength appear in the main text (Figure~\ref{fig:exp2_advantage}).

Figure~\ref{fig:exp2_roc} shows ROC curves for different values of $T$. As $T$ increases, the gap between $\SeMIstar$ and the other methods widens, while $\SeMIunif$ and $\SeMImax$ maintain an advantage over Final Observation.

\begin{figure}[h!]
    \centering
    \begin{subfigure}{0.48\textwidth}
        \includegraphics[width=\linewidth]{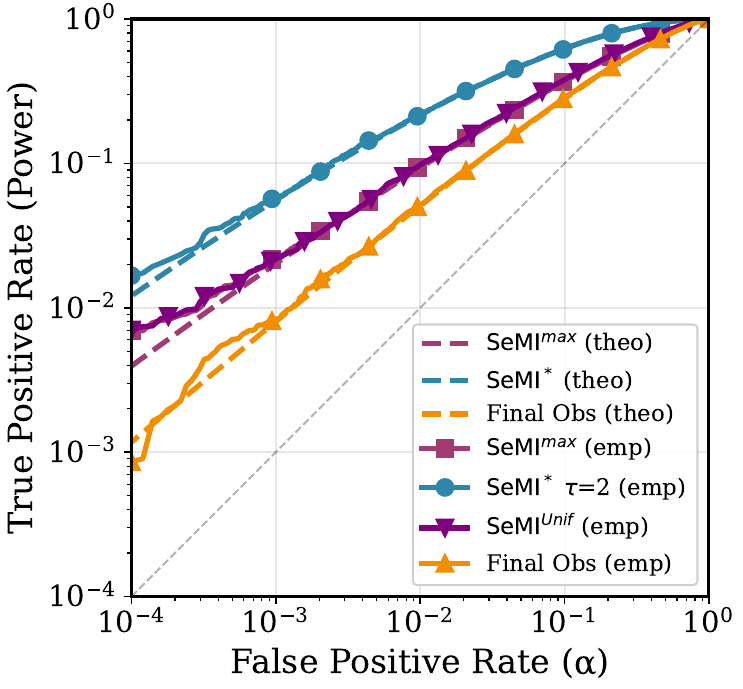}
        \caption{$T=5$}
    \end{subfigure}
    \begin{subfigure}{0.48\textwidth}
        \includegraphics[width=\linewidth]{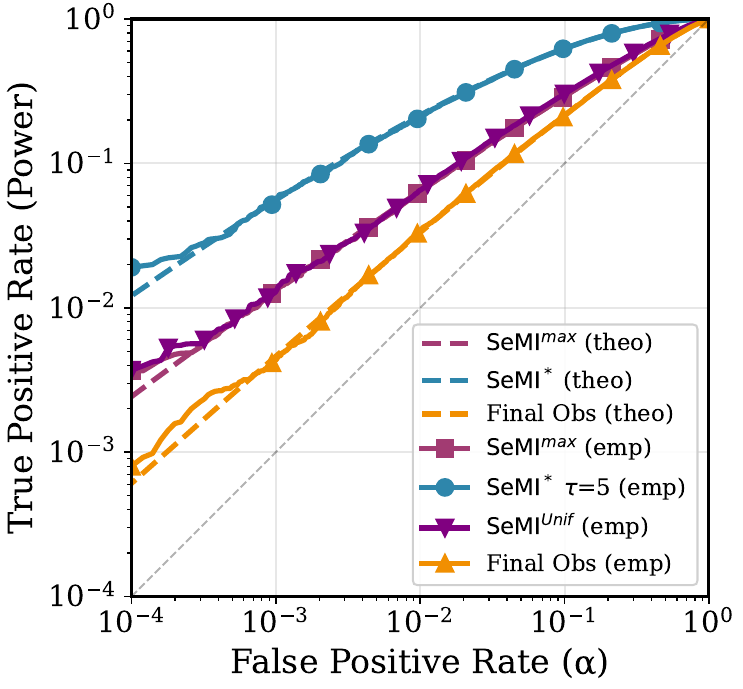}
        \caption{$T=10$}
    \end{subfigure}
    \caption{ROC curves comparing $\SeMIstar$, $\SeMIunif$, $\SeMImax$, and Final Observation for different $T$.}
    \label{fig:exp2_roc}
\end{figure}

\subsection{SGD Validation}
\label{app:exp-sgd}

We validate the extension to SGD (Section~\ref{sec:practical_attack}) using linear regression with $d=5$ features, $T=10$ updates, batch size $n=50$, and learning rate $\eta=0.05$.

Figure~\ref{fig:exp4_sgd}(a) shows ROC curves for insertion at $\tau=5$. $\SeMISGD$ outperforms the Delta and Back-Front heuristics from~\cite{jagielskiHowCombineMembershipInference2023}.

Figure~\ref{fig:exp4_sgd}(b) shows power as a function of insertion time $\tau$. The Back-Front statistic depends only on the endpoints $\theta_0$ and $\theta_T$, so its power is approximately invariant to where in the trajectory the target is inserted. The Delta statistic shows decreasing power as $\tau$ increases: early in training, gradient magnitudes are larger, producing stronger signals. $\SeMISGD$ maintains high power across all $\tau$.

\begin{figure}[h!]
    \centering
    \begin{subfigure}{0.48\textwidth}
        \includegraphics[width=\linewidth]{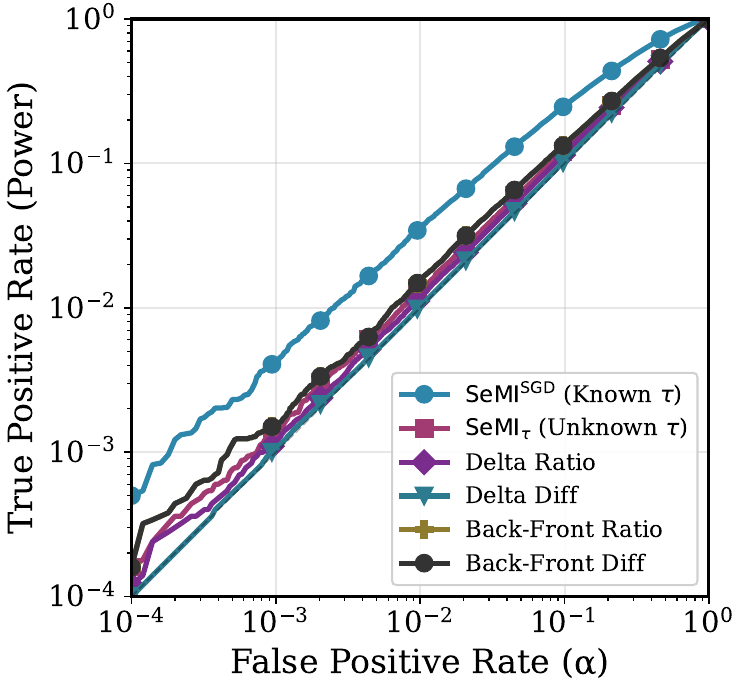}
        \caption{ROC Comparison ($\tau=5$)}
    \end{subfigure}
    \begin{subfigure}{0.48\textwidth}
        \includegraphics[width=\linewidth]{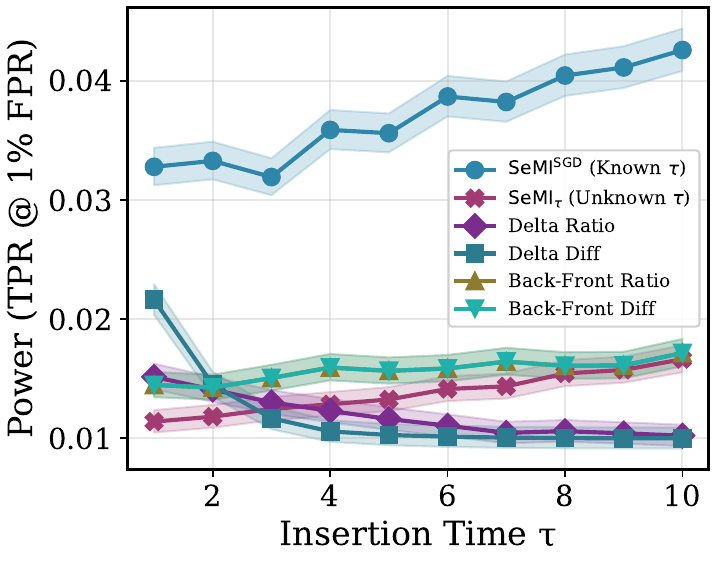}
        \caption{Power vs Insertion Time $\tau$}
    \end{subfigure}
    \caption{SGD validation on linear regression. (a) ROC curves. (b) Power vs $\tau$. $\SeMISGD$ maintains power across $\tau$ in this linear-regression setting; on the deep-model audit of Section~\ref{sec:experiments}, audit tightness instead grows with $\tau$ (Observation~4).}
    \label{fig:exp4_sgd}
\end{figure}

\subsection{Multivariate Setting}
\label{app:exp-multivariate}

We test the extension to $d$-dimensional Gaussian data (Appendix~\ref{app:multivariate}). Parameters: $n=10$, $T=5$, $\tau=3$, target $z^*$ with Mahalanobis distance 3 from the mean.

Figure~\ref{fig:exp7_multivariate}(a) displays the decision boundary in $d=2$. Figure~\ref{fig:exp7_multivariate}(b) shows that power remains constant as dimension $d$ increases from 2 to 50. For a fixed Mahalanobis distance, the test statistic depends on the projection onto the target direction, making power independent of ambient dimension.

\begin{figure}[h!]
    \centering
    \begin{subfigure}{0.48\textwidth}
        \includegraphics[width=\linewidth]{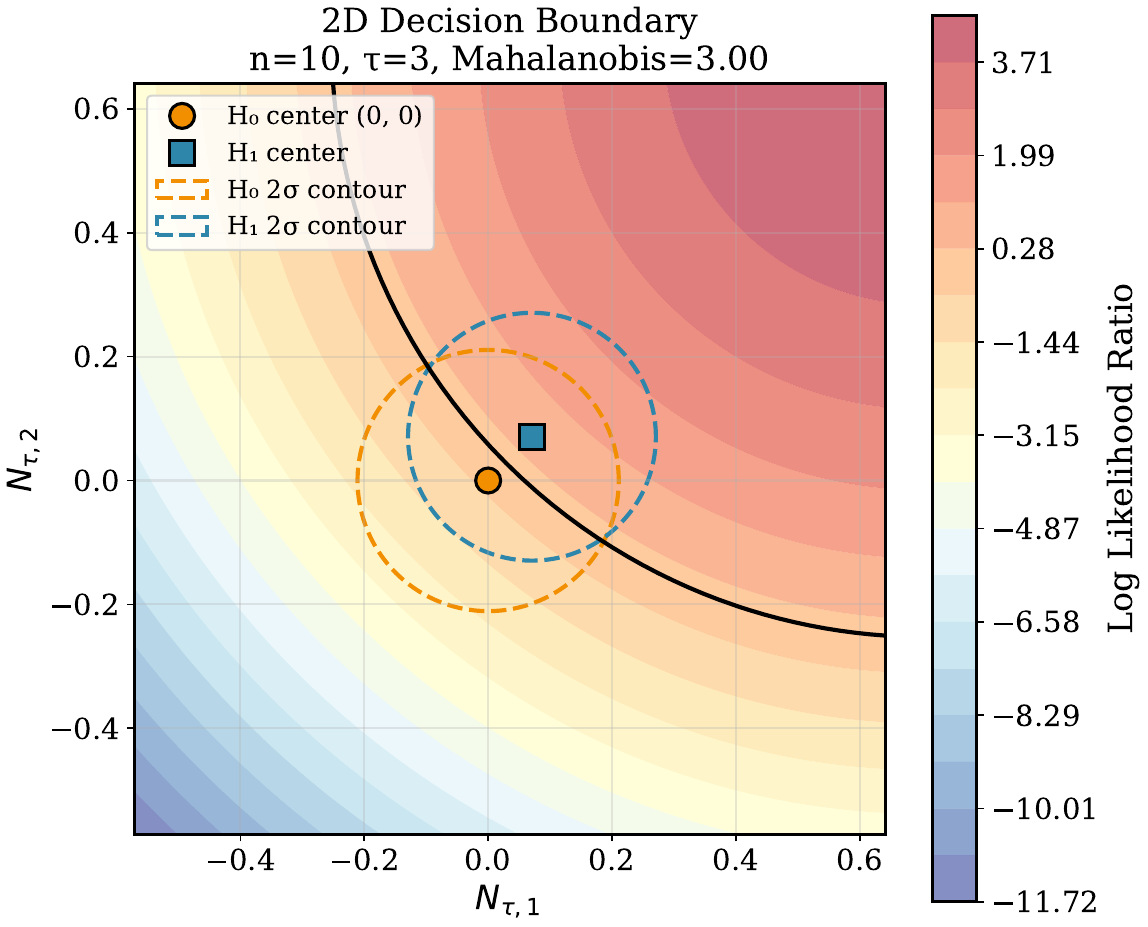}
        \caption{Decision Boundary ($d=2$)}
    \end{subfigure}
    \begin{subfigure}{0.48\textwidth}
        \includegraphics[width=\linewidth]{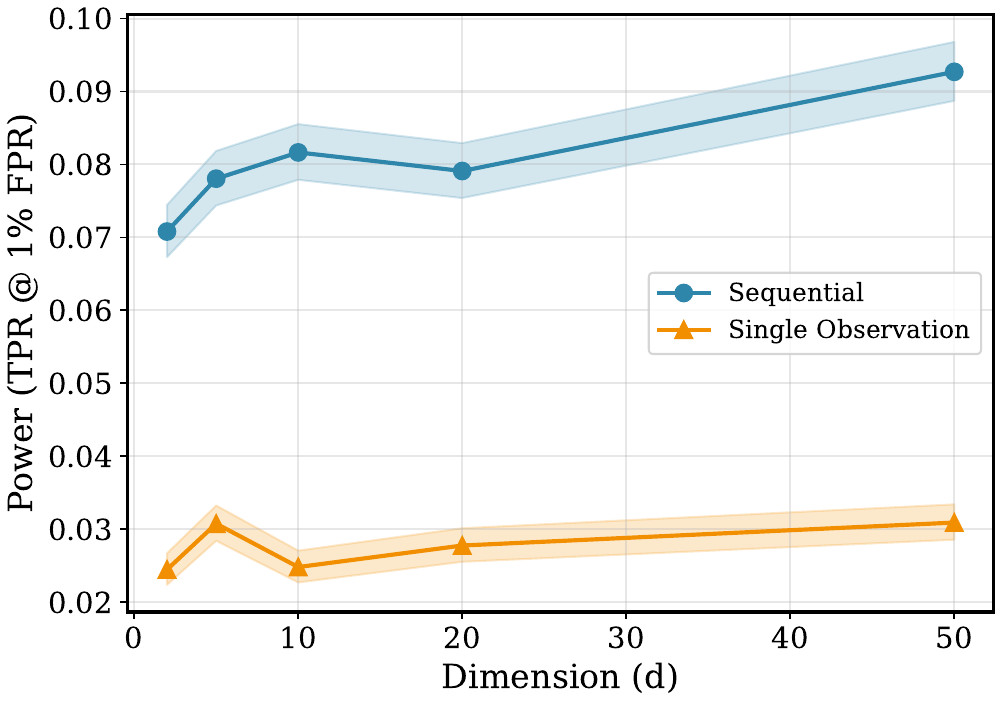}
        \caption{Power vs Dimension}
    \end{subfigure}
    \caption{Multivariate setting. (a) Decision boundary in 2D. (b) Power vs dimension $d$.}
    \label{fig:exp7_multivariate}
\end{figure}

\subsection{DP-SGD Fine-tuning Instantiation}
\label{app:exp-dpsgd-roc}

\textbf{Setup.} This is a separate protocol from the disjoint-phase audit of Section~\ref{sec:experiments}: instead of splitting one DP-SGD run into $T$ phases on disjoint subsets, we pretrain on a public split with SGD and then run $T$ DP-SGD updates on a private split, averaging audit bounds over $\tau \in \{1, \ldots, T\}$. We use batch size 64, $T = 10$ updates, and $\delta = 10^{-4}$. The auditor observes the parameter trajectory $(\theta_0, \ldots, \theta_T)$ and knows the insertion time $\tau$. For each audit, the target $z^*$ is selected from a candidate pool by taking the sample with the highest gradient Mahalanobis distance, estimated from a held-out reference set. We average results over $\tau \in \{1, \ldots, T\}$; Appendix~\ref{app:exp-tau-selection} shows that selecting the best $\tau$ post hoc yields tighter bounds, and how to choose $\tau$ a priori for a given architecture remains open. Three architectures are tested: multinomial logistic regression on Fashion-MNIST~\citep{xiao2017fashion}, a frozen pretrained VGG-16~\citep{simonyan2014very} with a single trainable layer on CIFAR-10~\citep{krizhevsky2009learning}, and a fully connected network with 128 hidden units on Purchase-100~\citep{shokri2017membership}.

\begin{figure}[h!]
    \centering
    \begin{subfigure}{0.32\textwidth}
        \includegraphics[width=\linewidth]{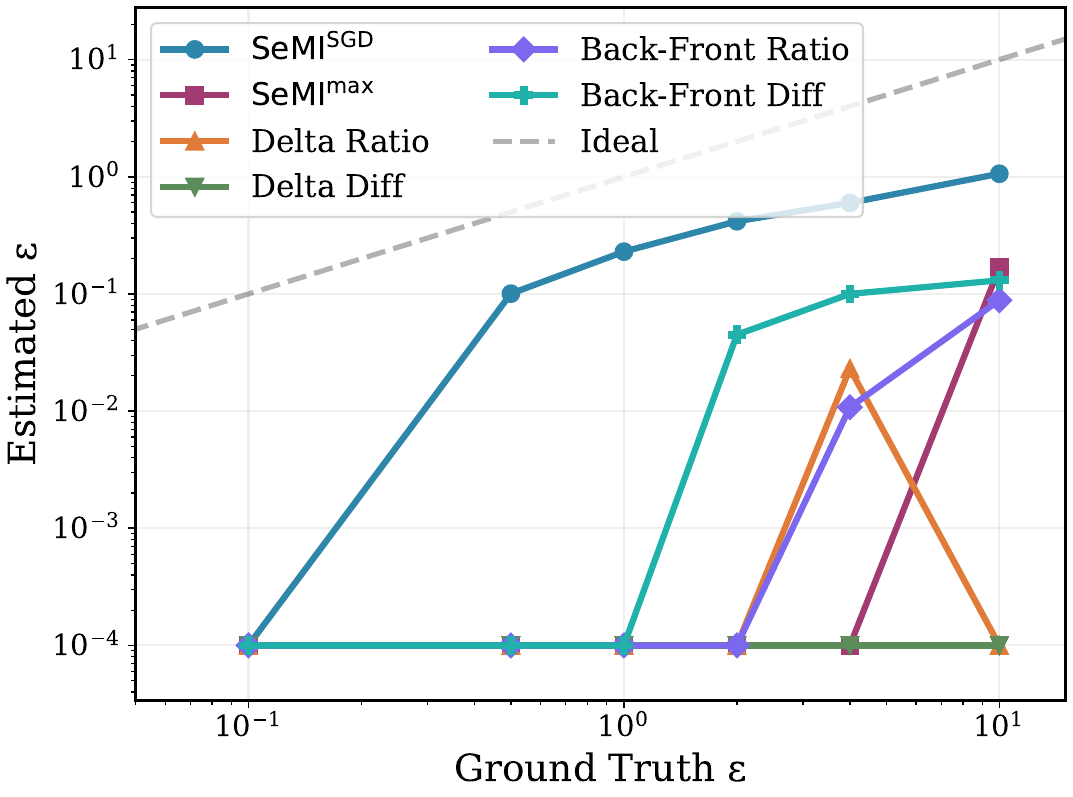}
        \caption{Fashion-MNIST}
    \end{subfigure}
    \hfill
    \begin{subfigure}{0.32\textwidth}
        \includegraphics[width=\linewidth]{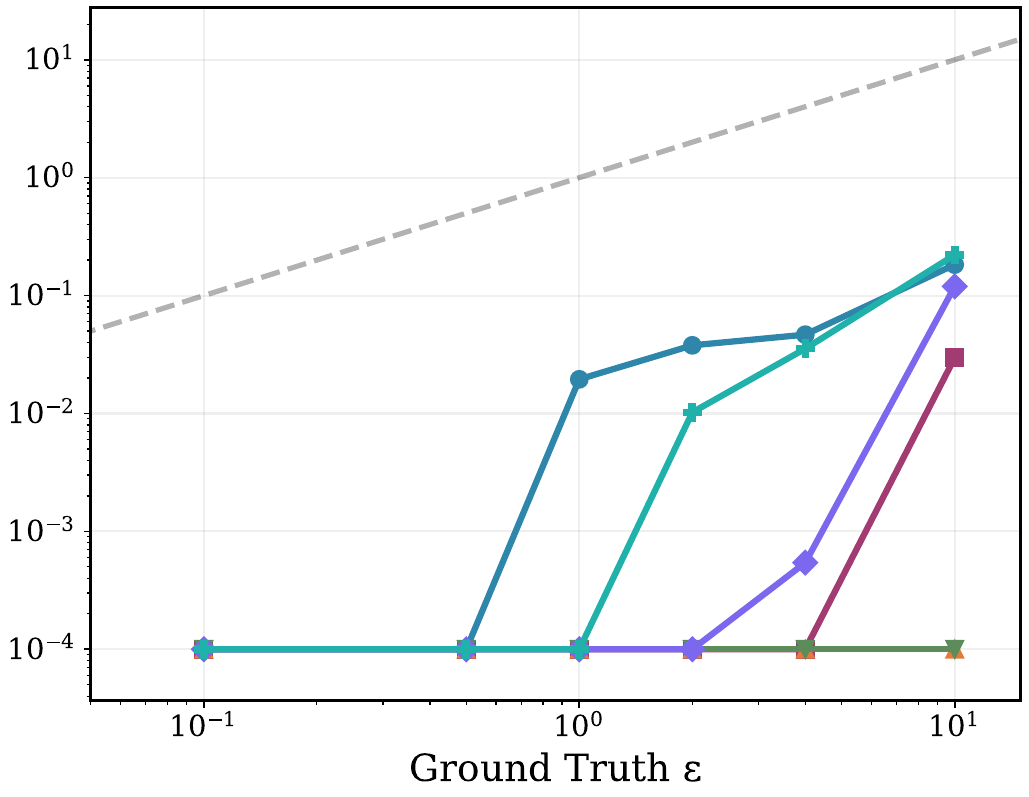}
        \caption{CIFAR-10}
    \end{subfigure}
    \hfill
    \begin{subfigure}{0.32\textwidth}
        \includegraphics[width=\linewidth]{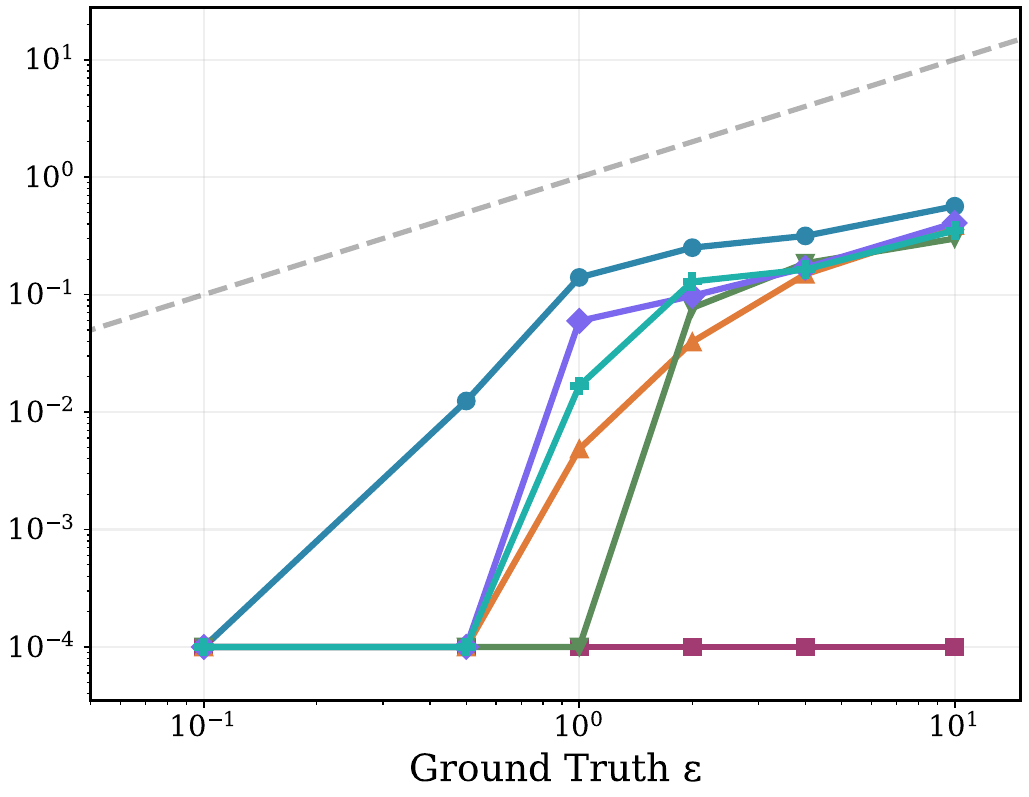}
        \caption{Purchase-100}
    \end{subfigure}
    \caption{Private fine-tuning with $T = 10$ DP-SGD updates and known $\tau$. Estimated $\varepsilon$ lower bound versus ground-truth $\varepsilon$, averaged over $\tau \in \{1, \ldots, T\}$. Closer to the diagonal is tighter. $\SeMISGD$ yields tighter lower bounds than loss-based heuristics across the privacy range.}
    \label{fig:exp14_epsilon}
\end{figure}

\textbf{$\varepsilon$ lower bounds.} Figure~\ref{fig:exp14_epsilon} shows the audit lower bound $\underline{\varepsilon}$ versus ground-truth $\varepsilon$ on the three datasets. On Fashion-MNIST, $\SeMISGD$ produces tighter lower bounds than every baseline across the privacy range. On CIFAR-10 and Purchase-100, $\SeMISGD$ matches or beats the baselines, with the clearest gains in low-privacy regimes.

\textbf{ROC curves.} Figure~\ref{fig:exp14_roc_fmnist} reports ROC curves on Fashion-MNIST at two privacy levels. Without DP noise, $\SeMISGD$ and $\SeMImax$ both achieve near-perfect discrimination (AUC $\approx 1.0$), while the heuristic baselines perform substantially worse (Back-Front AUC $\approx 0.80$, Delta AUC $\approx 0.65$). At $\varepsilon = 1.0$, all methods degrade due to the DP noise, but $\SeMISGD$ maintains an advantage (AUC $= 0.60$) over baselines (AUC $\approx 0.51$--$0.53$); the gap is most pronounced at low false positive rates, which drive audit tightness.

\begin{figure}[h!]
    \centering
    \begin{subfigure}{0.48\textwidth}
        \includegraphics[width=\linewidth]{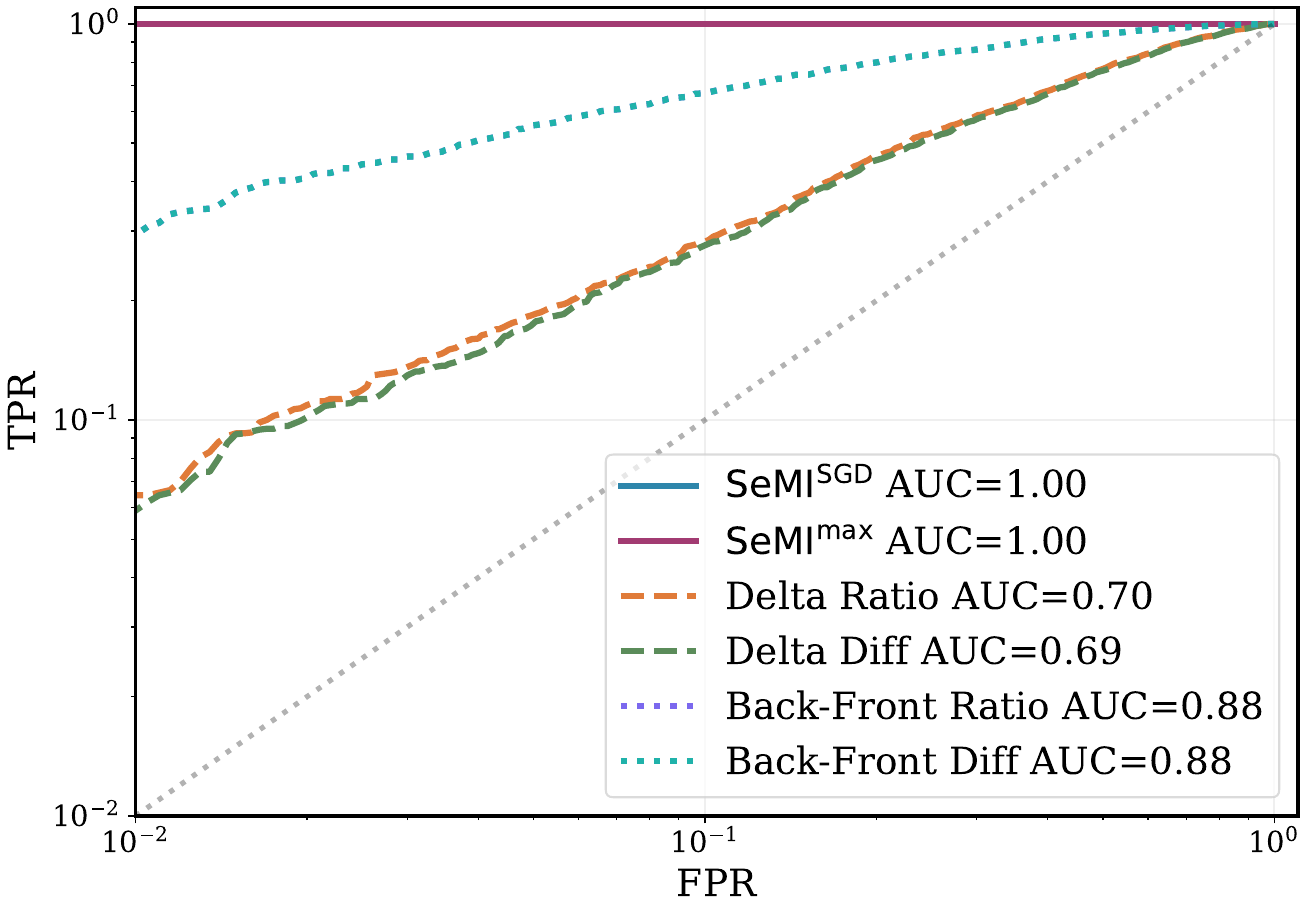}
        \caption{Non-private ($\varepsilon = \infty$)}
    \end{subfigure}
    \begin{subfigure}{0.48\textwidth}
        \includegraphics[width=\linewidth]{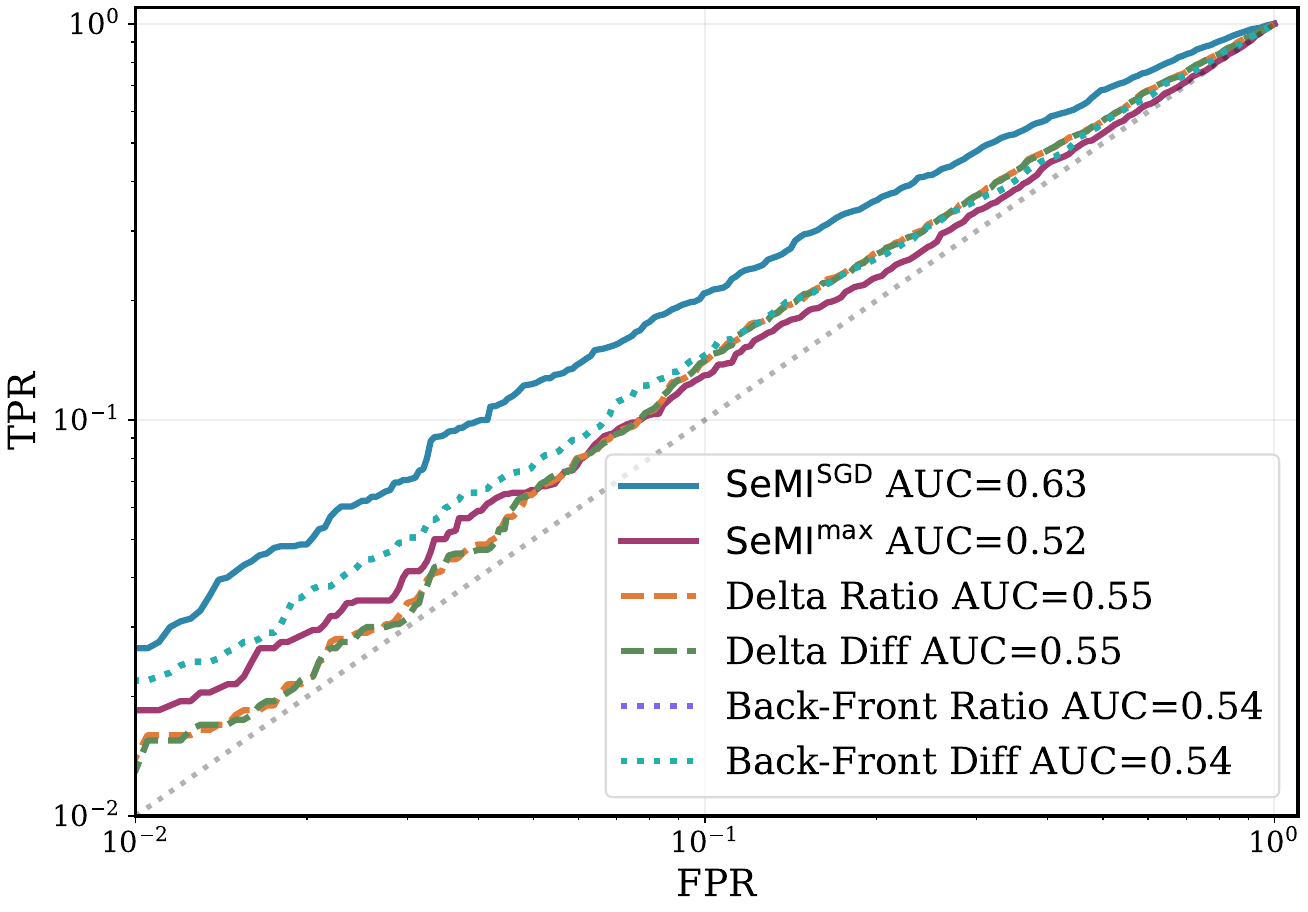}
        \caption{Private ($\varepsilon = 1.0$)}
    \end{subfigure}
    \caption{ROC curves for Fashion-MNIST pretrain-then-finetune with DP-SGD. (a) Without DP noise, $\SeMISGD$ achieves near-perfect discrimination. (b) With DP-SGD ($\varepsilon = 1.0$), $\SeMISGD$ maintains an advantage over heuristic baselines. Cut at FPR $= 0.01$.}
    \label{fig:exp14_roc_fmnist}
\end{figure}

\subsection{Effect of Insertion Time Selection}
\label{app:exp-tau-selection}

Appendix~\ref{app:exp-dpsgd-roc} reports privacy lower bounds averaged over insertion times $\tau \in \{1, \ldots, T\}$. Here we compare this averaging strategy against selecting the best $\tau$ post hoc, i.e., reporting the tightest bound across all insertion times.

Figure~\ref{fig:exp14_tau_selection} shows results for Fashion-MNIST. Solid lines show the averaged estimates; dashed lines show the maximum (best $\tau$) estimates. For $\SeMISGD$, selecting the best $\tau$ yields substantially tighter bounds, particularly at moderate privacy levels ($\varepsilon \in [0.5, 2]$). The gap between Average and Max indicates that some insertion times are more favorable for auditing than others.

This observation suggests that an auditor who can choose when to insert the target, or who audits at multiple insertion times and reports the best result, can obtain tighter privacy lower bounds.

\begin{figure}[h!]
    \centering
    \includegraphics[width=\textwidth]{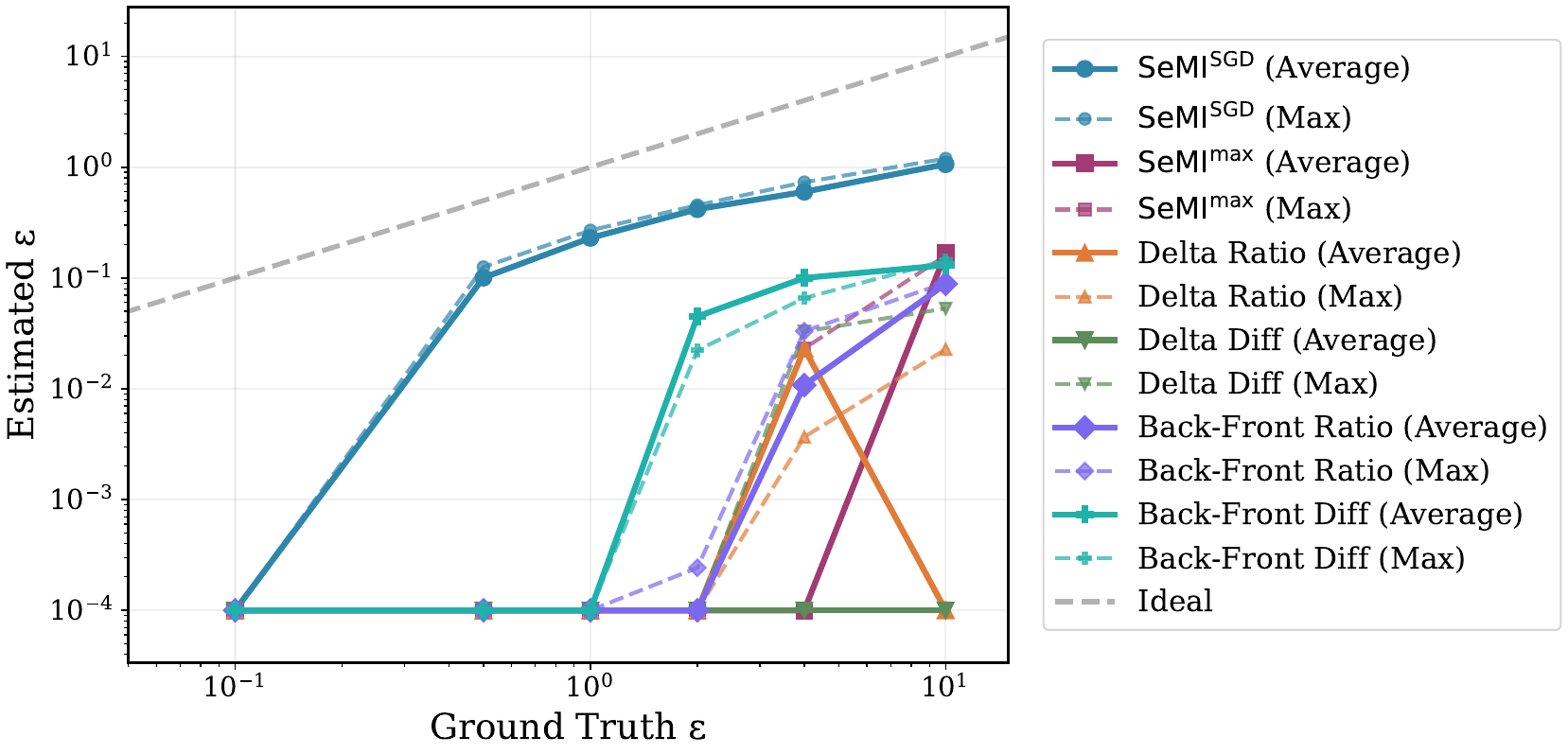}
    \caption{Effect of insertion time selection on Fashion-MNIST. Solid lines: averaged over $\tau$. Dashed lines: best $\tau$ (post hoc selection). Selecting the optimal insertion time yields tighter privacy lower bounds.}
    \label{fig:exp14_tau_selection}
\end{figure}

\subsection{Smaller Privacy Budgets}
\label{app:exp-low-eps}

We rerun the Fashion-MNIST audit of Section~\ref{sec:experiments} at $\varepsilon \in \{1, 2, 4\}$, keeping the canary, the protocol and the insertion time of the main text ($T = 5$, $\tau = 5$). Table~\ref{tab:low-eps} reports the AUC and a single $\hat\varepsilon$ computed over all audit rounds, in place of the subsampled violins of Figure~\ref{fig:epsilon_distrib}. The ranking of Section~\ref{sec:experiments} holds at every $\varepsilon$, with the sequential attacks ahead of the loss heuristics and LiRA last, since it sees only the final model. Smaller $\varepsilon$ makes the methods harder to separate, as the DP noise drives every attack towards chance and $\hat\varepsilon$ towards zero.

\begin{table}[h]
    \centering
    \small
    \caption{AUC / $\hat\varepsilon$ on Fashion-MNIST at $T = 5$, $\tau = 5$, $\delta = 10^{-4}$.}
    \label{tab:low-eps}
    \begin{tabular}{lcccc}
        \toprule
        Attack & $\varepsilon = 1$ & $\varepsilon = 2$ & $\varepsilon = 4$ & $\varepsilon = 8$ \\
        \midrule
        $\SeMISGD$ & 0.535 / 0.06 & 0.549 / 0.11 & 0.595 / 0.30 & 0.623 / 0.49 \\
        $\SeMISGDmax$ & 0.537 / 0.02 & 0.538 / 0.09 & 0.592 / 0.31 & 0.608 / 0.41 \\
        Loss Diff & 0.523 / 0.00 & 0.527 / 0.00 & 0.529 / 0.02 & 0.534 / 0.10 \\
        Loss Ratio & 0.521 / 0.00 & 0.526 / 0.00 & 0.530 / 0.00 & 0.533 / 0.07 \\
        LiRA & 0.515 / 0.00 & 0.519 / 0.00 & 0.501 / 0.00 & 0.509 / 0.00 \\
        \bottomrule
    \end{tabular}
\end{table}

\subsection{Gaussianity of Batch Gradients}
\label{app:exp-gaussianity}

The assumption is on the batch-mean gradient, which is closer to Gaussian at the large batch sizes typical of DP training. To test it, we freeze $\theta_{\tau-1}$ from one audited run ($T = 5$, $\tau = 5$), draw 4000 batch-gradient means per dataset under SGD and DP-SGD, and test normality along the direction the score weights, $\Sigma_g^{-1}\delta_g$, and along the four leading principal directions. Normality cannot be tested directly in dimension $d \sim 10^4$ from 4000 samples, so both tests in Table~\ref{tab:gaussianity} are on one-dimensional projections. The first column reports a Shapiro--Wilk $p$-value, and the second counts rejections at level $5\%$ among the four principal directions. Gaussianity is accepted everywhere except Fashion-MNIST under plain SGD, where the projection has skew $+0.11$ and excess kurtosis $-0.09$. Every test accepts normality under DP-SGD, the regime we audit.

\begin{table}[h]
    \centering
    \small
    \caption{Normality tests on batch-gradient means at $\theta_{\tau-1}$ ($T = 5$, $\tau = 5$, 4000 samples).}
    \label{tab:gaussianity}
    \begin{tabular}{lcc}
        \toprule
        Data / regime & $\Sigma_g^{-1}\delta_g$ direction, $p$-value & PCA directions rejected \\
        \midrule
        Fashion-MNIST SGD & 0.008 & 0/4 \\
        Fashion-MNIST DP-SGD & 0.35 & 0/4 \\
        CIFAR-10 SGD & 0.40 & 0/4 \\
        CIFAR-10 DP-SGD & 0.70 & 0/4 \\
        Purchase-100 SGD & 0.83 & 0/4 \\
        Purchase-100 DP-SGD & 0.91 & 0/4 \\
        \bottomrule
    \end{tabular}
\end{table}

\newpage
\section*{NeurIPS Paper Checklist}

\begin{enumerate}

    \item {\bf Claims}
    \item[] Question: Do the main claims made in the abstract and introduction accurately reflect the paper's contributions and scope?
    \item[] Answer: \answerYes{}
    \item[] Justification: The three contributions stated in the abstract and Section~\ref{sec:introduction} map to Sections~\ref{sec:mi_tests},~\ref{sec:optimal_test},~\ref{sec:practical_attack}, and~\ref{sec:experiments}. Assumptions and scope are stated where claims are made.

    \item {\bf Limitations}
    \item[] Question: Does the paper discuss the limitations of the work performed by the authors?
    \item[] Answer: \answerYes{}
    \item[] Justification: Section~\ref{sec:conclusion} discusses the Gaussian batch-gradient assumption, and Observation~4 in Section~\ref{sec:experiments} the open question of choosing $\tau$ a priori. Section~\ref{sec:experiments} reports regimes where gains are largest and smallest.

    \item {\bf Theory assumptions and proofs}
    \item[] Question: For each theoretical result, does the paper provide the full set of assumptions and a complete (and correct) proof?
    \item[] Answer: \answerYes{}
    \item[] Justification: Each theorem and lemma states its assumptions in the body. Full proofs are cross-referenced and given in the appendix (optimal-test, uniform-prior, unknown-$\tau$, multivariate, SGD, audit lemma).

    \item {\bf Experimental result reproducibility}
    \item[] Question: Does the paper fully disclose all the information needed to reproduce the main experimental results of the paper to the extent that it affects the main claims and/or conclusions of the paper (regardless of whether the code and data are provided or not)?
    \item[] Answer: \answerYes{}
    \item[] Justification: Section~\ref{sec:experiments} reports the training pipeline, hyperparameters, and audit-protocol parameters; Appendix~\ref{app:experiments} extends this to the fine-tuning experiments. All datasets are public.

    \item {\bf Open access to data and code}
    \item[] Question: Does the paper provide open access to the data and code, with sufficient instructions to faithfully reproduce the main experimental results, as described in supplemental material?
    \item[] Answer: \answerYes{}
    \item[] Justification: The code with reproduction scripts is provided as supplemental material. All datasets used are publicly available.

    \item {\bf Experimental setting/details}
    \item[] Question: Does the paper specify all the training and test details (e.g., data splits, hyperparameters, how they were chosen, type of optimizer) necessary to understand the results?
    \item[] Answer: \answerYes{}
    \item[] Justification: Optimizer, learning rate, clipping norm, batch size, number of steps, phase decomposition, splits, and audit-protocol parameters are reported in Section~\ref{sec:experiments}; architecture-specific details are in Appendix~\ref{app:experiments}.

    \item {\bf Experiment statistical significance}
    \item[] Question: Does the paper report error bars suitably and correctly defined or other appropriate information about the statistical significance of the experiments?
    \item[] Answer: \answerYes{}
    \item[] Justification: $\hat\varepsilon$ uses Clopper--Pearson upper bounds on Type-I/II errors over $R=1900$ rounds with a union bound across the calibrated grid at overall confidence $1-\xi=0.98$ (Section~\ref{sec:experiments}, Appendix~\ref{sec:audit_protocol}). Violins in Figure~\ref{fig:epsilon_distrib} come from subsampling $90\%$ of the rounds without replacement.

    \item {\bf Experiments compute resources}
    \item[] Question: For each experiment, does the paper provide sufficient information on the computer resources (type of compute workers, memory, time of execution) needed to reproduce the experiments?
    \item[] Answer: \answerYes{}
    \item[] Justification: Hardware (NVIDIA GeForce RTX 2080 Ti) and approximate runtimes per experiment are reported in Appendix~\ref{app:experiments}.

    \item {\bf Code of ethics}
    \item[] Question: Does the research conducted in the paper conform, in every respect, with the NeurIPS Code of Ethics \url{https://neurips.cc/public/EthicsGuidelines}?
    \item[] Answer: \answerYes{}
    \item[] Justification: The work uses public benchmark datasets, involves no human subjects, and contributes to privacy auditing of differentially private training, in line with the Code of Ethics.

    \item {\bf Broader impacts}
    \item[] Question: Does the paper discuss both potential positive societal impacts and negative societal impacts of the work performed?
    \item[] Answer: \answerYes{}
    \item[] Justification: Tighter privacy audits improve empirical accountability of DP training (positive). The same attacks could amplify membership inference on undefended models; the focus on auditing mechanisms with formal $\varepsilon$ guarantees bounds this risk.

    \item {\bf Safeguards}
    \item[] Question: Does the paper describe safeguards that have been put in place for responsible release of data or models that have a high risk for misuse (e.g., pre-trained language models, image generators, or scraped datasets)?
    \item[] Answer: \answerNA{}
    \item[] Justification: The paper releases no pretrained models, generative models, or scraped datasets. Released artifacts are auditing scripts on public benchmarks.

    \item {\bf Licenses for existing assets}
    \item[] Question: Are the creators or original owners of assets (e.g., code, data, models), used in the paper, properly credited and are the license and terms of use explicitly mentioned and properly respected?
    \item[] Answer: \answerYes{}
    \item[] Justification: Datasets (Fashion-MNIST, CIFAR-10, Purchase-100) and architectures (VGG-16) are cited at first use, as are baseline attacks (LiRA, Delta, Back-Front). Licenses are listed in the supplemental README.

    \item {\bf New assets}
    \item[] Question: Are new assets introduced in the paper well documented and is the documentation provided alongside the assets?
    \item[] Answer: \answerYes{}
    \item[] Justification: The supplemental code includes a README with reproduction instructions.

    \item {\bf Crowdsourcing and research with human subjects}
    \item[] Question: For crowdsourcing experiments and research with human subjects, does the paper include the full text of instructions given to participants and screenshots, if applicable, as well as details about compensation (if any)?
    \item[] Answer: \answerNA{}
    \item[] Justification: The paper does not involve crowdsourcing or human-subjects research.

    \item {\bf Institutional review board (IRB) approvals or equivalent for research with human subjects}
    \item[] Question: Does the paper describe potential risks incurred by study participants, whether such risks were disclosed to the subjects, and whether Institutional Review Board (IRB) approvals (or an equivalent approval/review based on the requirements of your country or institution) were obtained?
    \item[] Answer: \answerNA{}
    \item[] Justification: The paper does not involve human-subjects research.

    \item {\bf Declaration of LLM usage}
    \item[] Question: Does the paper describe the usage of LLMs if it is an important, original, or non-standard component of the core methods in this research? Note that if the LLM is used only for writing, editing, or formatting purposes and does \emph{not} impact the core methodology, scientific rigor, or originality of the research, declaration is not required.
    \item[] Answer: \answerNA{}
    \item[] Justification: LLMs are not part of the core methodology.

\end{enumerate}

\end{document}